\documentclass[mnsc,nonblindrev]{informs3_hide}

\OneAndAHalfSpacedXI 

\usepackage{mathabx, makecell}
\usepackage{natbib, multirow, multicol, xspace, enumitem, amsmath, pifont, algorithm, subcaption}
\usepackage[noend]{algpseudocode}
\usepackage{accents}
\usepackage{algorithmicx}
\usepackage{graphicx}
\usepackage{multicol}
\usepackage{mdframed}
\usepackage{lmodern}
\usepackage{tabu}
\bibpunct[, ]{(}{)}{,}{a}{}{,}%
\def\bibfont{\small}%
\usepackage{bm}
\usepackage[normalem]{ulem}
\usepackage[dvipsnames]{xcolor}
\usepackage[all]{xy}
\usepackage{pifont}
\usepackage{hyperref}
\hypersetup{
    colorlinks=true,
    linkcolor=blue,
    filecolor=magenta,      
    urlcolor=cyan,
    citecolor=blue
    }
\usepackage{ulem}
\usepackage{framed}

\newcommand{\defeq}{\overset{\text{\tiny def}}{=}}

\newcommand{\nsc}[1]{{\normalfont\textsc{#1}}}

\newcommand{\gatedprioritygreedy}{Gated Priority-based Greedy\xspace}
\newcommand{\multivaryingalg}{Order-Size F-Priority\xspace}
\newcommand{\singvaryingalg}{Cost-Comparison AdjV-Priority\xspace}
\newcommand{\multiinvaralg}{Cost-Comparison V-Priority\xspace}
\newcommand{\singinvaralg}{Randomized-Cost-Comparison V-Priority}
\newcommand{\singvaryingalgimprove}{Order-Size AdjV-Priority\xspace}
\newcommand{\CR}{\mathfrak{R}}

\theoremstyle{TH}%

\TheoremsNumberedThrough     
\ECRepeatTheorems

\EquationsNumberedThrough    

\MANUSCRIPTNO{}

\begin{document}

\RUNAUTHOR{}

\RUNTITLE{Gated Priority-based Greedy Policies for Two-layer Multi-item Order Fulfillment}

\TITLE{Optimal and Order-optimal Gated Priority-based Greedy Policies for Two-layer Multi-item Order Fulfillment}

\ARTICLEAUTHORS{
\AUTHOR{Xi Chen\footnotemark[1]}
\AFF{Leonard N.~Stern School of Business, New York University, New York, NY 10012, USA, \EMAIL{xc13@stern.nyu.edu}}
 \AUTHOR{Yuze Chen\footnotemark[1]}
 \AFF{Qiuzhen College, Tsinghua University, Beijing 100084, China, \EMAIL{yz-chen21@mails.tsinghua.edu.cn}}
 \AUTHOR{Ziyi Chen\footnotemark[1]}
 \AFF{Qiuzhen College, Tsinghua University, Beijing 100084, China, \EMAIL{chenziyi25@mails.tsinghua.edu.cn}}
 \AUTHOR{Yuan Zhou\footnotemark[1]}
 \AFF{Yau Mathematical Sciences Center \& Department of Mathematical Sciences, Tsinghua University, Beijing 100084, China, \EMAIL{yuan-zhou@tsinghua.edu.cn}}
 } 

\renewcommand{\thefootnote}{\fnsymbol{footnote}}
\footnotetext[1]{Author names listed in alphabetical order.}
\renewcommand{\thefootnote}{\arabic{footnote}}

\ABSTRACT{We study how an e-commerce firm should make real-time fulfillment decisions in a two-layer distribution network when multi-item customer orders arrive sequentially and future demand is unknown. The central managerial tension is whether to use scarce front distribution center (FDC) inventory to save current fulfillment cost or preserve that inventory for future orders that may be more valuable to serve locally. We formulate an adversarial online model with multiple FDCs, one regional distribution center (RDC), multi-unit multi-item orders, and item-specific and time-varying variable costs. Our theoretical objective is to characterize when simple, interpretable, and implementable fulfillment rules can perform nearly as well as an optimal clairvoyant planner. We develop a family of \gatedprioritygreedy policies, derive competitive-ratio guarantees under both time-varying and time-invariant cost structures, and establish matching or near-matching lower bounds for any online algorithm. Numerical experiments show that the proposed policies perform strongly relative to generalized myopic and forecast-based benchmarks. The analysis yields managerial guidance on when local inventory should be protected, when splitting orders is worth the fixed-cost burden, and how the relative magnitudes of fixed and variable costs determine the value of more sophisticated optimization.}

\maketitle

\section{Introduction}

The rapid growth of e-commerce has made order fulfillment a central operational challenge in modern retail. Large online platforms process massive volumes of orders and must do so cost-effectively across geographically distributed inventories and customer destinations. Fulfillment decisions therefore play a direct role in shaping operational efficiency and profitability.

E-commerce fulfillment is particularly challenging because of three defining features. First, customer orders often contain multiple items with heterogeneous cost and inventory characteristics, which requires the firm to make coordinated fulfillment decisions across different products within the same order. Second, fulfillment costs are only partially stable over time: although the fixed cost of using a distribution center may remain constant over the selling horizon, per-unit variable costs may vary across orders over time because different orders are shipped to different destinations, and may also be affected by seasonality, promotions, congestion, and other changing operating conditions. Third, firms often operate a two-layer fulfillment network with multiple front distribution centers (FDCs) and a regional distribution center (RDC) \citep{shenJDcomTransactionLevelData2024}. In such networks, different FDCs may offer different cost advantages for the same order depending on the geographic match between fulfillment locations and customer destinations, with these differences arising primarily through destination-dependent per-unit variable costs rather than fixed activation costs. At the same time, each FDC holds only limited inventory. The RDC, in contrast, carries deeper inventory and provides a reliable fallback option, although usually at a higher fulfillment cost. Together, these three features create a fulfillment environment that is substantially more complex than the one faced in traditional store-based retail.

These operational features give rise to a dynamic two-layer multi-item order fulfillment problem. For each arriving order, the firm must decide which location should supply each item, taking into account current inventory availability, the fixed cost of using fulfillment locations, the variable cost of assigning items across locations, and the future option value of scarce FDC inventory. Because the relative cost of using an FDC depends on the destination of the order, the firm must also account for the geographic match between inventory locations and demand locations. The core tension is intertemporal: allocating FDC inventory to reduce current fulfillment cost may diminish the firm's ability to fulfill future orders at low cost. The problem is also computationally challenging. When an order contains multiple items and the network includes more than one FDC, the set of feasible fulfillment plans grows combinatorially with the number of items. As we show in this paper, even the single-period fulfillment problem is NP-hard, which makes it difficult to derive real-time fulfillment rules that are both computationally efficient and operationally effective. These difficulties are further amplified by sequential order arrivals, irreversible fulfillment decisions, and time-varying costs. The central question is therefore how to design a real-time fulfillment policy that remains implementable at scale while performing well under demand uncertainty and exploiting limited local inventory effectively.

Our paper addresses this question through a theoretical framework that is closely connected to practice. We characterize the performance loss induced by immediate and irreversible real-time decisions under uncertain future order arrivals, time-varying fulfillment costs, and computational constraints. To this end, we develop a \emph{gated priority-based greedy} framework. We show that, under different operational settings, appropriately designed gating conditions and priority rules over distribution centers enable a simple greedy policy to attain a finite competitive ratio relative to the clairvoyant benchmark, independent of the selling horizon, and that these competitive ratio guarantees are either optimal or order-optimal in all settings we study. This perspective is particularly relevant for large-scale e-commerce platforms, where highly sophisticated optimization methods may be difficult to deploy reliably in real time. In such settings, transparent and computationally efficient fulfillment rules with provable worst-case guarantees can offer an attractive balance between rigor and practicality.

\subsection{Overview of Our Model}

\noindent\underline{Model.}
We study the multi-item order fulfillment problem in a two-layer distribution network consisting of $K$ front distribution centers (FDCs) and one regional distribution center (RDC). FDCs hold finite and non-replenishable inventory across multiple items, while the RDC acts as an unlimited backup that can always fulfill any remaining demand. Customer orders arrive sequentially over a finite selling horizon. Each order may request multiple items and, importantly, multiple units of each item, which provides a more realistic generalization of the single-unit demand assumption commonly adopted in prior work \citep{jasinLPBasedCorrelatedRounding2015,maOrderOptimalCorrelatedRounding2023,amilMultiItemOrderFulfillment2025}. Upon the arrival of each order, the fulfillment policy must immediately determine how to allocate demand across the FDCs and the RDC, without observing future orders. Each distribution center incurs a fixed activation cost $f_k$ (where $k=0$ denotes the RDC and $k\in[K]$ denotes an FDC) whenever it is used to fulfill any part of an order, regardless of the quantity shipped, as well as an item-specific per-unit variable cost. The objective is to minimize the total fulfillment cost over the selling horizon. A distinctive feature of our model is that variable costs are allowed to vary over time, which captures fulfillment environments in which variable costs change across arriving orders due to destination-dependent fulfillment costs and evolving operating conditions. To the best of our knowledge, this practically relevant setting has not been studied in the existing literature.

\medskip \noindent\underline{Competitive Analysis.}
We evaluate performance through the competitive ratio, defined as the worst-case ratio between the total fulfillment cost incurred by an online policy and that incurred by an optimal offline benchmark with full knowledge of the entire order sequence and all cost realizations. We work in a fully adversarial setting, in which the initial inventory levels, the sequence of customer orders, and the time-varying variable costs are all chosen by an adversary, without any distributional assumptions. This framework yields robust, distribution-free performance guarantees and directly captures the risk of poor performance under challenging demand and cost environments.

\subsection{Our Contributions}
\label{subsec: contributions and main approach}

\noindent \underline{A \gatedprioritygreedy Framework.} A recent paper, \citet{zhaoMultiitemOnlineOrder2025a}, studies a natural myopic policy that minimizes the fulfillment cost of each order in isolation in a single-FDC–RDC network. A natural extension to the multi-FDC setting is to apply the same principle and minimize the cost of each order separately. Although this generalized myopic policy may perform well in some instances, computing the optimal allocation for each order is already computationally intractable: as we show in Section~\ref{subsec: baseline algorithms}, the corresponding single-order optimization problem is NP-hard.

To overcome this difficulty, we develop a \emph{gated priority-based greedy} framework. The starting point is a priority-based greedy policy: for each item in an arriving order, the algorithm ranks distribution centers according to a prescribed cost-based priority rule and then constructs the fulfillment plan greedily. While this approach substantially improves computational tractability, a purely priority-based greedy policy can be overly short-sighted and may perform poorly in the worst case. We therefore introduce a \emph{gating condition} that determines whether the greedy allocation should be executed or the entire order should instead be routed to the RDC. If the gating condition is triggered, the algorithm fulfills the entire order solely through the RDC; otherwise, it follows the greedy allocation. The resulting framework is computationally efficient: once the gating condition and priority rankings are determined, the policy can be implemented in time proportional to the product of the number of items and the number of FDCs.

Another key advantage of this framework is its flexibility. Both the priority rule and the gating condition can be tailored to the underlying cost structure. Representative priority rules include \emph{fixed-cost-based} and \emph{variable-cost-based} rankings, which we refer to as the \textsc{Fixed-Cost} and \textsc{Variable-Cost} priority rules, respectively. Representative gating conditions include \emph{order-size threshold rules} and \emph{cost-comparison rules} that compare the greedy allocation with RDC-only fulfillment; we refer to these as the \textsc{Order-Size} and \textsc{Cost-Comparison} conditions, respectively. In particular, for the general multi-FDC setting with time-varying variable costs, we pair the \textsc{Fixed-Cost} priority rule with the \textsc{Order-Size} gating condition, under which sufficiently large orders are routed entirely to the RDC once their size exceeds a calibrated threshold. For the general multi-FDC settings with time-invariant variable costs, we instead pair the \textsc{Variable-Cost} priority rule with the \textsc{Cost-Comparison} gating condition, which routes an order to the RDC whenever the cost of the greedy allocation exceeds that of RDC-only fulfillment. As we explain in Section~\ref{sec: main results for multiple fdcs case}, these priority rules and gating conditions are simple and intuitive in the corresponding settings, while yielding fulfillment policies that are both computationally efficient and operationally interpretable.  The framework can also be extended to incorporate \emph{time-dependent priority rules}, \emph{FDC-gating conditions}, under which the entire order is routed to the sole FDC rather than to the RDC, and \emph{randomized gating conditions}. These extensions are useful for the refined treatment of the single-FDC setting presented in Section~\ref{sec: refinements for single-fdc settings}. A summary of all instantiations of the framework and its extensions, together with the corresponding priority rules and gating conditions, is provided in Table~\ref{table: specified GPG algorithms} in Section~\ref{sec:main-results-multiple-fdc-overview}.

\begin{table}
{
    \centering
    \resizebox{\textwidth}{!}
    {
    \tabulinesep=0.3mm
    \scriptsize
    \begin{tabu}{|c|c|c|c|c|}
        \hline
        \# of FDCs &\makecell{Cost\\Structure} & \makecell{Upper Bound\\on Competitive Ratio} & \makecell{Lower Bound\\on Competitive Ratio} & ${\text{UB}}/{\text{LB}}$ \\
        \hline
        \multirow{2}{*}{\makecell{Multi-FDC\\ ($K \geq 2$)}} & \makecell{Time-\\varying$^\star$} & \makecell{$O\left( \Big( \frac{f_0}{\min_{k\in[K]}f_k}\wedge  \sqrt{\frac{f_0}{a}} \Big) \vee \frac{b}{a} \right)$\\(\nsc{\multivaryingalg},\\ 
        Theorem~\ref{thm: upper bound with bounded cost})$^\dagger$} & \makecell{$\Omega\left( \Big( \frac{f_0}{\min_{k\in[K]}f_k}\wedge  \sqrt{\frac{f_0}{a}} \Big) \vee \frac{b}{a} \right)$\\~\\(Theorem~\ref{thm: lower bound with bounded cost})} &  \makecell{$\leq 6.473$\\ (Proposition~\ref{prop:ub-lb-ratio-multi-varying})} \\
        \cline{2-5}
        & \makecell{Time-\\invariant$^{\star\star}$}  &\makecell{ $\frac{f_0+\sum_{k\in[K]} f_k}{\min_{k\in[K]} f_k}\vee 2$\\(\nsc{\multiinvaralg}, \\ Theorem~\ref{thm: upper bound with time invariant cost})$^\dagger$} & \makecell{$\frac{f_0+\sum_{k\in[K]} f_k}{\min_{k\in[K]} f_k}$\\~\\(Theorem~\ref{thm: lower bound with time invariant cost})$^\dagger$} & 1 \\
        \hline    
        \multirow{3}{*}[-15pt]{\makecell{Single-FDC\\ ($K=1$)}} & \makecell{Time-\\varying$^\star$} & \makecell{ $O\left(  \Big( \frac{f_0}{f_1}\wedge \sqrt{\frac{f_0}{a}}\Big) \vee \sqrt{\frac{b}{a}}  \right)$ \\ ({\sc Better-of-Two},
        Corollary~\ref{cor:comp-ratio-upper-bound-better-of-two})} & \makecell{$\Omega\left( \Big( \frac{f_0}{f_1}\wedge \sqrt{\frac{f_0}{a}}\Big) \vee \sqrt{\frac{b}{a}}  \right)$ \\ 
        (Theorem~\ref{thm: lower bound with bounded cost k=1})} & \makecell{$\leq 19.828$\\ (Proposition~\ref{prop:ub-lb-ratio-single-varying})} \\
        \cline{2-5}    
        & \makecell{Time-\\invariant$^{\star\star}$}  & \makecell{$\begin{cases}
            1+\frac{1}{1-w+2\sqrt{1-w}},& w<\frac{\sqrt{5}-1}{2}\\
            1+w, & w\geq \frac{\sqrt{5}-1}{2}
        \end{cases}$ \\ ($w={f_0}/{f_1}$, \\ \textsc{Rand.-Cost-Comp. V-Priority}, \\ Theorem~\ref{thm: upper bound with time invariant cost k=1})} & \makecell{$ \left(1+\frac{f_0}{f_1}\right) \vee \frac{5}{4}$ \\~\\ ~\\ (Theorem~\ref{thm: lower bound with time invariant cost k=1})} & $\leq 1.123$ \\
        \cline{2-5}    
        & \makecell{Time-\\invariant$^{\star\star}$ \\ \citep{zhaoMultiitemOnlineOrder2025a}} & Myopic policy: $\left(1+\frac{f_0}{f_1}\right) \vee 2 $ &  \makecell{For fixed $f_1$: $\Omega(\sqrt{f_0})$} &  Unbounded\\
        \hline 
    \end{tabu}
    }
}

\smallskip
{\scriptsize $^\star$~: In the time-varying setting (Assumption~\ref{assumption: bounded cost}), the competitive ratio depends on the fixed costs $f_0, f_1, \dots, f_K$ and on the lower and upper bounds $a$ and $b$ of the variable costs, as defined in Eq.~\eqref{eq:def-comp-ratio-bounded-varaible-cost}.

$^{\star\star}$: In the time-invariant setting (Assumption~\ref{assumption: time invariant cost}), the competitive ratio only depends on the fixed costs $f_0, f_1, \dots, f_K$, as defined in Eq.~\eqref{eq:def-comp-ratio-time-invariant}.

$^\dagger$~: These bounds remain valid when $K=1$. However, for this special case, we prove sharper bounds, reported in the rows corresponding to $K = 1$.

\vspace{-2ex}
}
    \caption{Summary of Our Results}
    \label{table: main results in order}
\end{table}

\medskip
\noindent \underline{Provably Optimal and Order-optimal Competitive Ratios.} 
We consider two alternative assumptions on the cost structure. The \emph{time-invariant} assumption, formally stated in Assumption~\ref{assumption: time invariant cost}, is standard in the online fulfillment literature and assumes that all costs remain constant across the sequence of arriving orders. The \emph{time-varying} assumption, formally stated in Assumption~\ref{assumption: bounded cost}, allows the per-unit variable costs to vary across arriving orders. This latter assumption captures practical settings in which variable costs are destination-dependent, while the fixed cost of using a distribution center remains stable over the selling horizon. Under both assumptions, for a general number of FDCs with $K\ge 2$, we instantiate the \nsc{\gatedprioritygreedy} framework to obtain two algorithms: \nsc{\multiinvaralg}, designed for the time-invariant setting using the {\sc Variable-Cost} priority rule and the {\sc Cost-Comparison} gating condition; and \nsc{\multivaryingalg}, designed for the time-varying setting using the {\sc Fixed-Cost} priority rule and the {\sc Order-Size} gating condition. We establish upper bounds on their competitive ratios in Theorem~\ref{thm: upper bound with time invariant cost} and Theorem~\ref{thm: upper bound with bounded cost}, respectively.

The rows corresponding to $K\ge 2$ in Table~\ref{table: main results in order} summarize these bounds. The upper bound for \nsc{\multiinvaralg} exactly matches the lower bound, establishing its optimality with respect to the worst-case competitive ratio. The upper bound for \nsc{\multivaryingalg} matches the lower bound up to a factor of at most $6.473$, thereby establishing that the algorithm is order-optimal. Throughout, we view the competitive ratio as a function of the fixed costs $f_0, f_1, \dots, f_K$, while taking the worst-case ratio over the unknown order sequence and the time-varying variable costs.
In the time-varying setting, we further assume that the variable costs lie in the interval $[a,b]$, where $0 < a \le b$, as specified in Assumption~\ref{assumption: bounded cost}, and define the competitive ratio accordingly as a function of $a$ and $b$ in Eq.~\eqref{eq:def-comp-ratio-bounded-varaible-cost}. This boundedness assumption is essential: if variable costs are unbounded, then the competitive ratio is also unbounded, as shown in Theorem~\ref{thm: lower bound with bounded cost}, rendering the problem less meaningful under competitive analysis. In the time-invariant setting, by contrast, both the upper and lower bounds depend only on the fixed costs $f_0, f_1, \dots, f_K$, in accordance with the competitive-ratio definition in Eq.~\eqref{eq:def-comp-ratio-time-invariant}.

Although the algorithms described above apply to any number of FDCs $K\geq 1$, their competitive ratio guarantees may be further improved in the single-FDC setting ($K=1$). The corresponding results are reported in the first two $K=1$ rows of Table~\ref{table: main results in order}. We refine the algorithmic guarantees for both the time-varying and time-invariant variable-cost settings. These algorithms are again instantiated from extensions of our \nsc{\gatedprioritygreedy} framework; in particular, the \nsc{Better-of-Two} algorithm is a hybrid of two such instantiations. Table~\ref{table: specified GPG algorithms} in Section~\ref{sec:main-results-multiple-fdc-overview} provides a more detailed summary of their features. As a result, in the time-varying setting, the \nsc{Better-of-Two} algorithm achieves an order-optimal competitive ratio, matching the lower bound up to a constant factor. In the time-invariant setting, the \nsc{Randomized-Cost-Comparison V-Priority} algorithm matches the lower bound up to a factor of $1.123$, improving upon the factor-$2$ guarantee for the general-$K$ case.

\medskip
\noindent\underline{Competitive Ratio Lower Bounds for Any Fulfillment Policy.}
We also establish competitive ratio lower bounds for each scenario that apply to \emph{all} online fulfillment policies, including randomized ones. The corresponding results are also reported in Table~\ref{table: main results in order}. Our lower bounds are based on a unified \emph{adversarial two-instance argument} that forces any online algorithm to deplete valuable FDC inventory prematurely and thereby incur higher future costs. The key idea is to construct pairs of instances with a common prefix but different continuations, so that the algorithm must commit before it can determine whether early use of FDC inventory is beneficial or harmful. This construction relies on two mechanisms: one based on \emph{fixed-cost heterogeneity}, which induces early depletion of low-fixed-cost FDC inventory, and the other based on \emph{temporal variation} in per-unit variable costs. The latter mechanism is implemented differently across settings: when $K\ge 2$, it uses multiple FDCs with identical initial variable costs but asymmetric future evolutions, whereas when $K=1$, it sets the initial variable cost at an intermediate level to obscure whether early allocation or later preservation is preferable. The time-varying lower bound combines both mechanisms, whereas the time-invariant lower bound relies primarily on the fixed-cost construction. We then invoke \emph{Yao's minimax principle} to extend the argument to randomized algorithms.

\medskip
\noindent\underline{Improvement upon Existing Work.} For the single-FDC setting, \citet{zhaoMultiitemOnlineOrder2025a} study a  myopic policy under the assumption that the variable costs $c_0$ and $c_1$ for the RDC and the FDC, respectively, are both time- and item-independent. Their theoretical analysis yields three bounds on the competitive ratio: (1) if $f_0\leq f_1$ and $c_0>c_1$, the competitive ratio is upper bounded by $1+\left( \frac{f_0}{f_1+\alpha c_1}\vee\frac{f_1}{f_0+\alpha c_0} \right)$; (2) if $f_0>f_1, c_0>c_1$, it is upper bounded by $1+\frac{f_0}{f_1+c_1}$; (3) if $f_0>f_1, c_0\leq c_1$, it is upper bounded by $1+\left( \frac{(f_0-f_1)-(c_1-c_0)}{f_1+c_1}\vee 0 \right)$. One advantage of these results is that they capture a more fine-grained dependence on $c_0$ and $c_1$. However, if we instead consider the worst-case competitive ratio under adversarial choices of $c_0$ and $c_1$, then our guarantees improve upon theirs. Specifically, the results of \citet{zhaoMultiitemOnlineOrder2025a} imply that the competitive ratio of the myopic policy is upper bounded by $1+\left( \frac{f_0}{f_1}\vee 1 \right)=\max\left\{ 1+\frac{f_0}{f_1},\ 2 \right\}$. Compared with this guarantee, our \nsc{\multiinvaralg} algorithm attains the same worst-case competitive ratio while allowing for item-dependent variable costs. Moreover, our \nsc{\singinvaralg} algorithm, which is specifically designed for the single-FDC setting, achieves a strictly better competitive ratio when $w=\frac{f_0}{f_1}$ is small. For clarity, this comparison is also summarized in Table~\ref{table: main results in order}.

We also strengthen the lower bound on the competitive ratio achievable by randomized algorithms. \citet{zhaoMultiitemOnlineOrder2025a} show a lower bound that grows on the order of $\sqrt{f_0}$ as $f_0\to\infty$. By contrast, our lower bound in Theorem~\ref{thm: lower bound with time invariant cost k=1} shows that the competitive ratio is at least $\max\left\{ 1+\frac{f_0}{f_1},\ \frac{5}{4} \right\}$, which grows linearly with $f_0$ and matches the upper bound for \nsc{\singinvaralg} within a factor of $1.123$, showing the near-optimality of our \nsc{\singinvaralg} algorithm for the single-FDC fulfillment problem.

\subsection{Literature Review}
\label{subsec: literature review}

\medskip \noindent \underline{E-commerce fulfillment and order splitting.}
Real-time order fulfillment has been studied extensively in operations and revenue management because splitting a multi-item order across facilities can increase outbound shipping costs and degrade the customer experience. Early work emphasized implementable decision rules and re-optimization in large-scale execution systems. \citet{xuBenefitsReevaluatingRealTime2009} study real-time multi-item order fulfillment and show the value of periodically reevaluating assignments, with the objective of reducing the total number of shipments. \citet{acimovicMakingBetterFulfillment2015} develop a fulfillment heuristic based on the dual values of a transportation linear program with known demand, and \citet{andrewsPrimalDualAlgorithms2019} develop primal-dual algorithms for an industrial order fulfillment system without demand forecasting.

\medskip \noindent \underline{Provable policies for multi-item e-commerce fulfillment.}
The theoretical literature most closely related to ours studies provable fulfillment policies for online multi-item orders. The seminal work of \citet{jasinLPBasedCorrelatedRounding2015} formulates a stochastic multi-item e-commerce fulfillment problem with multiple facilities, finite inventories, fixed and variable shipping costs, and customer orders drawn from a known distribution. They combine a deterministic item-facility linear program with correlated rounding and establish asymptotic average-case performance guarantees. \citet{maOrderOptimalCorrelatedRounding2023} improves this line of work by modifying the correlated-rounding scheme and obtaining an asymptotically $(1+\ln q)$-competitive policy, where $q$ measures order size. More recently, \citet{amilMultiItemOrderFulfillment2025} revisit LP formulations through a method-based perspective, connect the fulfillment problem to prophet-inequality arguments, and derive nonasymptotic average-case guarantees. In contrast to these forecast-informed approaches, \citet{zhaoMultiitemOnlineOrder2025a} study a two-layer RDC-FDC network under adversarial arrivals and analyze the worst-case competitive ratio of a myopic policy. 

Table~\ref{tab: literature comparison} summarizes the closest theoretical papers. Existing LP-based studies primarily examine the performance of forecast-informed policies when the order distribution is known. \citet{zhaoMultiitemOnlineOrder2025a} move to adversarial arrivals, but their model is restricted to a single FDC, single-unit item requests, and variable costs that are independent of both item and time. Our model incorporates multiple FDCs, multi-unit multi-item orders, item-specific and time-varying variable costs, and adversarial fluctuations in both order arrivals and variable costs. This combination introduces methodological challenges absent from previous work: cross-facility order splitting becomes computationally more difficult, and intertemporal inventory allocation under adversarial arrivals and heterogeneous variable costs entails a substantially more intricate trade-off.

\begin{table}
    \centering
    \renewcommand{\arraystretch}{1.8}
    \scriptsize
    \resizebox{\textwidth}{!}{
    \begin{tabular}{|c|c|c|c|c|c|}
        \hline
        Model features & \cite{jasinLPBasedCorrelatedRounding2015} & \cite{maOrderOptimalCorrelatedRounding2023} & \cite{amilMultiItemOrderFulfillment2025} & \cite{zhaoMultiitemOnlineOrder2025a} & This paper \\
        \hline 
        Arriving orders & \multicolumn{3}{c|}{Stochastic, under known distribution} & \multicolumn{2}{c|}{Adversarial} \\
        \hline
        \# of DCs & \multicolumn{3}{c|}{Multiple}  & \makecell{Two-layer,\\ single FDC} & \makecell{Two-layer,\\ multiple FDCs} \\
        \hline
        \makecell{Quantity of an item\\in an order} & \multicolumn{4}{c|}{Single} & Multiple \\
        \hline
        \makecell{Variable-cost\\structure} & \multicolumn{3}{c|}{\makecell{Time-invariant,\\ item-specific}} & \makecell{Time-invariant,\\item-independent} & \makecell{Time-varying, \\item-specific} \\
        \hline
        Main policy & \multicolumn{2}{c|}{\makecell{Item-facility LP\\with rounding}} & \makecell{Method-based LP\\with rounding} & Myopic & \makecell{Gated priority-based\\ greedy}\\
        \hline
        \makecell{Performance\\guarantee}  & \multicolumn{2}{c|}{\makecell{Average-case,\\asymptotic}} & \makecell{Average-case,\\nonasymptotic} & \multicolumn{2}{c|}{\makecell{Worst-case,\\nonasymptotic}} \\
        \hline
    \end{tabular}}
    \caption{Comparison with Closely Related Theoretical Literature}
    \label{tab: literature comparison}
\end{table}

\medskip \noindent \underline{Online resource allocation.}
Our work is also connected to the broader literature on online resource allocation, which studies sequential decisions with limited resources and evaluates policies using competitive ratio or regret. Classic examples include online bipartite matching \citep{karpOptimalAlgorithmOnline1990}, online budgeted allocation and AdWords \citep{mehtaAdwordsGeneralizedOnline2007}, and the primal-dual methodology for online packing and covering problems \citep{buchbinderDesignCompetitiveOnline2009}. In operations management, robust revenue management and online booking use competitive analysis to obtain distribution-free guarantees when demand forecasts are unavailable or unreliable \citep{lanRevenueManagementLimited2008,ballRobustRevenueManagement2009}, while LP re-solving policies provide strong regret guarantees under stochastic arrival models \citep{jasinResolvingHeuristicBounded2012,bumpensantiReSolvingHeuristicUniformly2020}. More recently, \citep{heOnlineResourceAllocation2025} shows the effectiveness of primal-dual methods for online resource allocation without repeated re-solving, and recent work on online allocation or matching with reusable resources shows that greedy-like policies can be effective when resources return after stochastic service durations \citep{goyalAsymptoticallyOptimalCompetitive2025,simchi-leviGreedyLikePoliciesOnline2025}. Our problem can be viewed as a fulfillment-specific online resource allocation problem, where FDC inventories are allocated sequentially to multi-item orders. However, the presence of fulfillment split and fixed activation costs makes the model structurally different from classical online resource allocation problems.

\medskip \noindent \underline{Delay, omnichannel, replenishment, and pricing extensions.}
A related stream studies fulfillment together with other retail decisions. The benefit of delaying fulfillment decisions is studied by a series of papers \citep{weiShippingConsolidationTwo2021,xieBenefitsDelayOnline2025,zhouEcommerceOrderFulfillment2025}. Omnichannel retailing integrates online demand with store-based inventory and in-store demand \citep{govindarajanJointInventoryFulfillment2021,hubnerRevivalRetailStores2022}. Inventory replenishment has also been studied jointly with online fulfillment decisions \citep{acimovicMitigatingSpilloverOnline2017,goedhartReplenishmentFulfilmentDecisions2023,lingOnlineOrderFulfillment2026}. Another stream of work combines dynamic pricing and order fulfillment \citep{leiJointDynamicPricing2018,harshaDynamicPricingOmnichannel2019,qiuJointPricingOrdering2021}. In contrast to these richer models, we focus on a more stylized fulfillment setting that permits sharper theoretical analysis and leads to optimal or order-optimal competitive-ratio guarantees, while these extensions suggest natural directions for future work.

\section{Model Description}
\label{sec: model description}

In this section, we describe the formal setting of the model. There are $n$ distinct items, indexed by $i=1,2,\dots,n$, stored across $K$ FDCs and one RDC. The RDC is indexed by $k=0$ and has unlimited inventory for all items. The FDCs are indexed by $k=1,2,\dots,K$, and each FDC $k$ starts with an initial inventory of $I_{k,0}^i$ units of item $i$, where $I_{k,0}^i \ge 0$. FDC inventories are nonreplenishable throughout the selling horizon.

Customer orders arrive sequentially over a finite horizon of $T$ discrete periods, indexed by $t=1,2,\dots,T$. Each order may request one or more items, and the quantity requested for item $i$ in period $t$ is denoted by $S_t^i$. We allow $S_t^i$ to be any nonnegative integer, in contrast to the binary-demand assumption commonly adopted in prior work \citep{amilMultiItemOrderFulfillment2025,zhaoMultiitemOnlineOrder2025a}. Let $\boldsymbol{S}_t=(S_t^i)_{i\in[n]}$ denote the order vector in period $t$. Upon the arrival of each order, the fulfillment decision must be made immediately and irrevocably.

An order may be fulfilled using inventory from one or more FDCs and/or the RDC, and the resulting cost depends on the chosen fulfillment plan. If FDC $k$ is used to fulfill any part of an order, then a fixed cost $f_k$ is incurred, regardless of the number of items or units supplied by that FDC. In addition, fulfilling item $i$ from FDC $k$ in period $t$ incurs a per-unit variable cost $c_{k,t}^i$, which may vary over time. Similarly, if the RDC is used to fulfill any part of an order, then a fixed cost $f_0$ is incurred, together with a per-unit variable cost $c_{0,t}^i$ for each unit of item $i$ supplied by the RDC in period $t$. The objective is to determine an online fulfillment policy that minimizes the total fulfillment cost over the selling horizon.

Unlike \citet{zhaoMultiitemOnlineOrder2025a}, which assumes that variable costs are independent of both item and time, we allow the variable costs $c_{k,t}^i$ to be item-specific and time-varying. This generalization captures important features of real-world fulfillment systems, including destination-dependent shipping costs, fluctuating transportation costs, and changing labor or operating conditions. Since decisions are made online, the fulfillment policy may use only the information available up to the current period and has no access to future customer orders.

In each period $t$, after observing the arriving order $\boldsymbol{S}_t=(S_t^i)_{i\in[n]}$ and the variable costs $(c_{k,t}^i)_{k\in{0}\cup[K], i\in[n]}$, the fulfillment policy must immediately choose fulfillment quantities $m_{k,t}^i$, where $m_{k,t}^i$ denotes the quantity of item $i$ supplied from location $k$ in period $t$, with $k=0$ corresponding to the RDC. These quantities must satisfy the following constraints:
\begin{itemize}
    \item \emph{Demand fulfillment constraint.} For each item $i$ and period $t$, the total quantity supplied from all FDCs and the RDC must equal the requested quantity:
    $\sum_{k=0}^K m_{k,t}^i = S_t^i$ holds for all $i$ and $t$.
    \item \emph{Inventory constraint.} For each FDC $k$ and item $i$, the total fulfillment quantity from FDC $k$ up to time $t$ cannot exceed its initial inventory:
    $\sum_{\tau=1}^t m_{k,\tau}^i \leq I_{k,0}^i$ holds for all $k\neq 0$, $i$, and $t$.
    Let $I_{k,t}^i = I_{k,0}^i - \sum_{\tau=1}^t m_{k,\tau}^i$ be the remaining inventory of item $i$ at FDC $k$ at the end of time $t$. Then the inventory constraint can be equivalently written as
    $m_{k,t}^i \leq I_{k,t-1}^i$ for all $k\neq 0$, $i$, and $t$.
    \item \emph{Non-negativity constraint.} The fulfillment quantities must be non-negative, i.e., $m_{k,t}^i \geq 0$ holds for all $k$, $i$, and $t$.
\end{itemize}
The fulfillment cost incurred in period $t$ is
$\sum_{k=0}^K \left[ f_k \cdot \mathbb{I}\left(\sum_{i=1}^n m_{k,t}^i > 0\right) + \sum_{i=1}^n c_{k,t}^i \cdot m_{k,t}^i \right]$,
where $\mathbb{I}(\cdot)$ denotes the indicator function. Therefore, the total fulfillment cost over the entire selling horizon is
\begin{align} \label{eq:def-total-fulfillment-cost}
    \sum_{t=1}^T \sum_{k=0}^K \left[ f_k \cdot \mathbb{I}\left(\sum_{i=1}^n m_{k,t}^i > 0\right) + \sum_{i=1}^n c_{k,t}^i \cdot m_{k,t}^i \right].
\end{align}
The goal of the fulfillment policy is to minimize this total cost.

\medskip
\noindent \underline{Competitive Ratio.} We study the problem in an adversarial setting, where the initial inventories of the FDCs, the sequence of customer orders, and the corresponding variable costs are chosen by an adversary. The performance of a fulfillment policy is evaluated by its \emph{competitive ratio}, defined as the worst-case ratio between the total fulfillment cost incurred by the online policy and that incurred by an optimal offline algorithm with full knowledge of the entire sequence of customer orders and variable costs in advance. Since this worst-case ratio is taken over all admissible order sequences and variable costs, it can be viewed as a function only of the fixed costs $f_0,f_1,\ldots,f_K$.

In the fully general time-varying-variable-cost setting, however, this competitive ratio may be unbounded. We therefore focus on a bounded-variable-cost version of the problem, in which the variable costs are assumed to lie in the interval $[a,b]$ for some $0<a\le b$. Specifically, for any instance $I=\{I_{k,0}^i,~c_{k,t}^i,~S_t^i\}_{i\in[n],t\in [T], k=0,\dots,K}$, let $\mathrm{ALG}(I)$ and $\mathrm{OPT}(I)$ denote the total fulfillment costs incurred by the online algorithm and the optimal offline algorithm, respectively. We then define the competitive ratio of the algorithm as
\begin{align} \label{eq:def-comp-ratio-bounded-varaible-cost}
    \CR(\mathrm{ALG})=\CR(\mathrm{ALG}; f_0,f_1,\dots,f_K, [a,b]) \defeq \sup_{n, T}\sup_{\substack{I_{k,0}^i\\(k\in[K],i\in[n])}}\sup_{\substack{c_{k,t}^i\in [a,b]\\ (k=0,\dots, K, i\in[n],t\in[T])}}\sup_{\boldsymbol{S}_t:t\in[T]} \frac{\mathrm{ALG}(I)}{\mathrm{OPT}(I)}.
\end{align}

In the more restrictive time-invariant-variable-cost setting, by contrast, we are able to design competitive algorithms even when the variable costs are unbounded. In this case, we define the competitive ratio as
\begin{align} \label{eq:def-comp-ratio-time-invariant}
    \CR_{\mathrm{inv}}(\mathrm{ALG})=\CR_{\mathrm{inv}}(\mathrm{ALG}; f_0,f_1,\dots,f_K) \defeq \sup_{n, T}\sup_{\substack{I_{k,0}^i\\(k\in[K],i\in[n])}}\sup_{\substack{c_{k,t}^i \equiv c_k^i\\(k=0,\dots, K, i\in[n],t\in[T])}}\sup_{\boldsymbol{S}_t:t\in[T]} \frac{\mathrm{ALG}(I)}{\mathrm{OPT}(I)}.
\end{align}

When there is no ambiguity, we omit the parameters $f_0, f_1, \dots, f_K, a$, and $b$ from the competitive ratio notations in Eq.~\eqref{eq:def-comp-ratio-bounded-varaible-cost} and Eq.~\eqref{eq:def-comp-ratio-time-invariant}. A fulfillment policy is said to be \emph{$\alpha$-competitive} if its competitive ratio is at most $\alpha$, where the precise definition of the competitive ratio depends on the setting under consideration. Our goal is to design a computationally efficient online algorithm with the smallest possible competitive ratio. Throughout the paper, we use the terms ``algorithm'' and ``policy'' interchangeably.

\section{Competitive Ratio Upper \& Lower Bounds for Multi-FDC Networks}
\label{sec: main results for multiple fdcs case}

In this section, we present our algorithmic and lower bound results for the two-layer multi-item order fulfillment problem over general multi-FDC networks, that is, for all $K \geq 1$. We begin in Section~\ref{subsec: characterization of GPG policies} by formally introducing the \nsc{\gatedprioritygreedy} framework. We then present and analyze its instantiations under two specific settings: time-varying variable costs (Section~\ref{subsec: time-varying variable costs case}) and time-invariant variable costs (Section~\ref{subsec: time-invariant variable costs case}). For each setting, we also establish a complementary lower bound on the competitive ratio achievable by any online algorithm in the corresponding subsection. Before turning to the formal development, we first provide a detailed overview of the motivation and intuition behind the \gatedprioritygreedy framework and its scenario-specific instantiations in Section~\ref{sec:main-results-multiple-fdc-overview}.

\subsection{Overview of Algorithmic Intuitions}
\label{sec:main-results-multiple-fdc-overview}
\noindent \underline{Myopic Policy and Its Computational Intractability.}  \citet{zhaoMultiitemOnlineOrder2025a} studied the \emph{myopic policy} in the single-FDC setting ($K=1$), where each arriving order is fulfilled so as to minimize its current cost in isolation. Because there is only one FDC and the variable costs are identical across items, this decision reduced to a simple comparison between using the FDC and routing the entire order to the RDC, and is therefore computationally straightforward. They further showed that this policy admitted strong competitive-ratio guarantees in the single-FDC setting.

A natural generalization to the multi-FDC setting is to apply the same myopic principle and compute, for each arriving order, the fulfillment plan that minimizes the current order's cost subject to the inventory and demand constraints. Formally, at period $t$, this policy solves
\[\min_{\{m_{k,t}^i\}_{k,i}} \sum_{k=0}^K\left[ f_k\cdot\mathbb{I}\left( \sum_{i=1}^n m_{k,t}^i > 0 \right)+ \sum_{i=1}^n c_{k,t}^i \cdot m_{k,t}^i \right],\]
subject to the demand and inventory constraints. However, for general $K$, this natural extension becomes computationally intractable: the policy must determine both which subset of distribution centers to activate and how to allocate the order across them. As we show in Appendix~\ref{subsec: baseline algorithms}, the resulting optimization problem is NP-hard in general via a reduction from Set Cover. This computational intractability motivates our search for a more tractable online decision rule.

\medskip \noindent\underline{Priority-based Greedy as a Tractable Surrogate.}
To address this computational intractability, we begin with a more radical but much more tractable surrogate, which we call the \emph{priority-based greedy policy}. Before any orders arrive, the policy fixes, for each item $i$, a priority ranking $\prec_i$ over all distribution centers (DCs). Upon the arrival of an order, the policy fulfills each item greedily according to this ranking: it first uses as much inventory as possible from the highest-priority DC, then moves to the next DC in the ranking, and continues until the demand is fully satisfied. Since the RDC has infinite inventory, every order can always be fulfilled.
This priority-based greedy rule is computationally simple, requiring only $O(nK)$ time per order, and therefore offers a dramatic improvement over the exponential complexity of the myopic approach. The key remaining question is whether one can design the priority rankings so that such a simple greedy rule also achieves strong worst-case performance guarantees.

\medskip \noindent \underline{Why Pure Priority-based Greedy is Not Enough.}
Despite its computational appeal, a \emph{purely} priority-based greedy policy can perform poorly in the worst case because it is too short-sighted. In particular, by greedily exploiting currently attractive FDC inventory, it may deplete scarce FDC stock too early and thereby sacrifice substantial future savings. The following example illustrates this phenomenon.
\begin{example}[Failure of purely priority-based greedy]
\label{example: pure greedy failure}
We consider a simple instance with one FDC, one RDC, and a single item. For any integer $M \geq 1$, consider a problem instance with horizon $T=M+1$. The fixed costs are $f_0=1$ and $f_1=0$, and the variable costs satisfy $c_{0,t}=c_0=1/M$ and $c_{1,t}=c_1=1/M$ for all $t\in[T]$. The initial FDC inventory is $I_{1,0}=M$. The order sequence is given by $S_1^1=M$ and $S_2^1=S_3^1=\cdots=S_{M+1}^1=1$. We now consider the following two priority rankings for the only item $1$:
\begin{itemize}
\item \underline{$\mathrm{FDC}\prec_1 \mathrm{RDC}$.} The greedy policy uses the FDC to fulfill the first order, incurring cost $f_1 + S_1^1 \cdot c_1 = 0 + M \cdot (1/M) = 1$. The FDC inventory is then exhausted, so all remaining orders must be fulfilled from the RDC, incurring an additional cost of $(T-1) \cdot (f_0 + 1 \cdot c_0) = M(1 + 1/M) = M+1$. Hence, the total cost is $1 + (M+1) = M+2$. 
\item \underline{$\mathrm{RDC}\prec_1 \mathrm{FDC}$.} Since the RDC is never depleted, The greedy policy uses the RDC to fulfill all orders. The resulting cost is $(f_0 + S_1^1 \cdot c_0) + (T-1)\cdot  (f_0 + 1 \cdot c_0) = 1+M\cdot(1/M)+ M \cdot (1 + 1 \cdot (1/M)) = M + 3$.  
\end{itemize}
By contrast, the optimal offline policy fulfills the first order from the RDC, incurring cost $f_0+M\cdot c_0=1+M\cdot(1/M) = 2$, and reserves the FDC inventory for the later small orders. The remaining orders are then fulfilled from the FDC at total cost $(T-1)\cdot (f_1+1 \cdot c_1)=M \cdot (0 + 1 \cdot (1/M)) = 1$. Thus, the total cost incurred by the offline policy is $2 + 1 = 3$.

Therefore, under either priority ranking, the purely priority-based greedy algorithm incurs a competitive ratio of $\Omega(M)$. By contrast, our \nsc{\multivaryingalg} also applies to this instance, and Theorem~\ref{thm: upper bound with bounded cost} guarantees a competitive ratio of $O(\sqrt{M})$. Hence, the purely priority-based greedy algorithm can be worse by a factor of $\Omega(\sqrt{M})$, where $M$ can be arbitrarily large.
\end{example} 

This example highlights the central weakness of the pure priority-based greedy algorithm. Under the ranking $\mathrm{FDC}\prec_1 \mathrm{RDC}$, the algorithm uses FDC inventory too aggressively, without accounting for its future value. Under the ranking $\mathrm{RDC}\prec_1 \mathrm{FDC}$, the algorithm goes to the opposite extreme and fails to use the FDC inventory at all.

\medskip \noindent\underline{From Priority-based Greedy to \gatedprioritygreedy.}
To address this weakness, we augment the priority-based greedy rule with a \emph{gating condition}. The resulting policy first determines whether the arriving order should be handled greedily using all DCs or instead be routed entirely to the RDC. If the gating condition is triggered, the order is fulfilled solely through the RDC; otherwise, the order is fulfilled according to the priority-based greedy allocation. We refer to this general class of policies as the \emph{gated priority-based greedy} framework.
The logic is simple: the priority rule determines \emph{how} an order should be allocated when FDC inventory is worth using, while the gating condition determines \emph{whether} FDC inventory should be used at all. In this way, the framework preserves the computational tractability of greedy allocation while protecting against the most harmful forms of short-sighted inventory depletion or excessive fixed-cost expenditure.

\medskip
\noindent\underline{Scenario-specific Instantiations.}
A key advantage of the \gatedprioritygreedy framework is its flexibility: both the priority rule and the gating condition can be tailored to the underlying cost structure. The appropriate design depends on which component of the cost structure primarily drives the worst-case competitive ratio in each setting.

\begin{table}[htbp]
    \centering
    \tabulinesep=0.6mm
    \footnotesize
    \begin{tabu}{|c|c|c|c|c|c|}
        \hline
        \# of FDCs & \makecell{Cost\\ Structure} & Framework & Algorithm & Priority Rule & Gating Condition \\
        \hline
         \multirow{2}{*}{Multi-FDC} & \makecell{Time-\\varying\\(Assump.~\ref{assumption: bounded cost})} & \multirow{2}{*}{Original} & \makecell{\nsc{Order-Size F-Pri.}\\(Algorithm~\ref{alg: multi vary alg})} & {\textsc{Fixed-Cost}} & \textsc{Order-Size}  \\
        \cline{2-2}\cline{4-6}
       & \makecell{Time-\\invariant\\(Assump.~\ref{assumption: time invariant cost})} & &  \makecell{\nsc{Cost-Comp.~V-Pri.}\\(Algorithm~\ref{alg: multi invar alg})} & \textsc{Variable-Cost} & \textsc{Cost-Comparison} \\
        \hline
        \multirow{3}{*}{Single-FDC} & \multirow{2}{*}{\makecell{Time-\\varying\\(Assump.~\ref{assumption: bounded cost})}} & \makecell{Extension w/\\ time-dependent\\ priority rules (\ding{72})}  & \makecell{\nsc{Cost-Comp.~AdjV-Pri.} \\(Algorithm~\ref{alg: sing vary alg})} & \makecell{\textsc{Adjusted-}\\ \textsc{Variable-Cost}} & \textsc{Cost-Comparison}  \\
        \cline{3-6}
        & & \makecell{Extend (\ding{72})  \\ further w/ \\ FDC-gating cond.} &  \makecell{\textsc{Order-Size AdjV-Pri.} \\(Algorithm~\ref{alg: sing vary alg improved})} & \makecell{\textsc{Adjusted-}\\ \textsc{Variable-Cost}} & \makecell{\textsc{Order-Size}\\ FDC-gating cond.}  \\
        \cline{2-6}
        & \makecell{Time-\\invariant\\(Assump.~\ref{assumption: time invariant cost})} & \makecell{Extension w/ \\ randomized\\ gating condition} &  \makecell{\textsc{Rand.-Cost-Comp.}\\\textsc{V-Pri.} \\(Algorithm~\ref{alg: sing invar alg})} & \textsc{Variable-Cost} & \makecell{\textsc{Randomized-}\\ \textsc{Cost-Comparison}}  \\
        \hline
    \end{tabu}
    \caption{Summary of the Scenario-specific Instantiations of the \nsc{\gatedprioritygreedy} Framework and its Extensions.}
    \label{table: specified GPG algorithms}
\end{table}

\begin{itemize}
\item \textit{Time-varying variable costs (Assumption~\ref{assumption: bounded cost}).} In this setting, the variable costs are bounded in the range of $[a,b]$, where $a$ and $b$ are fixed parameters. The effect of variable-cost fluctuations on the competitive ratio is therefore controlled by the ratio $b/a$. Thus, the main source of worst-case loss comes from how fixed activation costs are incurred. Accordingly, we use the \textsc{Fixed-Cost} priority, which ranks DCs in nondecreasing order of fixed cost and gives higher priority to DCs that are cheaper to activate. To prevent the greedy rule from depleting FDC inventory too aggressively, we combine this priority rule with an \textsc{Order-Size} gating condition: when the total size of an incoming order exceeds a threshold $\theta$, the entire order is routed to the RDC. Intuitively, large orders can better amortize the RDC fixed cost and are therefore less valuable targets for scarce FDC inventory, whereas small orders benefit more from the fixed-cost savings of local fulfillment. This yields the \nsc{\multivaryingalg} algorithm; see Section~\ref{subsec: time-varying variable costs case}.

\item \textit{Time-invariant variable costs (Assumption~\ref{assumption: time invariant cost}).} In this setting, the variable costs $c_k^i$ are constant over time but may be arbitrarily large relative to the fixed costs. Thus, the main source of worst-case loss comes from how variable costs are incurred. Accordingly, we use the \textsc{Variable-Cost} priority rule, which ranks DCs in nondecreasing order of variable cost for each item. The main risk in this setting is not over-depletion of FDC inventory, but rather excessive activation of multiple FDCs, each of which incurs its own fixed cost. We therefore introduce a \textsc{Cost-Comparison} gating condition: the greedy allocation is used only when its total cost does not exceed the cost of fulfilling the same order entirely through the RDC. Otherwise, the order is routed to the RDC. This yields the \nsc{\multiinvaralg} algorithm; see Section~\ref{subsec: time-invariant variable costs case}.

\item \textit{Refinements for single-FDC settings.}  For single-FDC networks, we design algorithms with sharper competitive ratio guarantees based on extensions of the \nsc{\gatedprioritygreedy} framework. These extensions allow for \emph{time-dependent priority rules}, \emph{FDC-gating conditions}, under which the entire order is routed to the sole FDC rather than to the RDC, and \emph{randomized gating conditions}. In the time-varying variable cost setting, we instantiate the extended framework using the \nsc{Adjusted-Variable-Cost} priority rule, which generalizes the \nsc{Variable-Cost} rule by reweighting the variable costs through an additional parameter, together with the \nsc{Cost-Comparison} gating condition and the \nsc{Order-Size} FDC-gating condition. This yields the \nsc{Cost-Comparison AdjV-Priority} algorithm and the \nsc{Order-Size AdjV-Priority} algorithm; see Section~\ref{subsec: time-varying variable costs case single fdc}. In the time-invariant variable cost setting, we introduce the \nsc{Randomized-Cost-Comparison} gating condition. Combined with the \nsc{Variable-Cost} priority rule, this produces the \nsc{Randomized-Cost-Comparison V-Priority} algorithm; see Section~\ref{subsec: time-invariant variable costs case single fdc}.
\end{itemize}

Table~\ref{table: specified GPG algorithms} summarizes these instantiations. The following subsections formalize the framework and establish the competitive guarantees of the resulting algorithms.

\subsection{Formal Introduction of the \nsc{\gatedprioritygreedy} Framework}
\label{subsec: characterization of GPG policies}

We now formalize the class of \nsc{\gatedprioritygreedy} fulfillment policies that captures the ideas described above. The framework is summarized in Algorithm~\ref{alg: GPG policy}. For each item $i$, the policy maintains a priority ranking, that is, a total order $\prec_i$ over all distribution centers (DCs). It also maintains a gating condition $G(\{f_k,c_k^i,S^i,\hat{m}_k^i\})$ which is a Boolean-valued function of the fixed costs, the current-period variable costs, the current-period order quantities, and the pure greedy fulfillment plan induced by the priority rankings. Upon the arrival of an order $\boldsymbol{S}_t$, the policy first computes the greedy fulfillment plan $\{\hat{m}_{k,t}^i\}_{k,i}$ induced by $\prec_i$ (Line~\ref{line:GPG-policy-compute-greedy-plan}), and then evaluates the gating condition $G(\{f_k,c_{k,t}^i,S_t^i,\hat{m}_{k,t}^i\})$. If the gating condition is triggered, the order is fulfilled entirely through the RDC; otherwise, it is fulfilled according to the priority-based greedy plan.

\begin{algorithm}[ht]
    \caption{\nsc{\gatedprioritygreedy}}
    \label{alg: GPG policy}
    \begin{algorithmic}[1]
        \State \textbf{Input:} A priority ranking $\prec_i$ among all DCs and a gating condition $G(\{f_k,c_k^i,S^i,\hat{m}_k^i\})\in\{0,1\}$;
        \For{each time period $t=1,2,\dots,T$}
            \State Observe customer order $\boldsymbol{S}_t$ and variable costs $\{c_{k,t}^i\}_{\substack{k=0,\dots,K\\i=1,\dots,n}}$;
            \State Compute the pure greedy fulfillment quantities according to the priority ranking $\prec_i$:
            \[\hat{m}_{k,t}^i = \min\Big\{ \Big( S_{t}^i-\sum_{\substack{k'\in[K]:k'\prec_i k}} I_{k',t-1}^i \Big)^+, I_{k,t-1}^i \Big\}\mathbb{I}\left( k\prec_i 0 \right),\ \forall\,k\in[K],i\in[n],\]
            \[\hat{m}_{0,t}^i = S_t^i - \sum_{k\in[K]} \hat{m}_{k,t}^i,\ \forall\,i\in[n];\] \label{line:GPG-policy-compute-greedy-plan}
            \If{gating condition $G(\{f_k,c_{k,t}^i,S_t^i,\hat{m}_{k,t}^i\})=1$} \label{line:gpg-5}
                \State $m_{0,t}^i \gets S_t^i$ and $m_{k,t}^i \gets 0$ for all $k \in [K]$, $i \in [n]$; \label{line:gpg-6} \Comment{route entire order to RDC}
            \Else
                \State $m_{k,t}^i \gets \hat{m}_{k,t}^i$ for all $k \in \{0,\dots,K\}$, $i \in [n]$; \Comment{follow priority-based greedy}
            \EndIf
            \State Execute fulfillment plan $\left\{ m_{k,t}^i \right\}$, and update inventory levels: $I_{k,t}^i \gets I_{k,t-1}^i - m_{k,t}^i$, $\forall k, i$.
        \EndFor
    \end{algorithmic}
\end{algorithm}

For a pure priority-based greedy algorithm, higher-priority FDCs are always used as much as possible before lower-priority FDCs are considered. As a gated extension, \nsc{\gatedprioritygreedy} preserves this property during periods when the pure greedy allocation is selected. The following lemma formalizes this observation and plays a central role in the analysis of each scenario-specific instantiation of the framework.

\begin{lemma}
    \label{lemma: key lemma of GPG policy}
    Fix an item $i$. Without loss of generality, suppose that the total order $\prec_i$ on $\{0,\dots, K\}$ is $1 \prec_i 2 \prec_i \dots \prec_i k_0 \prec_i 0 \prec_i k_0+1 \prec_i \dots \prec_i K$, so that FDC $1$ has the highest priority. For any period $t$, let $\{\hat{m}_{k,t}^i\}_{k=0}^K$ denote the pure greedy fulfillment plan generated by Algorithm~\ref{alg: GPG policy}. Define $A=\{t:\sum_{k=1}^K {m}_{k,t}^i>0\}$, namely, the set of periods in which the actual fulfillment plan uses at least one FDC to fulfill demand for item $i$. Then for any set $B\supseteq A$, we have
    \[\sum_{k=1}^j\sum_{t\in B}\hat{m}_{k,t}^i\geq \sum_{k=1}^j\sum_{t\in B}m_{k,t}^{i,*},\qquad\qquad \forall~j=1,2,\dots, k_0.\]
\end{lemma}

The key step in proving the lemma is the following chain of inequalities:
\[\sum_{k=1}^j\sum_{t\in B}\hat{m}_{k,t}^i\geq \min\left\{ \sum_{t\in B}S_t^i,~\sum_{k=1}^jI_{k,0}^i \right\}\geq \sum_{k=1}^j\sum_{t\in B}m_{k,t}^{i,*}.\]
The second inequality follows directly from the demand-fulfillment and inventory constraints. To see the first inequality, observe that, in each period $t$, the pure greedy quantities $\hat{m}_{k,t}^i$ are generated according to the priority order $\prec_i$. Hence, for the FDCs $\{1,\dots,j\}$, either the demand for item $i$ is fully covered, or all available inventory at these FDCs is allocated. In addition, for every period $t\in A$, the \nsc{\gatedprioritygreedy} algorithm coincides with the pure greedy plan. Together, these observations capture the main idea behind the proof of the first inequality. A detailed proof is provided in Section~\ref{subsec: proof of lemma} of the e-companion.

\subsection{Scenario-specific Instantiation: Time-varying Variable Costs}
\label{subsec: time-varying variable costs case}

We first consider the setting in which the variable costs $c_{k,t}^i$ may vary over time. To analyze this case, we impose the following boundedness assumption. As we show later, this assumption is essentially necessary for obtaining a finite competitive ratio for any online fulfillment policy.

\begin{assumption}
    \label{assumption: bounded cost}
    There exist constants $b>a>0$ such that $a\leq c_{k,t}^i \leq b$ for all $k,i,t$.
\end{assumption}

This assumption is natural in practice: although per-unit variable costs may vary over time because of destination mix and changing operating conditions, they are typically bounded away from both zero and infinity over a fixed service region and planning horizon. More importantly, as shown later in Theorem~\ref{thm: lower bound with bounded cost}, the minimax competitive ratio necessarily depends on the scale factor $b/a$. Hence, if $b\to\infty$ or $a\to 0^+$, the competitive ratio diverges. Thus, without Assumption~\ref{assumption: bounded cost}, no online policy can admit a finite worst-case guarantee. At the same time, the assumption still permits rich heterogeneity, since the costs remain item-specific, facility-specific, and time-varying.

Under Assumption~\ref{assumption: bounded cost}, we instantiate the \nsc{\gatedprioritygreedy} framework using the \textsc{Fixed-Cost} priority rule, which ranks DCs in nondecreasing order of fixed cost $f_k$, together with the \textsc{Order-Size} gating condition, which is triggered whenever the order size exceeds a prescribed threshold $\theta$. The formal definitions of the algorithm, the priority rule, and the gating condition are given in Algorithm~\ref{alg: multi vary alg}.

\begin{algorithm}[ht]
    \caption{\nsc{\multivaryingalg} (Input parameter: $\theta \geq 0$)}
    \label{alg: multi vary alg}
    Instantiate \nsc{\gatedprioritygreedy} using:
    \begin{itemize}
    \item \underline{\textsc{Fixed-Cost} priority}: let $\prec_i=\prec_\mathrm{F}~ \forall i \in [n]$ where
\[k\prec_\mathrm{F}j~~\Leftrightarrow~~f_k<f_j\text{ or } f_k=f_j,~k<j,\quad \forall~k,j\in\{0,\dots,K\};\]
    \item \underline{\textsc{Order-Size} gating condition} (with parameter $\theta$):
    \[G\left( \left\{ f_k,c_{k,t}^i,S_t^i,\hat{m}_{k,t}^i \right\}_{\substack{k=0,\dots,K\\i=1,\dots,n}} \right)=\mathbb{I}\left( \sum_{i=1}^nS_t^i>\theta \right).\]
    \end{itemize}
\end{algorithm}

As discussed in Section~\ref{sec:main-results-multiple-fdc-overview}, under the bounded variable cost assumption, the primary source of worst-case loss arises from the fixed costs. This observation motivates the use of the \textsc{Fixed-Cost} priority. At the same time, the \textsc{Order-Size} gating condition, parameterized by the threshold $\theta$, prevents the FDC inventory from being depleted too aggressively. The following Theorem~\ref{thm: upper bound with bounded cost} establishes an upper bound on the competitive ratio of the \nsc{\multivaryingalg} algorithm. 

As can be seen from the proof, the threshold $\theta$ plays a balancing role. A larger value of $\theta$ offers stronger protection against premature depletion of the FDC inventory: even if our algorithm is later forced to rely on the RDC under the pure priority-based greedy policy, the corresponding fixed cost is amortized over at least $\theta$ items in the order. By contrast, a smaller value of $\theta$ helps limit the competitive ratio loss incurred by the ``all-RDC'' fulfillment plan to at most $\theta$, as established by our key technical Lemma~\ref{lemma: key lemma of GPG policy}.

\begin{theorem}
    \label{thm: upper bound with bounded cost}
    Under Assumption~\ref{assumption: bounded cost}, the competitive ratio of \nsc{\multivaryingalg} defined in Algorithm \ref{alg: multi vary alg} satisfies
    \[\CR(\nsc{\multivaryingalg}) \leq \max\left\{ \theta,\ \frac{f_0+b\theta}{\min_{k\in[K]}f_k+a\theta},\ \frac{b}{a} \right\}.\]
    In particular, if we choose $\theta=\sqrt{\frac{f_0}{a}+\frac{(\min_{k\in[K]}f_k-b)^2}{4a^2}}-\frac{\min_{k\in[K]}f_k-b}{2a}$, then
    \[\CR(\nsc{\multivaryingalg})\leq \max\left\{ \sqrt{\frac{f_0}{a}+\frac{(\min_{k\in[K]}f_k-b)^2}{4a^2}}-\frac{\min_{k\in[K]}f_k-b}{2a},\ \frac{b}{a} \right\}.\]
\end{theorem}
\begin{proof}{Proof.}
Let $A = \{t: \sum_{i=1}^n S_t^i \leq \theta\}$ be the set of time periods that the total number of items in the order is at most $\theta$ (i.e., when the algorithm follows the pure priority-based greedy policy). Let $\{m_{k,t}^{i,*}\}_{k=0}^K$ denote the fulfillment plan of the optimal offline policy. For simplicity, we also denote $f=\min_{k\in[K]}f_k$ as the minimal fixed cost among all FDCs. The total cost incurred by the \nsc{\multivaryingalg} fulfillment policy can be upper bounded as follows:
\begin{align}
    (\text{cost incurred by~}& \nsc{\multivaryingalg})  = \sum_{t=1}^T \sum_{k=0}^K \left[ f_k \cdot \mathbb{I}\left(\sum_{i=1}^n m_{k,t}^i > 0\right) + \sum_{i=1}^n c_{k ,t}^i \cdot m_{k,t}^i \right] \notag \notag \\
    & \qquad\qquad  \leq \sum_{t\in A} \left[\sum_{k=0}^K f_k \cdot \mathbb{I}\left(\sum_{i=1}^n m_{k,t}^i > 0\right) + b\sum_{i=1}^n S_t^i \right] +\sum_{t\not\in A} \left[ f_0 + b\sum_{i=1}^n S_t^i \right].\label{eq:proof-os-fp-alg-cost-ub}
\end{align}
For the optimal offline algorithm, we lower bound its total cost as follows:
\begin{align}
    \mathrm{OPT} & = \sum_{t=1}^T \sum_{k=0}^K \left[ f_k \cdot \mathbb{I}\left(\sum_{i=1}^n m_{k,t}^{i,*} > 0\right) + \sum_{i=1}^n c_{k ,t}^i \cdot m_{k,t}^{i,*} \right] \notag\notag\\
    & \geq \sum_{t\in A} \left[\sum_{k=0}^K  f_k \cdot \mathbb{I}\left(\sum_{i=1}^n m_{k,t}^{i,*} > 0\right) + a \sum_{i=1}^n S_t^i \right]  + \sum_{t\notin A} \left[ \min\{f,f_0\} + a\sum_{i=1}^n S_t^i \right]. \label{eq:proof-os-fp-opt-cost-lb} 
\end{align}
Comparing Eq.~\eqref{eq:proof-os-fp-alg-cost-ub} and Eq.~\eqref{eq:proof-os-fp-opt-cost-lb}, for the terms corresponding to $t\notin A$, we have
\begin{align}
\sum_{t\not\in A} \left[f_0+b\sum_iS_t^i\right] \leq \max\left\{ \frac{b}{a},\ \frac{f_0+b\theta}{f+a\theta} \right\} \cdot \sum_{t\not\in A} \left[\min\{f,f_0\}+a\sum_iS_t^i\right] . \label{eq:proof-os-fp-part-nA}
\end{align}
For the terms corresponding to $t\in A$, we will prove that the cost incurred by \nsc{\multivaryingalg} is at most $\theta$ times the offline optimum. For each $t \in A$, by $\sum_{i=1}^n S_t^i \leq \theta$, we have $\sum_{i=1}^n m_{k,t}^{i,*} \leq \theta\cdot \mathbb{I}\left(\sum_{i=1}^n m_{k,t}^{i,*} > 0\right)$ holds for each DC $k \in [K] \cup \{0\}$. Therefore, 
\begin{align}
\sum_{t\in A}\sum_{k=0}^K\left(f_k\cdot \sum_{i=1}^n m_{k,t}^{i,*} \right) \leq \theta\cdot \sum_{t\in A}\sum_{k=0}^K\left[ f_k\cdot \mathbb{I}\left(\sum_{i=1}^n m_{k,t}^{i,*}>0 \right) \right].  \label{eq:proof-os-fp-1}
\end{align}
On the other hand, we prove the following claim after the proof of this theorem.
\begin{claim}\label{claim:proof-os-fp-tool}
For each item $i\in[n]$, we have
$\sum_{t\in A}\sum_{k=0}^K\left(f_k\cdot m_{k,t}^i \right) \leq \sum_{t\in A}\sum_{k=0}^K\left(f_k\cdot m_{k,t}^{i,*} \right)$.
\end{claim}
Claim~\ref{claim:proof-os-fp-tool} implies that
\begin{equation}
    \sum_{t\in A}\sum_{k=0}^K\left(f_k\cdot \sum_{i=1}^n m_{k,t}^i \right) \leq \sum_{t\in A}\sum_{k=0}^K\left(f_k\cdot \sum_{i=1}^n m_{k,t}^{i,*} \right).
    \label{eq:proof-os-fp-2}
\end{equation}
Combining Eq.~\eqref{eq:proof-os-fp-1}, Eq.~\eqref{eq:proof-os-fp-2}, we have
\begin{align}
&\sum_{t\in A} \left[\sum_{k=0}^K f_k \cdot \mathbb{I}\left(\sum_{i=1}^n m_{k,t}^i > 0\right) + b\sum_{i=1}^n S_t^i \right] \leq \sum_{t\in A} \left[\sum_{k=0}^K f_k \cdot \sum_{i=1}^n m_{k,t}^i + b\sum_{i=1}^n S_t^i \right] \notag \\
&\qquad\qquad \leq \sum_{t\in A} \left[\sum_{k=0}^K f_k \cdot \sum_{i=1}^n m_{k,t}^{i,*} + b\sum_{i=1}^n S_t^i \right] \leq \sum_{t\in A} \left[\theta \cdot \sum_{k=0}^K f_k \cdot \mathbb{I}\left(\sum_{i=1}^n m_{k,t}^{i,*}>0 \right)  + b\sum_{i=1}^n S_t^i \right] \notag \\
&\qquad\qquad\leq \max\left\{\theta, \frac{b}{a}\right\} \cdot \left[\sum_{k=0}^K f_k \cdot \mathbb{I}\left(\sum_{i=1}^n m_{k,t}^{i,*}>0 \right)  + a\sum_{i=1}^n S_t^i \right]. \label{eq:proof-os-fp-part-A}
\end{align}
Combining Eq.~\eqref{eq:proof-os-fp-alg-cost-ub}, Eq.~\eqref{eq:proof-os-fp-opt-cost-lb}, Eq.~\eqref{eq:proof-os-fp-part-nA}, and Eq.~\eqref{eq:proof-os-fp-part-A}, we prove the theorem.
\hfill\Halmos
\end{proof}

\begin{proof}{Proof of Claim~\ref{claim:proof-os-fp-tool}.}
Fix an item $i \in [n]$. Without loss of generality, we assume that $f_1\leq \dots \leq f_{k_0} < f_0 \leq f_{k_0 + 1} \leq \dots \leq f_K$.  Note that
\begin{align}
& \sum_{t\in A}\left( \sum_{k=0}^K f_k\cdot m_{k,t}^i \right) - \sum_{t\in A}\left( \sum_{k=0}^K f_k\cdot m_{k,t}^{i,*} \right)
=\ \sum_{t\in A}\left[ f_0\cdot \left(m_{0,t}^i - m_{0,t}^{i,*}\right)+\sum_{k=1}^K f_k \cdot \left( m_{k,t}^i - m_{k,t}^{i,*} \right) \right] \notag\\
&\qquad =\ \sum_{t\in A}\left[ -f_0\cdot \sum_{k=1}^K\left( m_{k,t}^i - m_{k,t}^{i,*} \right) +\sum_{k=1}^K f_k \cdot \left( m_{k,t}^i - m_{k,t}^{i,*} \right) \right] =\ \sum_{k=1}^K \left( f_k - f_0 \right) \cdot \sum_{t\in A}\left( m_{k,t}^i - m_{k,t}^{i,*} \right) \notag\\
&\qquad \leq\  \sum_{k=1}^{k_0}\left( f_k - f_0 \right) \cdot \sum_{t\in A}\left( m_{k,t}^i - m_{k,t}^{i,*} \right) =  \sum_{k=1}^{k_0}\left( f_k - f_0 \right) \cdot \sum_{t\in A}\left( \hat{m}_{k,t}^i - m_{k,t}^{i,*} \right), \label{eq:claim-proof-os-fp-tool-1}
\end{align}
where the inequality is because $f_k\geq f_0$ and $m_{k,t}^i=0$ for all $k>k_0$, by the \nsc{\multivaryingalg} algorithm, and the last equality is because the algorithm follows the pure priority-based greedy policy during periods $t \in A$.
We further have
\begin{align}
    \text{Eq.~\eqref{eq:claim-proof-os-fp-tool-1}} =  \sum_{j=1}^{k_0-1}\left( f_j-f_{j+1} \right)\cdot \sum_{k=1}^j \sum_{t\in A}\left( \hat{m}_{k,t}^i - m_{k,t}^{i,*} \right)+\left( f_{k_0} - f_0 \right) \cdot \sum_{k=1}^{k_0} \sum_{t\in A}\left( \hat{m}_{k,t}^i - m_{k,t}^{i,*} \right) \leq 0, \label{eq:claim-proof-os-fp-tool-2}
\end{align}
Where the inequality is due to that $f_j\leq f_{j+1} ~\forall j \in \{1,\dots,k_0-1\}$, $f_{k_0} \leq f_0$, and that $\sum_{k=1}^j \sum_{t\in A}\left( \hat{m}_{k,t}^i - m_{k,t}^{i,*} \right) \geq 0$ holds for all $j\in \{1,\dots,k_0\}$ (which is guaranteed by Lemma~\ref{lemma: key lemma of GPG policy}, as $A$ contains all time periods in which the policy uses at least one FDC to fulfill the order).

Combining Eq.~\eqref{eq:claim-proof-os-fp-tool-1} and Eq.~\eqref{eq:claim-proof-os-fp-tool-2}, we prove the claim.
\hfill\Halmos
\end{proof}

\medskip \noindent \textbf{Lower Bound.}
The next theorem establishes a lower bound on the competitive ratio of any online fulfillment policy under Assumption~\ref{assumption: bounded cost} for the case $K\geq 2$. The full proof is deferred to Section~\ref{subsec: proof of lower bound with time varying cost} of the e-companion. The corresponding lower bound for the case $K=1$ will be presented later in Theorem~\ref{thm: lower bound with bounded cost k=1}.
\begin{theorem}
\label{thm: lower bound with bounded cost}
Under Assumption~\ref{assumption: bounded cost}, for any $K\geq 2$ and any given fixed costs $\{f_k\}_{k=0}^K$, every online fulfillment policy satisfies
$\CR(\mathrm{ALG}) \geq \max\left\{1,\ \frac{b}{4a},\ \frac{1}{4}\max_{n\geq 2}\min\left\{ n,\ \frac{f_0}{\min_{k\in[K]}f_k+na}\right\} \right\}$.
\end{theorem}

We now compare the above lower bound with the upper bound established in Theorem~\ref{thm: upper bound with bounded cost}. The following proposition shows that the two bounds match up to a constant factor, which implies the order-optimality of our \nsc{\multivaryingalg} algorithm. The proof is straightforward, and is deferred to Section~\ref{subsec: proof-of-ub-lb-ratio-multi-varying} of the e-companion.

\begin{proposition}
    \label{prop:ub-lb-ratio-multi-varying}
    The upper bound in Theorem~\ref{thm: upper bound with bounded cost} and the lower bound in Theorem~\ref{thm: lower bound with bounded cost} match up to a constant factor of $2(1+\sqrt{5})\leq 6.473$.
\end{proposition}

\subsection{Scenario-specific Instantiation: Time-invariant Variable Costs}
\label{subsec: time-invariant variable costs case}

In this section, we instantiate the \nsc{\gatedprioritygreedy} framework for the setting with time-invariant variable costs, as formalized in the following assumption.
\begin{assumption}
    \label{assumption: time invariant cost}
    The variable costs are time-invariant; that is, $c_{k,t}^i \equiv c_k^i$ for all $k$, $i$, and $t$.
\end{assumption}

This assumption is widely adopted in prior work \citep{jasinLPBasedCorrelatedRounding2015,zhaoMultiitemOnlineOrder2025a,amilMultiItemOrderFulfillment2025}. Under this setting, the variable costs are known to the online algorithm in advance. We instantiate the \nsc{\gatedprioritygreedy} framework using the \nsc{Variable-Cost} priority rule, which ranks DCs for each item $i$ in nondecreasing order of variable cost $c_k^i$, together with the \textsc{Cost-Comparison} gating condition, which is triggered whenever the total cost of fulfilling the order by priority-based greedy allocations exceeds the cost of fulfilling the entire order through the RDC. The formal definitions of the algorithm, the priority rule, and the gating condition are provided in Algorithm~\ref{alg: multi invar alg}.

\begin{algorithm}[ht]
    \caption{\nsc{\multiinvaralg}}
    \label{alg: multi invar alg}
    Instantiate \nsc{\gatedprioritygreedy} using:
    \begin{itemize}
    \item \underline{\textsc{Variable-Cost} priority}: for $\forall i \in [n]$, define $\prec_i$ as
    $k\prec_i j \Leftrightarrow c_k^i<c_j^i\text{ or } c_k^i=c_j^i,~k<j$; 
    \item \underline{\textsc{Cost-Comparison} gating condition}:
    \[G\left( \left\{ f_k,c_k^i,S_t^i,\hat{m}_{k,t}^i \right\}_{\substack{k=0,\dots,K\\i=1,\dots,n}} \right)=\mathbb{I}\left( \sum_{k=0}^K\left[ f_k\cdot \mathbb{I}\left(\sum_{i=1}^n \hat{m}_{k,t}^i > 0\right) + \sum_{i=1}^n c_k^i \cdot \hat{m}_{k,t}^i \right] > f_0 + \sum_{i=1}^n c_0^i \cdot S_t^i \right).\]
    \end{itemize}
\end{algorithm}

As discussed in Section~\ref{sec:main-results-multiple-fdc-overview}, under the time-invariant variable cost assumption, the main source of worst-case loss comes from the variable costs. This observation motivates the use of the \textsc{Variable-Cost} priority rule. At the same time, the \textsc{Cost-Comparison} gating condition prevents excessive activation of FDCs and the resulting large fixed costs. The following Theorem~\ref{thm: upper bound with time invariant cost} establishes an upper bound on the competitive ratio of the \nsc{\multiinvaralg} algorithm.

\begin{theorem}
    \label{thm: upper bound with time invariant cost}
    Under Assumption~\ref{assumption: time invariant cost}, the competitive ratio of \nsc{\multiinvaralg} defined in Algorithm \ref{alg: multi invar alg} satisfies
    \[\CR_{\mathrm{inv}}(\nsc{\multiinvaralg}) \leq \max\left\{ \frac{f_0+\sum_{k\in[K]} f_k}{\min_{k\in[K]} f_k},\ 2 \right\}.\]
\end{theorem}
\begin{proof}{Proof.}
Let $\{m_{k,t}^{i,*}\}_{k=0}^K$ denote the fulfillment plan of the optimal offline policy. For simplicity, we also denote $f=\min_{k\in[K]}f_k$ as the minimal fixed cost among all FDCs. We define the costs incurred at period $t$ by the optimal offline policy and by the \nsc{\multiinvaralg} policy, respectively, as
\begin{align*}
    V_t^*=\sum_{k=0}^K \left[ f_k \cdot \mathbb{I}\left(\sum_{i=1}^n m_{k,t}^{i,*} > 0\right) + \sum_{i=1}^n c_k^i \cdot m_{k,t}^{i,*} \right], \qquad
    V_t=\sum_{k=0}^K \left[ f_k \cdot \mathbb{I}\left(\sum_{i=1}^n m_{k,t}^i > 0\right) + \sum_{i=1}^n c_k^i \cdot m_{k,t}^i \right].
\end{align*}
Then, we have 
\[(\text{cost incurred by~}\nsc{\multiinvaralg}) = \sum_{t=1}^T V_t ,\quad \mathrm{OPT} = \sum_{t=1}^T V_t^*.\]
At each time period $t$, we define an intermediate cost $\bar{V_t}$ as follows:
\begin{equation}
    \bar{V_t}=\left\{
    \begin{aligned}
        & f_0 + \sum_{k=1}^Kf_k\cdot \mathbb{I}\left( \sum_{i=1}^n m_{k,t}^i>0 \right) + \sum_{i=1}^n c_0^i\cdot S_t^i,&\quad\text{(if OPT only uses RDC)}\\
        & f_0 + \sum_{k=1}^Kf_k + \sum_{k=0}^K\sum_{i=1}^n c_k^i\cdot m_{k,t}^{i,*}. &\text{(otherwise)}
    \end{aligned}\right.
    \label{eq:def-intermediate-cost}
\end{equation}

\medskip\noindent \underline{Step I: Upper bounding the intermediate costs.} We first prove the following upper bound on the intermediate costs:
\begin{equation}
    \bar{V_t}\leq \max\left\{ \frac{f_0+\sum_{k=1}^Kf_k}{f},2 \right\} \cdot V_t^*.
    \label{eq:intermediate-cost-upper-bound}
\end{equation}
For time periods $t$ where OPT uses only the RDC to fulfill the order, the cost of OPT is $V_t^*=f_0 + \sum_{i=1}^n c_0^i\cdot S_t^i$. If \nsc{\multiinvaralg} uses only the RDC as well, then we have the intermediate cost $\bar{V_t}=V_t^*$. If \nsc{\multiinvaralg} uses at least one FDC, then we know that the gating condition is not triggered, i.e., $\sum_{k=0}^K\left[ f_k\cdot \mathbb{I}\left( \sum_{i=1}^n m_{k,t}^i>0 \right) + \sum_{i=1}^n c_k^i\cdot m_{k,t}^i \right] \leq f_0 + \sum_{i=1}^n c_0^i\cdot S_t^i$; in this case, we have
\begin{equation*}
    \bar{V_t}=f_0 + \sum_{k=1}^Kf_k\cdot \mathbb{I}\left( \sum_{i=1}^n m_{k,t}^i>0 \right) + \sum_{i=1}^n c_0^i\cdot S_t^i\leq 2f_0+2\sum_{i=1}^n c_0^i\cdot S_t^i\leq 2V_t^*.
\end{equation*}
For time periods $t$ where OPT uses at least one FDC to fulfill the order, we bound the intermediate cost as follows: 
\begin{align*}
    \bar{V_t}=f_0 + \sum_{k=1}^Kf_k + \sum_{k=0}^K\sum_{i=1}^n c_k^i\cdot m_{k,t}^{i,*}\leq\frac{f_0+\sum_{k=1}^Kf_k}{f}\cdot V_t^*.
\end{align*}
Combining both cases above, we prove the upper bound in Eq.~\eqref{eq:intermediate-cost-upper-bound}.

\medskip\noindent \underline{Step II: Upper bounding our algorithm's cost via the intermediate costs.} 
We upper bound that the total cost of \nsc{\multiinvaralg} as follows:
\begin{equation}
    \sum_{t=1}^T V_t\leq \sum_{t=1}^T \bar{V_t}.
    \label{eq:total-cost-upper-bound-by-intermediate-cost}
\end{equation}
To prove Eq.~\eqref{eq:total-cost-upper-bound-by-intermediate-cost}, we let $A$ be the set of time periods that \nsc{\multiinvaralg} uses at least one FDC to fulfill the order. We consider the following cases.
\begin{itemize}
    \item \underline{Case 1.} For periods $t\notin A$, we have $V_t=f_0 + \sum_{i=1}^n c_0^i\cdot S_t^i$. 
    \begin{itemize}
        \item \underline{Case 1a.} If OPT uses RDC only, then we further have $V_t=\bar{V_t}$.
        \item \underline{Case 1b.}  If OPT uses at least one FDC, since the gating condition is triggered, we have 
\begin{align}
    V_t&=f_0+\sum_{i=1}^nc_0^i\cdot S_t^i \leq \sum_{k=0}^K\left[ f_k\cdot \mathbb{I}\left( \sum_{i=1}^n \hat{m}_{k,t}^i>0 \right) + \sum_{i=1}^n c_k^i\cdot \hat{m}_{k,t}^i \right]\notag\\
    &\qquad\qquad \leq f_0+\sum_{k=1}^K f_k + \sum_{k=0}^K \sum_{i=1}^n c_k^i\cdot \hat{m}_{k,t}^i  = \bar{V}_t + \left( \sum_{k=0}^K \sum_{i=1}^n c_k^i\cdot \hat{m}_{k,t}^i - \sum_{k=0}^K \sum_{i=1}^n c_k^i\cdot m_{k,t}^{i,*} \right).
    \label{eq:cost-upper-bound-by-intermediate-cost-case1}
\end{align}
    \end{itemize}
    \item \underline{Case 2.} For periods $t\in A$, we have $m_{k,t}^i=\hat{m}_{k,t}^i$ for all $k,i$. If OPT uses RDC only, the definition of $\bar{V}_t$ implies that $\bar{V}_t=f_0 + \sum_{k=1}^Kf_k\cdot \mathbb{I}\left( \sum_{i=1}^n m_{k,t}^i>0 \right) + \sum_{i=1}^n \sum_{k=0}^K c_k^i\cdot m_{k,t}^{i,*}$, and therefore
\begin{equation}
    V_t=\sum_{k=0}^K \left[ f_k \cdot \mathbb{I}\left(\sum_{i=1}^n m_{k,t}^i > 0\right) + \sum_{i=1}^n c_k^i \cdot \hat{m}_{k,t}^i \right] \leq \bar{V_t}+ \left( \sum_{k=0}^K \sum_{i=1}^n c_k^i\cdot \hat{m}_{k,t}^i - \sum_{k=0}^K \sum_{i=1}^n c_k^i\cdot m_{k,t}^{i,*} \right).
    \label{eq:cost-upper-bound-by-intermediate-cost-case2}
\end{equation}
If OPT uses at least one FDC, then we may also verify that have Eq.~\eqref{eq:cost-upper-bound-by-intermediate-cost-case2} holds.
\end{itemize}
Summarizing the cases above, let $B\supseteq A$ be the set of time periods that either \nsc{\multiinvaralg} or $\mathrm{OPT}$ uses at least one FDC to fulfill the order, and we have that
\begin{align}
    V_t \leq \bar{V_t}+ \sum_{t \in B} \left( \sum_{k=0}^K \sum_{i=1}^n c_k^i\cdot \hat{m}_{k,t}^i - \sum_{k=0}^K \sum_{i=1}^n c_k^i\cdot m_{k,t}^{i,*} \right). \label{eq:cost-upper-bound-by-intermediate-cost}
\end{align}

On the other hand, the following claim can be proved using our key technical Lemma~\ref{lemma: key lemma of GPG policy}. The proof of the claim uses similar arguments as  Claim~\ref{claim:proof-os-fp-tool}, and is deferred to Section~\ref{subsec:proof-claim-proof-cc-vp-tool} of the e-companion.
\begin{claim}
    \label{claim:proof-cc-vp-tool}
    For each item $i\in[n]$, we have
    $\sum_{t\in B}\left( \sum_{k=0}^K c_k^i\cdot \hat{m}_{k,t}^i \right)\leq \sum_{t\in B}\left( \sum_{k=0}^K c_k^i\cdot m_{k,t}^{i,*} \right)$.
\end{claim}
Combining Eq.~\eqref{eq:cost-upper-bound-by-intermediate-cost} and Claim~\ref{claim:proof-cc-vp-tool}, we establish Eq.~\eqref{eq:total-cost-upper-bound-by-intermediate-cost}. Finally, we prove the theorem by combining Eq.~\eqref{eq:intermediate-cost-upper-bound} and Eq.~\eqref{eq:total-cost-upper-bound-by-intermediate-cost}.
\hfill\Halmos

\end{proof}

\medskip \noindent \textbf{Lower Bound.}
The next theorem establishes a lower bound on the competitive ratio of any online fulfillment policy under Assumption~\ref{assumption: time invariant cost} for all $K\geq 1$. The proof is deferred to Section~\ref{subsec: proof of lower bound with time varying cost} of the e-companion. A refined lower bound for the case of $K=1$ will be presented later in Theorem~\ref{thm: lower bound with time invariant cost k=1}.

\begin{theorem}
    \label{thm: lower bound with time invariant cost}
    Under Assumption~\ref{assumption: time invariant cost}, for any $K\geq 1$ and any given fixed costs $\{f_k\}_{k=0}^K$, every online fulfillment policy satisfies
    $\CR_{\mathrm{inv}}(\mathrm{ALG}) \geq \frac{f_0+\sum_{k\in[K]} f_k}{\min_{k\in[K]} f_k}$.
\end{theorem}

This theorem shows that the lower bound on the competitive ratio coincides exactly with the upper bound attained by the \nsc{\multiinvaralg} algorithm whenever the fixed costs satisfy $\frac{f_0+\sum_{k\in[K]} f_k}{\min_{k\in[K]} f_k}\geq 2$. Therefore, under Assumption~\ref{assumption: time invariant cost}, in the case $K\geq 2$, or  when $K=1$ and $f_0 \geq f_1$, \nsc{\multiinvaralg} achieves the exact optimal competitive ratio.

\section{Refinements for Single-FDC Settings}
\label{sec: refinements for single-fdc settings}

While the algorithms presented in Section~\ref{sec: main results for multiple fdcs case} apply to any number of FDCs $K\geq 1$, their competitive ratio guarantees may be further improved in the single-FDC setting ($K=1$), which is the focus of this section. In Section~\ref{subsec: time-varying variable costs case single fdc}, we study the single-FDC setting with time-varying variable costs, whereas in Section~\ref{subsec: time-invariant variable costs case single fdc}, we consider the corresponding setting with time-invariant variable costs.

\subsection{Time-varying Variable Costs with Single FDC}
\label{subsec: time-varying variable costs case single fdc}

In this section, we present an improved algorithm for the time-varying variable cost setting with a single FDC. The algorithm also operates under Assumption~\ref{assumption: bounded cost}, namely, that the variable costs lie in the interval $[a,b]$. Our approach is a hybrid algorithm that combines two sub-algorithms: the first achieves a better dependence on $b/a$, while the second further improves the competitive ratio when $f_0 \geq f_1$. Both are based on variants of the \nsc{\gatedprioritygreedy} framework, with slight modifications to incorporate \emph{time-dependent priority rules} and \emph{FDC-gating conditions} that route the entire order to the FDC rather than the RDC. We first present the two sub-algorithms, then combine them into a hybrid algorithm and characterize its competitive ratio, and finally complement this result with a lower bound on the competitive ratio of any algorithm.

\medskip
\noindent\textbf{The first sub-algorithm: improve the dependence on ${b/a}$.} Compared with Theorem~\ref{thm: upper bound with bounded cost}, when the fixed costs are held fixed, our first refined algorithm improves the competitive ratio from $O(b/a)$ to $O(\sqrt{b/a})$. This algorithm is also an instantiation of our \nsc{\gatedprioritygreedy} framework. 

To improve the dependence on the variable cost parameters, we can no longer neglect the losses induced by variable costs, as we did in \nsc{\multivaryingalg} (Algorithm~\ref{alg: multi vary alg}, Section~\ref{subsec: time-varying variable costs case}). Motivated by this consideration, we introduce the \textsc{Adjusted-Variable-Cost} priority rule, which ranks the FDC and RDC by comparing $c_{1,t}^i$ with $\sqrt{a/b} \cdot c_{0,t}^i$. However, the main risk of using the \textsc{Adjusted-Variable-Cost} priority rule is analogous to that of using the \textsc{Variable-Cost} priority rule in \nsc{\multiinvaralg} (Algorithm~\ref{alg: multi invar alg}, Section~\ref{subsec: time-invariant variable costs case}): it may lead to excessive activation of the FDC and thus incur a large fixed cost. To control this risk, we adopt the same \textsc{Cost-Comparison} gating condition as in \nsc{\multiinvaralg}. 

We also note that the \textsc{Adjusted-Variable-Cost} priority rule is time-dependent, in the sense that the ranking may vary across time periods, whereas our original \nsc{\gatedprioritygreedy} framework only admits time-invariant priority rankings. Nevertheless, it is straightforward to extend the \nsc{\gatedprioritygreedy} framework to allow the algorithm to use a time-dependent priority rule at each time period. The formal definitions of the improved algorithm, the priority rule, and the gating condition are provided in Algorithm~\ref{alg: sing vary alg}.

\begin{algorithm}[h]
    \caption{\nsc{\singvaryingalg} (single FDC)}
    \label{alg: sing vary alg}
    Extend \nsc{\gatedprioritygreedy} by replacing all time-independent priority rules $\prec_i$ with time-dependent priority rules $\prec_{i, t}$ in Algorithm~\ref{alg: GPG policy}, and instantiate the resulting framework using:
    \begin{itemize}
    \item \underline{\textsc{Adjusted-Variable-Cost} priority}: for $\forall i \in [n]$, define $\prec_i$ as
    $1\prec_{i,t} 0~~\Leftrightarrow~~c_{1,t}^i<\sqrt{\frac{a}{b}}\cdot c_{0,t}^i$;
    \item \underline{\textsc{Cost-Comparison} gating condition}:
    \[G\left( \left\{ f_k,c_{k,t}^i,S_t^i,\hat{m}_{k,t}^i \right\}_{\substack{k=0,1\\i=1,\dots,n}} \right)=\mathbb{I}\left( \sum_{k=0,1}\left[ f_k\cdot \mathbb{I}\left(\sum_{i=1}^n \hat{m}_{k,t}^i > 0\right) + \sum_{i=1}^n c_{k,t}^i \cdot \hat{m}_{k,t}^i \right] > f_0 + \sum_{i=1}^n c_{0,t}^i \cdot S_t^i \right).\]
    \end{itemize}
\end{algorithm}

The following theorem provides an upper bound for the competitive ratio of \nsc{\singvaryingalg}. Its proof uses a similar technique as that of Theorem~\ref{thm: upper bound with time invariant cost}.
\begin{theorem}
    \label{thm: upper bound with bounded cost k=1}
    Under Assumption~\ref{assumption: bounded cost}, the competitive ratio of \nsc{\singvaryingalg} defined in Algorithm \ref{alg: sing vary alg} satisfies
    \[\CR(\nsc{\singvaryingalg}) \leq 1 + \max\left\{ \frac{f_0}{f_1},\ \sqrt{\frac{b}{a}} \right\}.\]
\end{theorem}

\begin{proof}{Proof.}
Let $\{m_{k,t}^{i,*}\}_{k=0}^K$ denote the fulfillment plan of the optimal offline policy. For simplicity, we also denote $f=\min_{k\in[K]}f_k$ as the minimal fixed cost among all FDCs. We define the costs incurred at period $t$ by the optimal offline policy and by the \nsc{\singvaryingalg} policy, respectively, as
\begin{align*}
    V_t^*=\sum_{k=0,1} \left[ f_k \cdot \mathbb{I}\left(\sum_{i=1}^n m_{k,t}^{i,*} > 0\right) + \sum_{i=1}^n c_{k,t}^i \cdot m_{k,t}^{i,*} \right], \quad
    V_t=\sum_{k=0,1} \left[ f_k \cdot \mathbb{I}\left(\sum_{i=1}^n m_{k,t}^i > 0\right) + \sum_{i=1}^n c_{k,t}^i \cdot m_{k,t}^i \right].
\end{align*}
Then, we have 
\[(\text{cost incurred by~}\nsc{\singvaryingalg}) = \sum_{t=1}^T V_t ,\quad \mathrm{OPT} = \sum_{t=1}^T V_t^*.\]
At each time period $t$, we define an intermediate cost $\bar{V_t}$ as follows:
\begin{equation}
    \bar{V_t}=\left\{
    \begin{aligned}
        & f_0 + f_1\cdot \mathbb{I}\left( \sum_{i=1}^n m_{1,t}^i>0 \right) + \sqrt{\frac{b}{a}}\cdot\sum_{i=1}^n c_{0,t}^i\cdot S_t^i,&\quad\text{(if OPT only uses RDC)}\\
        & f_0 + f_1 + \sqrt{\frac{b}{a}}\cdot\left( \sum_{i=1}^n c_{0,t}^i\cdot m_{0,t}^{i,*}+\sum_{i=1}^n c_{1,t}^i\cdot m_{1,t}^{i,*} \right). &\text{(otherwise)}
    \end{aligned}\right.
    \label{eq:def-intermediate-cost-time-varying-single-fdc}
\end{equation}

The proof of the following claim is similar to Steps I and II in the proof of Theorem~\ref{thm: upper bound with time invariant cost}. It also uses a variant of Lemma~\ref{lemma: key lemma of GPG policy} for the extended \nsc{\gatedprioritygreedy} framework with time-dependent priority rules in the case $K=1$, whose proof parallels that of Lemma~\ref{lemma: key lemma of GPG policy}. The proof of Claim~\ref{claim:proof-cc-vp-barV-time-varying-single-fdc} and the corresponding lemma details are deferred to Section~\ref{subsec:proof-claim-proof-cc-vp-barV-time-varying-single-fdc} of the e-companion.
\begin{claim}
\label{claim:proof-cc-vp-barV-time-varying-single-fdc}
For each time period $t$, we have
$    \bar{V_t}\leq \max\left\{ 1+\frac{f_0}{f_1},\ 1+\sqrt{\frac{b}{a}} \right\} \cdot V_t^*$.
Furthermore, we have 
$\sum_{t=1}^T V_t\leq \sum_{t=1}^T \bar{V_t}$.
\end{claim}

Finally, we prove the theorem by combining the two equations in Claim~\ref{claim:proof-cc-vp-barV-time-varying-single-fdc}.
\hfill\Halmos
\end{proof}

\medskip
\noindent\textbf{The second sub-algorithm: improve the performance when ${f_0 \geq f_1}$.} 
Algorithm~\ref{alg: sing vary alg} improves the competitive ratio dependence on $b/a$ from $O(b/a)$ to $O(\sqrt{b/a})$. However, there remains a gap between its overall competitive ratio and the true lower bound, as will be established in Theorem~\ref{thm: lower bound with bounded cost k=1}. This gap is particularly pronounced in the regime where $f_0/f_1$ is large, and arises mainly because the algorithm tends to use the RDC excessively, even when the FDC has abundant inventory.

To further improve the competitive ratio in this regime, we slightly revise the \nsc{\gatedprioritygreedy} framework to \textsc{FDC-\gatedprioritygreedy}. In this variant, the gating condition is replaced by the \emph{FDC-gating condition}, which ensures that the algorithm uses the FDC rather than the RDC whenever the FDC inventory is sufficient to fulfill the order. Specifically, Lines~\ref{line:gpg-5}--\ref{line:gpg-6} of Algorithm~\ref{alg: GPG policy} (for $K = 1$) are revised as follows, and all time-independent priority rules $\prec_i$ elsewhere in the algorithm are replaced with time-dependent priority rules $\prec_{i,t}$, as in Algorithm~\ref{alg: sing vary alg}:
\begin{align*}
{\scriptsize\text{\ref{line:gpg-5}:}}&\qquad\text{{\bf if}~gating condition $G(\{f_k,c_{k,t}^i,S_t^i,\hat{m}_{k,t}^i\})=1$~{\bf and}~$\forall i\in [n], I_{1,t-1}^i\geq S_t^i $~{\bf then}}\\
{\scriptsize\text{\ref{line:gpg-6}:}}&\qquad\qquad m_{1,t}^i \gets S_t^i, ~m_{0,t}^i \gets 0 \qquad \forall i \in [n] \qquad\qquad\qquad\qquad\triangleright\text{route entire order to FDC}
\end{align*}

We instantiate the \textsc{FDC-\gatedprioritygreedy} framework using the \nsc{Order-Size} gating condition, which is triggered when the total order size does not exceeds a predefined threshold $\theta$. For the time-dependent priority rule, we employ the \textsc{Adjusted-Variable-Cost} priority rule (a slightly more general version than that used in Algorithm~\ref{alg: sing vary alg}), where the FDC variable cost is adjusted using a general parameter $\eta$. The complete description of the improved algorithm is provided in Algorithm~\ref{alg: sing vary alg improved}.

\begin{algorithm}[h]
    \caption{\nsc{\singvaryingalgimprove} (single FDC, input parameters $\eta \geq 1$, $\theta \geq 0$)}
    \label{alg: sing vary alg improved}
    Instantiate the \nsc{FDC-\gatedprioritygreedy} framework using:
    \begin{itemize}
    \item \underline{\textsc{Adjusted-Variable-Cost} priority} (param.~$\eta$): $\forall i \in [n]$, define $\prec_i$ as
    $1\prec_{i,t} 0 \Leftrightarrow c_{1,t}^i<c_{0,t}^i/\eta$;
    \item \underline{\textsc{Order-Size} FDC-gating condition} (param.~$\theta$):
    \[G\left( \left\{ f_k,c_{k,t}^i,S_t^i,\hat{m}_{k,t}^i \right\}_{\substack{k=0,\dots,K\\i=1,\dots,n}} \right)=\mathbb{I}\left( \sum_{i=1}^nS_t^i\leq \theta \right).\]
    \end{itemize}
\end{algorithm}
The following theorem provides an upper bound on the competitive ratio of \nsc{\singvaryingalgimprove}. Its proof follows a framework similar to that of Theorem~\ref{thm: upper bound with bounded cost k=1}, in that it constructs an intermediate cost $\bar{V}_t$ to relate the cost incurred by the algorithm to that of the optimal offline policy. However, the construction of $\bar{V}_t$ is more involved and requires a case-by-case analysis. The full proof is deferred to Section~\ref{subsec:proof-upper-bound-time-varying-single-fdc-improved-alg} of the e-companion.
\begin{theorem}
    \label{thm: upper bound with bounded cost k=1 improved}
    Suppose $f_0 \geq f_1$. Under Assumption~\ref{assumption: bounded cost}, choosing $\eta=\sqrt{\frac{\max\left\{ f_0/2,\ b \right\}}{a}}$ and $\theta=\frac{f_0}{2a\eta}$, we have that the competitive ratio of \nsc{\singvaryingalgimprove} defined in Algorithm \ref{alg: sing vary alg improved} satisfies
    \[\CR(\nsc{\singvaryingalgimprove}) \leq (4+\sqrt{2})\sqrt{\frac{\max\left\{ f_0/2,\ b \right\}}{a}}.\]
\end{theorem}

\medskip
\noindent\textbf{Final algorithm by choosing the better of the two.} Our final algorithm for the time-varying variable cost setting with a single FDC chooses the better of the two algorithms introduced above. Specifically, the {\sc Better-of-Two} algorithm chooses Algorithm~\ref{alg: sing vary alg} whenever $f_0 \leq f_1$ or $1 + \max\{f_0/f_1, \sqrt{b/a}\} \leq (4+\sqrt{2})\max\{\sqrt{f_0/(2a)}, \sqrt{b/a}\}$, and it chooses Algorithm~\ref{alg: sing vary alg improved} otherwise.
The following corollary upper bounds the competitive ratio of this hybrid algorithm. The proof of Corollary~\ref{cor:comp-ratio-upper-bound-better-of-two} is deferred to Section~\ref{subsec:proof-cor-comp-ratio-upper-bound-better-of-two} of the e-companion.
\begin{corollary}
\label{cor:comp-ratio-upper-bound-better-of-two}
Under Assumption~\ref{assumption: bounded cost}, in the single-FDC setting, we have
\begin{align*}
\CR(\nsc{Better-of-Two}) & \leq \min\left\{ 1+\max\left\{\frac{f_0}{f_1},\sqrt{\frac{b}{a}}\right\},\  (4+\sqrt{2})\sqrt{\frac{\max\left\{ f_0/2,\ b \right\}}{a}} + \infty\cdot \mathbb{I}(f_0 < f_1) \right\}\\
& \leq \max\left\{\min\left\{2 \cdot \frac{f_0}{f_1},~(2\sqrt{2}+1)\sqrt{\frac{f_0}{a}}\right\}, ~(4+\sqrt{2})\sqrt{\frac{b}{a}}\right\}.
\end{align*}
\end{corollary}

\medskip
\noindent\textbf{Lower Bound.} The next theorem establishes a lower bound on the competitive ratio of any online fulfillment policy under Assumption~\ref{assumption: bounded cost} for the case $K=1$. Moreover, this lower bound matches, up to a constant factor, the upper bound of the \textsc{Better-of-Two} policy given in Corollary~\ref{cor:comp-ratio-upper-bound-better-of-two}. The proof of Theorem~\ref{thm: lower bound with bounded cost k=1} is deferred to Section~\ref{subsec:proof-lower-bound-bounded-cost-single-fdc} of the e-companion.
\begin{theorem}
    \label{thm: lower bound with bounded cost k=1}
    Under Assumption~\ref{assumption: bounded cost} and in the case of $K=1$, for any fixed costs $f_0$ and $f_1$ of the RDC and FDC, any online fulfillment policy has a competitive ratio of at least
    $\CR(\mathrm{ALG}) \geq \max\left\{1,\ \frac{1}{3}\sqrt{\frac{b}{a}},\ \frac{1}{4}\max_{n\geq 2}\min\left\{ n,\ \frac{f_0}{f_1+na}\right\} \right\}$.
\end{theorem}

We now compare the above lower bound with the upper bound established in Corollary~\ref{cor:comp-ratio-upper-bound-better-of-two}. The following proposition shows that the two bounds match up to a constant factor, which implies the order-optimality of our \nsc{Better-of-Two} algorithm. The proof is straightforward, and is deferred to Section~\ref{subsec: proof-of-ub-lb-ratio-single-varying} of the e-companion.

\begin{proposition}
    \label{prop:ub-lb-ratio-single-varying}
    The upper bound in Corollary~\ref{cor:comp-ratio-upper-bound-better-of-two} and the lower bound in Theorem~\ref{thm: lower bound with bounded cost k=1} match up to a constant factor of $\frac{4\cdot(9+4\sqrt{2})}{\sqrt{10+4\sqrt{2}}-1}\leq 19.828$.
\end{proposition}

\subsection{Time-Invariant Variable Costs Case}
\label{subsec: time-invariant variable costs case single fdc}

In this section, we present an improved algorithm for the single-FDC case in the time-invariant variable cost setting. Our goal is to improve the competitive ratio in the regime $f_0<f_1$, where \nsc{\multiinvaralg} achieves a competitive ratio of $2$ (Theorem~\ref{thm: upper bound with time invariant cost}). As in \nsc{\multiinvaralg}, we adopt the \nsc{Variable-Cost} priority rule. However, for the gating condition, we introduce a soft version of the \nsc{Cost-Comparison} gating condition by incorporating randomization based on the difference in variable cost between the greedy allocation and the fully-RDC allocation. Specifically, we define a monotone nonincreasing function $p : \mathbb{R} \to [0, 1]$ that maps the variable cost difference $\sum_{i=1}^n\left( c_0^i-c_1^i \right)\cdot \hat{m}_{1,t}^i$ to the probability with which the gating condition is triggered. We refer to this gating condition as \nsc{Randomized-Cost-Comparison}, and the precise definition of $p(\cdot)$ is given in Eq.~\eqref{eq:randomized-cost-comparison-p-func}. 

The intuition behind this design is as follows. When the variable cost increment of switching the FDC allocations entirely to RDC is small relative to the fixed FDC cost (i.e., $\sum_{i=1}^n(c_0^i-c_1^i)\cdot \hat{m}_{1,t}^i \leq f_1$), it is more cost-effective to fulfill the order entirely from the RDC, so we set $p(\cdot)=1$. As the variable cost increment grows, the benefit of using the RDC decreases, and we gradually reduce the probability of routing the order to the RDC. Once the increment exceeds a threshold (i.e., $\sum_{i=1}^n(c_0^i-c_1^i)\cdot \hat{m}_{1,t}^i > \max\left\{ f_1,~\frac{f_1^2}{f_0}-f_0 \right\}$), it becomes more cost-effective to use the FDC, and we set $p(\cdot)=0$. The resulting algorithm is a natural randomized extension of \nsc{\gatedprioritygreedy}. A complete description of the algorithm is provided in Algorithm~\ref{alg: sing invar alg}.

\begin{algorithm}[ht]
    \caption{\nsc{\singinvaralg} (single FDC)}
    \label{alg: sing invar alg}
    Instantiate \nsc{\gatedprioritygreedy} with randomized gating condition using:
    \begin{itemize}
    \item \underline{\textsc{Variable-Cost} priority}: for $\forall i \in [n]$, define $\prec_i$ as
    $1 \prec_i 0\Leftrightarrow c_1^i<c_0^i,~ \forall i \in [n]$;
    \item \underline{\textsc{Randomized-Cost-Comparison} gating condition}:
    \begin{align}
    &G\left( \left\{ f_k,c_k^i,S_t^i,\hat{m}_{k,t}^i \right\}_{\substack{k=0,\dots,K\\i=1,\dots,n}} \right)=\mathrm{Bernoulli}\left( p\left( \sum_{i=1}^n\left( c_0^i-c_1^i \right)\cdot \hat{m}_{1,t}^i \right) \right),  \notag \\
    &\qquad \qquad \qquad \text{where} \qquad\qquad p(x) = 
    \begin{cases}
        1,& x \leq f_1,\\
        \frac{f_1^2-(f_0+x)f_0}{f_1^2+(f_0+x)(x-f_1)},& f_1 < x \leq \max\left\{ f_1,~\frac{f_1^2}{f_0}-f_0 \right\},\\
        0,& x > \max\left\{ f_1,~\frac{f_1^2}{f_0}-f_0 \right\} .
    \end{cases} \label{eq:randomized-cost-comparison-p-func}
    \end{align}
    \end{itemize}
\end{algorithm}

The following theorem provides upper bounds the competitive ratio of \nsc{\singinvaralg}. Its proof uses a similar technique as that of Theorem~\ref{thm: upper bound with time invariant cost}.

\begin{theorem}
    \label{thm: upper bound with time invariant cost k=1}
    Let $w=\frac{f_0}{f_1}$. Under Assumption~\ref{assumption: time invariant cost}, the competitive ratio of the \nsc{\singinvaralg} defined in Algorithm \ref{alg: sing invar alg} satisfies
    \[\CR_{\mathrm{inv}}(\nsc{\singinvaralg}) \leq 
    \begin{cases}
        1+\frac{1}{1-w+2\sqrt{1-w}},\quad & w<\frac{\sqrt{5}-1}{2}\\
        1+w & w\geq \frac{\sqrt{5}-1}{2}
    \end{cases}.\]
\end{theorem}

\begin{proof}{Proof.}
Define the cost incurred by the optimal offline policy and the \nsc{\singinvaralg} fulfillment policy at time period $t$ as
\begin{align*}
    V_t^*&=f_0\cdot \mathbb{I}\left( \sum_{i=1}^n m_{0,t}^{i,*}>0 \right)+f_1 \cdot \mathbb{I}\left(\sum_{i=1}^n m_{1,t}^{i,*} > 0\right) + \sum_{i=1}^n c_0^i \cdot m_{0,t}^{i,*} + \sum_{i=1}^n c_1^i \cdot m_{1,t}^{i,*},\\
    V_t&=f_0\cdot \mathbb{I}\left( \sum_{i=1}^n m_{0,t}^i>0 \right)+f_1 \cdot \mathbb{I}\left(\sum_{i=1}^n m_{1,t}^i > 0\right) + \sum_{i=1}^n c_0^i \cdot m_{0,t}^i + \sum_{i=1}^n c_1^i \cdot m_{1,t}^i.
\end{align*}
Then the total cost incurred by the \nsc{\singinvaralg} fulfillment policy and optimal offline policy can be written as
\[(\text{cost incurred by~}\nsc{\singinvaralg}) = \sum_{t=1}^T V_t ,\quad \mathrm{OPT} = \sum_{t=1}^T V_t^*.\]
Let $I=\{i\in[n]:c_1^i<c_0^i\}$. At each time period $t$, we define an intermediate cost $\bar{V_t}$ as follows:
\begin{equation*}
    \bar{V_t}=\left\{
    \begin{aligned}
        & f_0 + \sum_{i\in I}(c_1^i-c_0^i)\cdot \left( m_{1,t}^{i,*}-\hat{m}_{1,t}^i\right)^++\sum_{i=1}^n c_0^i\cdot S_t^i,\quad &\text{if }\theta_t=1,\\
        & f_0\cdot \mathbb{I}\left( \sum_{i=1}^n m_{0,t}^i>0 \right) + f_1\cdot \mathbb{I}\left( \sum_{i=1}^n m_{1,t}^i>0 \right) + \sum_{i\in I}(c_1^i-c_0^i)\cdot m_{1,t}^{i,*}+\sum_{i=1}^n c_0^i\cdot S_t^i,&\text{if }\theta_t=0.
    \end{aligned}\right.
\end{equation*}
The proof of the following claim is similar to Steps I and II in the proof of Theorem~\ref{thm: upper bound with time invariant cost}. The proof of Claim~\ref{claim:proof-rcc-vp-barV-time-invariant-single-fdc} is deferred to Section~\ref{subsec:proof-claim-proof-rcc-vp-barV-time-invariant-single-fdc} of the e-companion.
\begin{claim}
    \label{claim:proof-rcc-vp-barV-time-invariant-single-fdc}
    Let $\mathcal{F}_{t-1}$ denote the filtration up to time period $t-1$. We have the following inequalities:
    \begin{equation}
        \sum_{t=1}^T V_t\leq \sum_{t=1}^T \bar{V_t}, \qquad \text{and} \qquad \frac{\mathbb{E}[\bar{V_t}|\mathcal{F}_{t-1}]}{V_t^*}\leq
        \begin{cases}
        1+\frac{1}{1-w+2\sqrt{1-w}}, &\text{if } w<\frac{\sqrt{5}-1}{2}\\
        1+w, &\text{if } w\geq \frac{\sqrt{5}-1}{2}
        \end{cases},\quad\forall\,t\in[T] .
        \label{eq:total-cost-upper-bound-by-intermediate-cost-time-invariant-single-fdc}
    \end{equation} 
\end{claim}
Finally, we prove the theorem by combining the inequalities in Eq.~\eqref{eq:total-cost-upper-bound-by-intermediate-cost-time-invariant-single-fdc}.
\hfill\Halmos
\end{proof}

The next theorem, proved in Section~\ref{subsec:proof-lowerbound-time-invariant-single-FDC}, gives a lower bound on the competitive ratio for any online fulfillment policy under Assumption~\ref{assumption: time invariant cost} for $K=1$. It also matches the upper bound of the \nsc{\singinvaralg} in Theorem~\ref{thm: upper bound with time invariant cost k=1} within a constant factor.

\begin{theorem}
    \label{thm: lower bound with time invariant cost k=1}
    Under Assumption~\ref{assumption: time invariant cost} and $K=1$, given the fixed costs $f_0$ and $f_1$ of the DCs, any online fulfillment policy $\mathrm{ALG}$ has a competitive ratio of at least
    $\CR_{\mathrm{inv}}(\mathrm{ALG}) \geq \max\left\{ 1+\frac{f_0}{f_1},\ \frac{5}{4} \right\}$.
\end{theorem}

Comparing the upper and lower bounds in Theorem~\ref{thm: upper bound with time invariant cost k=1} and Theorem~\ref{thm: lower bound with time invariant cost k=1}, we can see that when the ratio of the fixed costs $w=\frac{f_0}{f_1}$ is large enough, i.e., $w\geq \frac{\sqrt{5}-1}{2}$, the upper and lower bounds match exactly, implying the optimality of our \nsc{\singinvaralg} algorithm. When $w<\frac{\sqrt{5}-1}{2}$, the upper and lower bounds are still of the same order within a constant factor of 
$\left(1+1/\left(0.75+2\sqrt{0.75}\right)\right)/{1.25}\leq 1.123$.

\section{Experiments}
\label{sec: experiments}

We conduct numerical experiments to examine whether the proposed policies preserve their theoretical advantages in representative stochastic environments and under deliberately challenging demand sequences. We summarize the main findings below. Overall, the results indicate that our policies achieve a favorable balance between computational efficiency and fulfillment-cost performance. In stochastic instances, they attain fulfillment costs comparable to those of computationally more demanding myopic and forecast-based benchmark policies. Under stressful temporal demand patterns, the gating mechanism can further yield substantial cost savings by mitigating premature depletion of scarce FDC inventory. Complete experimental settings and detailed results are provided in Section~\ref{sec: additional numerical experiments}.

\medskip
\noindent\underline{Experimental setting.}
We evaluate the proposed policies in both the time-varying and time-invariant variable-cost settings using stochastic order-arrival instances in two-layer fulfillment networks with scarce FDC inventory and an unlimited-backup RDC. The experiments vary key problem dimensions, including the selling horizon and the number of FDCs, and compare our policies with generalized myopic and, where applicable, forecast-based LP policies. We additionally construct stress-test instances to examine the inventory-depletion failure mode that motivates the gating mechanism. Detailed parameter settings, benchmark definitions, and implementation information are provided in Section~\ref{subsec: experiment settings}, and the baseline algorithms are specified in Section~\ref{subsec: baseline algorithms}.

\medskip
\noindent\underline{Time-varying variable costs.}
We first compare \textsc{Order-Size F-Priority} with the generalized myopic policy in multi-FDC networks. As the selling horizon increases, \textsc{Order-Size F-Priority} attains fulfillment costs comparable to, and in several instances lower than, those of the myopic policy, while requiring substantially less computation. We then vary the number of FDCs. The running time of the myopic policy grows rapidly with network size, consistent with the computational intractability of per-order cost minimization, whereas the running time of \textsc{Order-Size F-Priority} remains nearly stable and its fulfillment cost remains close to that of the myopic benchmark. In the single-FDC setting, \textsc{Order-Size AdjV-Priority} similarly achieves costs comparable to, and often lower than, those of the myopic policy as the horizon increases. Please refer to Section~\ref{subsec: time-varying experiments} for details.

\medskip
\noindent\underline{Time-invariant variable costs.}
We next compare \textsc{Cost-Comparison V-Priority} with the generalized myopic policy and four forecast-based LP policies, namely \textsc{IPFC}, \textsc{DPFC}, \textsc{Dilate}, and \textsc{ForceOpen}. Although the LP-based policies receive the true demand distribution, \textsc{Cost-Comparison V-Priority} achieves comparable fulfillment costs while requiring substantially less running time. This computational advantage becomes more pronounced as the number of FDCs grows: the running times of both the myopic and LP-based policies increase with network size, whereas that of \textsc{Cost-Comparison V-Priority} remains nearly unchanged. In the single-FDC setting, \textsc{Randomized-Cost-Comparison V-Priority} also achieves cost performance comparable to the myopic and LP-based benchmarks. Please refer to Section~\ref{subsec: time-invariant experiments} for details.

\medskip
\noindent\underline{Stress test.}
Finally, we construct a family of time-varying instances in which an initial large order is followed by many small orders, so that a myopic policy exhausts scarce FDC inventory prematurely. In these instances, \textsc{Order-Size F-Priority} achieves substantially lower fulfillment costs than the myopic policy, and the gap increases as the RDC fixed cost grows. This experiment illustrates the operational role of the gating mechanism: it protects local inventory against demand sequences under which immediate cost minimization can lead to severe future fixed-cost exposure. Please refer to Section~\ref{subsec: additional extreme cases} for details.

\section{Conclusion}
\label{sec: conclusion}

In this paper, we study a general multi-item online order fulfillment problem with multi-unit demand, multiple distribution centers, and rich cost structures. Our main goal is to bridge a practically important managerial problem with a sharp theoretical question: how well can a simple real-time fulfillment policy perform relative to a clairvoyant planner when future demand is unknown? To answer this question, we develop a family of \gatedprioritygreedy policies, identify the cost regimes under which different variants are effective, and derive upper and lower bounds that establish optimality or near-optimality in several settings.

Beyond the formal guarantees, the paper contributes a clear operational message. The key to real-time fulfillment is not merely choosing the cheapest facility for the current order; it is managing the intertemporal value of scarce local inventory. Our results show that this dynamic trade-off can often be handled by simple and interpretable gating rules, which is encouraging for firms seeking solutions that are both analytically grounded and practically deployable.

Our work provides a systematic framework for designing and analyzing online fulfillment policies for modern e-commerce networks. Promising directions for future research include incorporating replenishment decisions, integrating online and store-based demand in omnichannel settings, allowing endogenous service promises or pricing, and studying learning-augmented variants in which forecasts are available but imperfect.

\bibliographystyle{informs2014} 
\bibliography{ref} 

\ECSwitch

\ECHead{Electronic Companion}

\section{Omitted Proofs}
\label{sec: proofs of main theorems}

\subsection{Proof of Lemma~\ref{lemma: key lemma of GPG policy}}
\label{subsec: proof of lemma}
It is clear that for all $j=1,\dots,k_0$, we have $\sum_{k=1}^jm_{k,t}^{i,*}\leq S_t^i$ and $\sum_{t\in B}m_{k,t}^{i,*}\leq I_{k,0}^i$, which implies that
\[\sum_{k=1}^j\sum_{t\in B}m_{k,t}^{i,*}\leq \min\left\{ \sum_{t\in B}S_t^i,~\sum_{k=1}^jI_{k,0}^i \right\}.\]
Now we only need to prove
\[\sum_{k=1}^j\sum_{t\in B}\hat{m}_{k,t}^i \geq \min\left\{ \sum_{t\in B}S_t^i,~\sum_{k=1}^jI_{k,0}^i \right\},\quad \forall~j=1,\dots,k_0.\]
By the definition of the \nsc{\gatedprioritygreedy} policy, for all $j=1,\dots,k_0$, we have
\[\sum_{k=1}^j\hat{m}_{k,t}^i=\min\left\{ S_t^i,~\sum_{k=1}^j I_{k,t-1}^i \right\},\quad\forall~t.\]
So we only need to prove 
\begin{equation}
    \sum_{t\in B} \min\left\{ S_t^i,~\sum_{k=1}^j I_{k,t-1}^i \right\} \geq \min\left\{ \sum_{t\in B}S_t^i,~\sum_{k=1}^jI_{k,0}^i \right\},\quad \forall~j=1,\dots,k_0.
    \label{eq: important inequality of lemma key lemma}
\end{equation}
Fix $j\in\{1,\dots,k_0\}$. We prove Equation~\eqref{eq: important inequality of lemma key lemma} via two steps.

\noindent\underline{{Step I.}} We prove that $\sum_{k=1}^j\sum_{t\in A}\hat{m}_{k,t}^i = \min\left\{ \sum_{t\in A}S_t^i,\ \sum_{k=1}^j I_{k,0}^i \right\}$. We consider the following two cases:
\begin{itemize}
    \item Suppose $\sum_{t\in A} S_t^i \leq \sum_{k=1}^j I_{k,0}^i$. For periods $t\in A$, the \nsc{\gatedprioritygreedy} policy fulfill the order by greedy quantities, so we have $\hat{m}_{k,t}^i=m_{k,t}^i$. Then we have
    \[I_{k,t-1}^i = I_{k,0}^i-\sum_{\tau=1}^{t-1} m_{k,\tau}^i = I_{k,0}^i-\sum_{\substack{\tau\leq t-1\\\tau\in A}} \hat{m}_{k,\tau}^i,\]
    and then
    \[\sum_{k=1}^jI_{k,t-1}^i = \sum_{k=1}^j I_{k,0}^i-\sum_{k=1}^j\sum_{\substack{\tau\leq t-1\\\tau\in A}} \hat{m}_{k,\tau}^i \geq \sum_{k=1}^j I_{k,0}^i-\sum_{\substack{\tau\leq t-1\\\tau\in A}}S_\tau^i\geq \sum_{\substack{\tau\geq t\\t\in A}}S_t^i.\]
    Therefore, we have $\min\left\{ S_t^i,~\sum_{k=1}^jI_{k,t-1}^i \right\}= S_t^i$ for $t\in A$, and thus
    \[\sum_{k=1}^j\sum_{t\in A}\hat{m}_{k,t}^i = \sum_{t\in A} \min\left\{ S_t^i,~\sum_{k=1}^j I_{k,t-1}^i \right\} = \sum_{t\in A} S_t^i = \min\left\{ \sum_{t\in A}S_t^i,~\sum_{k=1}^jI_{k,0}^i \right\}.\]

    \item Suppose $\sum_{t\in A} S_t^i > \sum_{k=1}^j I_{k,0}^i$. We first prove $\sum_{k=1}^jI_{k,T}^i=0$. Assume not, i.e. $\sum_{k=1}^jI_{k,T}^i>0$, then $\sum_{k=1}^jI_{k,t}^i>0$ for all $t$. Then for all $t\in A$, the \gatedprioritygreedy fulfillment policy fulfills orders from the FDCs, 
    \[\min\left\{ S_t^i,~\sum_{k=1}^jI_{k,t-1}^i \right\}=\sum_{k=1}^j\hat{m}_{k,t}^i=\sum_{k=1}^jm_{k,t}^i=\sum_{k=1}^jI_{k,t-1}^i-\sum_{k=1}^jI_{k,t}^i<\sum_{k=1}^jI_{k,t-1}^i.\]
    This implies that for all $t\in A$, $S_t^i<\sum_{k=1}^jI_{k,t-1}^i$ and $\sum_{k=1}^jI_{k,t}^i=\sum_{k=1}^jI_{k,t-1}^i - S_t^i$. For $t\notin A$, we have $\sum_{k=1}^jI_{k,t}^i=\sum_{k=1}^jI_{k,t-1}^i$. Therefore, we have
    \[\sum_{k=1}^jI_{k,T}^i = \sum_{k=1}^jI_{k,0}^i - \sum_{t\in A} S_t^i \leq 0,\]
    which contradicts the assumption that $\sum_{k=1}^jI_{k,T}^i>0$. Thus we have $\sum_{k=1}^jI_{k,T}^i=0$, and therefore, 
    \[\sum_{k=1}^j\sum_{t\in A}\hat{m}_{k,t}^i = \sum_{t\in A} \min\left\{ S_t^i,~\sum_{k=1}^j I_{k,t-1}^i \right\} = \sum_{k=1}^j I_{k,0}^i-\sum_{k=1}^jI_{k,T}^i = \sum_{k=1}^jI_{k,0}^i .\]
\end{itemize}

\noindent\underline{{Step II.}} We prove Equation~\eqref{eq: important inequality of lemma key lemma}. We consider the following two cases:
\begin{itemize}
    \item If there exists $t\in B\backslash A$ such that $S_t^i>\sum_{k=1}^jI_{k,t-1}^i$, since $\sum_{k=1}^jI_{k,t-1}^i\geq\sum_{k=1}^jI_{k,T}^i$, then we have $\sum_{k=1}^j\sum_{t\in B\backslash A}\hat{m}_{k,t}^i = \sum_{t\in B\backslash A} \min\left\{ S_t^i,\ \sum_{k=1}^jI_{k,t-1}^i \right\}\geq \sum_{k=1}^jI_{k,T}^i$. Therefore, we have
    \[\sum_{k=1}^j\sum_{t\in B}\hat{m}_{k,t}^i = \sum_{k=1}^j\sum_{t\in B\backslash A}\hat{m}_{k,t}^i + \sum_{k=1}^j\sum_{t\in A}\hat{m}_{k,t}^i \geq \sum_{k=1}^jI_{k,T}^i + \left( \sum_{k=1}^jI_{k,0}^i - \sum_{k=1}^jI_{k,T}^i \right) = \sum_{k=1}^jI_{k,0}^i.\]

    \item If for all $t\in B\backslash A$, we have $S_t^i\leq \sum_{k=1}^jI_{k,t-1}^i$, then we can obtain $\sum_{k=1}^j\sum_{t\in B\backslash A}\hat{m}_{k,t}^i = \sum_{t\in B\backslash A}\min\left\{ S_t^i,~\sum_{k=1}^jI_{k,t-1}^i \right\}=\sum_{t\in B\backslash A} S_t^i$. Therefore, we have
    \[\sum_{k=1}^j\sum_{t\in B}\hat{m}_{k,t}^i = \sum_{k=1}^j\sum_{t\in B\backslash A}\hat{m}_{k,t}^i + \sum_{k=1}^j\sum_{t\in A}\hat{m}_{k,t}^i \geq \sum_{t\in B\backslash A} S_t^i + \min\left\{ \sum_{t\in A}S_t^i,\ \sum_{k=1}^j I_{k,0}^i \right\} \geq \min\left\{ \sum_{t\in B}S_t^i,\ \sum_{k=1}^j I_{k,0}^i \right\}.\]
\end{itemize}

This completes the proof of Lemma~\ref{lemma: key lemma of GPG policy}. \hfill\Halmos

\subsection{Proof of Theorem~\ref{thm: lower bound with bounded cost}}
\label{subsec: proof of lower bound with time varying cost}
Given the number of FDCs $K$ and fixed costs of FDCs and RDC $(f_0, f_1,\dots,f_K)$, we denote the minimal fixed cost among all FDCs as $f=\min_{k\in[K]}f_k$. Without loss of generality, we assume $f=f_1\leq f_2\leq \dots \leq f_K$.

We first prove that $\CR(\mathrm{ALG}) \geq \frac{1}{4}\max_n \min\left\{ n, \frac{f_0}{f+na}, \frac{f_0}{f+2a} \right\}$. We construct two instances to prove the lower bound. In both instances, there are $n$ items, indexed by $i=1,2,\dots,n$. The initial inventory of FDC 1 is $I_{1,0}^i=1$ for all items $i=1,2,\dots,n$, and the initial inventory of other FDCs are all zero. The variable costs are set to be constant among all DCs, items and time periods: $c_{0,t}^i=c_{k,t}^i=a$ for all FDCs $k=1,2,\dots,K$. At time period $t=1$, the customer order requests one unit of items $i=1,2,\dots,n$. At each time period $t=2,3,\dots,n+1$, the customer order requests one unit of items $i=t-1$. For the first instance $I_1$, we set the time period to be $T=1$, and for the second instance $I_2$, we set the time period to be $T=n+1$.

Let $\mathrm{ALG}$ be any online fulfillment policy. We consider the following two cases based on the decision of $\mathrm{ALG}$ at time period $t=1$.

\noindent\underline{Case 1.} Suppose $\mathrm{ALG}$ uses the RDC at time period $t=1$ with probability at least $\frac{1}{2}$. Then for instance $I_1$, the expected cost incurred by $\mathrm{ALG}$ is at least $\mathrm{ALG}(I_1) \geq \frac{1}{2}f_0$. The optimal offline policy for instance $I_1$ uses only FDC 1 to fulfill the order, and its cost is $\mathrm{OPT}(I_1)=f+na$. Therefore, the competitive ratio of $\mathrm{ALG}$ is at least
\[\frac{\mathrm{ALG}(I_1)}{\mathrm{OPT}(I_1)} \geq \frac{1}{2}\frac{f_0}{f+na}.\]

\noindent\underline{Case 2.} Suppose $\mathrm{ALG}$ uses the RDC at time period $t=1$ with probability less than $\frac{1}{2}$. Then for instance $I_2$, whenever $\mathrm{ALG}$ uses only FDC 1 at time $t=1$, it must use the RDC at all periods $t=2,3,\dots,n+1$ because FDC 1 is exhausted. Therefore, the expected cost incurred by $\mathrm{ALG}$ is at least $\mathrm{ALG}(I_2) \geq \frac{1}{2}n f_0$. The policy that uses the RDC at time $t=1$ and then uses FDC 1 at periods $t=2,3,\dots,n+1$ incurs cost $f_0+nf+2na$. Hence the optimal offline policy satisfies $\mathrm{OPT}(I_2)\leq f_0+nf+2na$, and the competitive ratio of $\mathrm{ALG}$ is at least
\[\frac{\mathrm{ALG}(I_2)}{\mathrm{OPT}(I_2)} \geq \frac{1}{2}\frac{n f_0}{f_0+n(f+2a)}\geq\frac{1}{4}\min\left\{ n,\frac{f_0}{f+2a} \right\}.\]

Combining the above two cases, we have
\[\CR(\mathrm{ALG}) \geq \min\left\{ \frac{1}{2}\frac{f_0}{f+na},\ \frac{1}{4}\min\left\{ n,\frac{f_0}{f+2a} \right\} \right\}\geq \frac{1}{4}\min\left\{ n, \frac{f_0}{f+na}, \frac{f_0}{f+2a} \right\}.\]
Since we can choose $n$ arbitrarily, we obtain the desired lower bound:
\[\CR(\mathrm{ALG}) \geq \frac{1}{4}\max_n \min\left\{ n, \frac{f_0}{f+na}, \frac{f_0}{f+2a} \right\}.\]
Since the ratio is at least 1, so we further have
\[\CR(\mathrm{ALG}) \geq \max\left\{ 1,~ \frac{1}{4}\max_{n\geq 2}\min\left\{ n,~\frac{f_0}{f+na} \right\} \right\}.\]

Now we prove that $\CR(\mathrm{ALG}) \geq \frac{b}{4a}$. We also construct two instances to prove the lower bound. In both instances, there is only one type of items $n=1$. The initial inventories of FDC 1 and FDC 2 is $N$, where $N$ is a large integer to be specified later, and the initial inventories of other FDCs are all zero. Note that we can construct such instance only when $K\geq 2$. The time period is set to be $T=2$. At time period $t=1$, the customer order requests $N$ units of the item, and the variable cost is $c_{0,1}^1=b,\ c_{1,1}^1=a,\ c_{2,1}^1=a$.

For the first instance $I_1$, the customer order requests $N$ units of the item at time period $t=2$, and the variable cost is $c_{0,2}^1=b,\ c_{1,2}^1=b,\ c_{2,2}^1=a$. For the second instance $I_2$, the customer order requests $N$ units of the item at time period $t=2$, and the variable cost is $c_{0,2}^1=b,\ c_{1,2}^1=a,\ c_{2,2}^1=b$.

Let $m_0, m_1, m_2$ be the number of items fulfilled from RDC, FDC 1 and FDC 2 at time period $t=1$ by $\mathrm{ALG}$, respectively. Here $m_0, m_1, m_2$ may be random variables. Then the cost incurred at time $t=1$ is at least $b\cdot m_0$, and the remaining inventories of FDC 1 and FDC 2 at time $t=2$ are $N-m_1$ and $N-m_2$, respectively. For instance $I_1$, the cost incurred at time $t=2$ is at least $b\cdot m_2$, and for instance $I_2$, the cost incurred at time $t=2$ is at least $b\cdot m_1$. Therefore, we have
\[\mathrm{ALG}(I_1) \geq b\cdot \mathbb{E}[m_0 + m_2],\quad \mathrm{ALG}(I_2) \geq b\cdot \mathbb{E}[m_0 + m_1].\]
Combining the fact that $m_0 + m_1 + m_2 = N$, we have
\[\frac{1}{2}\left( \mathrm{ALG}(I_1)+\mathrm{ALG}(I_2) \right)\geq \frac{b}{2}\cdot \mathbb{E}[m_0+m_1+m_2]=\frac{bN}{2}.\]
The optimal offline policy of both instances is to use FDC 1 and FDC 2 to fulfill the orders solely, and the cost incurred by the optimal offline policy is $\mathrm{OPT}(I_1)=\mathrm{OPT}(I_2)=f_1+f_2+2aN$. Therefore, we have
\[\max\left\{ \frac{\mathrm{ALG}(I_1)}{\mathrm{OPT}(I_1)},\ \frac{\mathrm{ALG}(I_2)}{\mathrm{OPT}(I_2)} \right\} \geq \frac{bN}{2(f_1+f_2+2aN)}.\]
Letting $N\to\infty$, we obtain the desired lower bound:
\[\CR(\mathrm{ALG}) \geq \frac{b}{4a}.\]
Combining the above two lower bounds, we complete the proof. \hfill\Halmos

\subsection{Proof of Proposition~\ref{prop:ub-lb-ratio-multi-varying}}
\label{subsec: proof-of-ub-lb-ratio-multi-varying}
For simplicity, we denote $f=\min_{k\in[K]}f_k$ as the minimal fixed cost among all FDCs. Note that $\max_{n\geq 2}\min\left\{ n,\ \frac{f_0}{f+na}\right\}=\sqrt{\frac{f_0}{a}+\frac{f^2}{4a^2}}-\frac{f}{2a}$. Thus we only need to prove
\begin{equation}
    \frac{\max\left\{ \sqrt{\frac{f_0}{a}+\frac{(f-b)^2}{4a^2}}-\frac{f-b}{2a},\ \frac{b}{a} \right\}}{\max\left\{\frac{b}{4a},\ \frac{1}{4}\left( \sqrt{\frac{f_0}{a}+\frac{f^2}{4a^2}}-\frac{f}{2a} \right) \right\}}\leq 2(1+\sqrt{5})
\end{equation}
for all $f_0\geq 0$, $f\geq 0$ and $b\geq a>0$. 


Let $u=\frac{af_0}{b^2}$ and $v=\frac{f}{b}$. Then we only need to prove
\begin{equation}
    \frac{\max\left\{ \sqrt{u+\frac{(v-1)^2}{4}}-\frac{v-1}{2},\ 1 \right\}}{\max\left\{ \sqrt{u+\frac{v^2}{4}}-\frac{v}{2},\ 1 \right\}}\leq \frac{1+\sqrt{5}}{2}
\end{equation}
for all $u\geq 0$, $v\geq 0$.

Note that $\sqrt{u+\frac{v^2}{4}}-\frac{v}{2}\geq 1$ is equivalent to $u\geq v+1$, and $\sqrt{u+\frac{(v-1)^2}{4}}-\frac{v-1}{2}\geq 1$ is equivalent to $u\geq v$. We consider the following three cases:
\begin{itemize}
    \item If $u\geq v+1$, then we have $\sqrt{u+\frac{v^2}{4}}-\frac{v}{2}\geq 1$ and $\sqrt{u+\frac{(v-1)^2}{4}}-\frac{v-1}{2}\geq 1$. Therefore, we have
    \[\frac{\max\left\{ \sqrt{u+\frac{(v-1)^2}{4}}-\frac{v-1}{2},\ 1 \right\}}{\max\left\{ \sqrt{u+\frac{v^2}{4}}-\frac{v}{2},\ 1 \right\}} = \frac{\sqrt{u+\frac{(v-1)^2}{4}}-\frac{v-1}{2}}{\sqrt{u+\frac{v^2}{4}}-\frac{v}{2}}\leq \frac{\sqrt{u+\frac{1}{4}}+\frac{1}{2}}{\sqrt{u}} \leq \frac{1+\sqrt{5}}{2}.\]
    \item If $v\leq u < v+1$, then we have $\sqrt{u+\frac{(v-1)^2}{4}}-\frac{v-1}{2}\geq 1$ and $\sqrt{u+\frac{v^2}{4}}-\frac{v}{2}< 1$. Therefore, we have
    \[\frac{\max\left\{ \sqrt{u+\frac{(v-1)^2}{4}}-\frac{v-1}{2},\ 1 \right\}}{\max\left\{ \sqrt{u+\frac{v^2}{4}}-\frac{v}{2},\ 1 \right\}} = \sqrt{u+\frac{(v-1)^2}{4}}-\frac{v-1}{2}\leq \sqrt{v+1+\frac{(v-1)^2}{4}}-\frac{v-1}{2} \leq \frac{1+\sqrt{5}}{2}.\]

    \item If $u < v$, then we have $\sqrt{u+\frac{(v-1)^2}{4}}-\frac{v-1}{2}< 1$ and $\sqrt{u+\frac{v^2}{4}}-\frac{v}{2}< 1$. Therefore, we have $\max\left\{ \sqrt{u+\frac{(v-1)^2}{4}}-\frac{v-1}{2},\ 1 \right\}/\max\left\{ \sqrt{u+\frac{v^2}{4}}-\frac{v}{2},\ 1 \right\} = 1$.
\end{itemize}
Combining the above three cases, we complete the proof. \hfill\Halmos

\subsection{Proof of Claim~\ref{claim:proof-cc-vp-tool}}
\label{subsec:proof-claim-proof-cc-vp-tool}
Fix an item $i\in[n]$. Without loss of generality, we assume that $c_1^i\leq \dots \leq c_{k_0}^i < c_0^i \leq c_{k_0 + 1}^i \leq \dots \leq c_K^i$. Note that 
\begin{align}
    & \sum_{t\in B}\left( \sum_{k=0}^K c_k^i\cdot \hat{m}_{k,t}^i \right) - \sum_{t\in B}\left( \sum_{k=0}^K c_k^i\cdot m_{k,t}^{i,*} \right) = \sum_{t\in B}\left[ c_0^i\cdot \left(\hat{m}_{0,t}^i - m_{0,t}^{i,*}\right)+\sum_{k=1}^K c_k^i \cdot \left( \hat{m}_{k,t}^i - m_{k,t}^{i,*} \right) \right] \notag \\
    &\qquad =\sum_{t\in B}\left[ -c_0^i\cdot \sum_{k=1}^K\left( \hat{m}_{k,t}^i - m_{k,t}^{i,*} \right) +\sum_{k=1}^K c_k^i \cdot \left( \hat{m}_{k,t}^i - m_{k,t}^{i,*} \right) \right] = \sum_{k=1}^K \left( c_k^i - c_0^i \right) \cdot \sum_{t\in B}\left( \hat{m}_{k,t}^i - m_{k,t}^{i,*} \right) \notag \\
    &\qquad \leq \sum_{k=1}^{k_0}\left( c_k^i - c_0^i \right) \cdot \sum_{t\in B}\left( \hat{m}_{k,t}^i - m_{k,t}^{i,*} \right),
    \label{eq:claim-proof-cc-vp-tool-1}
\end{align}
where the inequality is because $c_k^i\geq c_0^i$ and $\hat{m}_{k,t}^i=0$ for all $k>k_0$, by the \nsc{\multiinvaralg} algorithm. We further have
\begin{equation}
    \text{Eq.~\eqref{eq:claim-proof-cc-vp-tool-1}} = \sum_{j=1}^{k_0}\left( c_j^i-c_{j+1}^i \right)\cdot \sum_{k=1}^j \sum_{t\in B}\left( \hat{m}_{k,t}^i - m_{k,t}^{i,*} \right)+\left( c_{k_0}^i - c_0^i \right) \cdot \sum_{k=1}^{k_0} \sum_{t\in B}\left( \hat{m}_{k,t}^i - m_{k,t}^{i,*} \right) \leq 0,
    \label{eq:claim-proof-cc-vp-tool-2}
\end{equation}
where the inequality is due to that $c_j^i\leq c_{j+1}^i ~\forall j \in \{1,\dots,k_0-1\}$, $c_{k_0}^i \leq c_0^i$, and that $\sum_{k=1}^j \sum_{t\in B}\left( \hat{m}_{k,t}^i - m_{k,t}^{i,*} \right) \geq 0$ holds for all $j\in \{1,\dots,k_0\}$ (which is guaranteed by Lemma~\ref{lemma: key lemma of GPG policy}, as $B$ contains all time periods in which the policy uses at least one FDC to fulfill the order).

Combining Eq.~\eqref{eq:claim-proof-cc-vp-tool-1} and Eq.~\eqref{eq:claim-proof-cc-vp-tool-2}, we prove the claim.
\hfill\Halmos

\subsection{Proof of Theorem~\ref{thm: lower bound with time invariant cost}}

Given the number of FDCs $K$ and fixed costs of FDCs and RDC $(f_0, f_1,\dots,f_K)$. Without loss of generality, we assume $f_1=\min_{k\in[K]} f_k$ is the minimal fixed cost among all FDCs. For any fixed $s,d,n$, we consider the following instance, illustrated in Table~\ref{tab: lower bound instance}. Each block represents $s$ or $d$ types of items, so the total number of item types is $sn+Kdn$. The initial inventory of FDC 1 is $I_{1,0}^i=1$ for all items $i$, and the initial inventory of other FDC $k$ ($k\geq2$) is $I_{k,0}^i=1$ for all items $i$ in the $k$-th row of the table, and $I_{k,0}^i=0$ for other items. The variable costs are set to be the same among all items: $c_0^i=c_0,\ c_1^i=c_1,\ c_{k}^i=c_2$ for all $k=2,\dots,K$. 

\renewcommand{\arraystretch}{1.5}
\begin{table}[ht]
    \centering
    \begin{tabular}{ccccc}
    & \multicolumn{4}{c}{FDC 1}                                                                       \\ \cline{2-5} 
    \multicolumn{1}{c|}{}      & \multicolumn{1}{c|}{$s$} & \multicolumn{1}{c|}{$s$} & \multicolumn{1}{c|}{$\dots$} & \multicolumn{1}{c|}{$s$}
    \\ \cline{2-5} 
    \multicolumn{1}{c|}{FDC 2}  & \multicolumn{1}{c|}{$d$} & \multicolumn{1}{c|}{$d$} & \multicolumn{1}{c|}{$\dots$} & \multicolumn{1}{c|}{$d$} 
    \\ \cline{2-5} 
    \multicolumn{1}{c|}{$\vdots$}      & \multicolumn{1}{c|}{$\vdots$}  & \multicolumn{1}{c|}{$\vdots$}  & \multicolumn{1}{c|}{$\ddots$} & \multicolumn{1}{c|}{$\vdots$}  
    \\ \cline{2-5} 
    \multicolumn{1}{c|}{FDC $K$} & \multicolumn{1}{c|}{$d$} & \multicolumn{1}{c|}{$d$} & \multicolumn{1}{c|}{$\dots$} & \multicolumn{1}{c|}{$d$} 
    \\ \cline{2-5} 
    \multicolumn{1}{c|}{} & \multicolumn{1}{c|}{$d$} & \multicolumn{1}{c|}{$d$} & \multicolumn{1}{c|}{$\dots$} & \multicolumn{1}{c|}{$d$} 
    \\ \cline{2-5} 
    \multicolumn{1}{l}{}       & \multicolumn{1}{l}{1}  & \multicolumn{1}{l}{2}  & \multicolumn{1}{l}{$\dots$}  & \multicolumn{1}{l}{$n$} 
    
    \end{tabular}
    \caption{Instance for Lower Bound with Time-Invariant Variable Costs (Theorem~\ref{thm: lower bound with time invariant cost})}
    \label{tab: lower bound instance}
\end{table}

Now we consider two instances. For the first instance $I_1$, the time period is $T=1$, and the customer order requests one unit of each item in the first $K$ row (i.e. except for the last row). For the second instance $I_2$, the time period is $T=n+1$. The customer order at time period $t=1$ is the same as instance $I_1$. At each time period $t=2,3,\dots,n+1$, the customer order requests one unit of each item in the $(t-1)$-th column of the table.

We first analyze the performance of the optimal offline policy in both instances. For instance $I_1$, the optimal offline policy uses only FDC 1 to fulfill the order, and the cost incurred by the optimal offline policy is 
\[\mathrm{OPT}(I_1)=f_1+[sn+(K-1)dn]c_1.\]
For instance $I_2$, we consider the following policy: at time $t=1$, the policy uses RDC to fulfill the items in the first row, and use FDC $k$ ($k\geq2$) to fulfill the items in the $k$-th row; at each time period $t=2,3,\dots,n+1$, the policy uses FDC 1 to fulfill the items in the $(t-1)$-th column. The cost incurred by this policy is $f_0+\sum_{k=2}^K f_k+snc_0+d(K-1)nc_2+[f_1+(s+Kd)c_1]n$. Thus the cost incurred by the optimal offline policy is no larger than \[\mathrm{OPT}(I_2)\leq f_0+\sum_{k=2}^K f_k+snc_0+d(K-1)nc_2+[f_1+(s+Kd)c_1]n.\]

By Yao's principle, the lower bound on the competitive ratio of any randomized online fulfillment policy is lower bounded by that of any deterministic online fulfillment policy under a worst-case distribution over the above two instances, i.e. 
\[\CR(\mathrm{ALG}) \geq \max_{p\in[0,1]}\min_{\text{deterministic } \mathrm{ALG}}\left\{ p\cdot \frac{\mathrm{ALG}(I_1)}{\mathrm{OPT}(I_1)}+(1-p)\cdot\frac{\mathrm{ALG}(I_2)}{\mathrm{OPT}(I_2)} \right\}.\]

Let $\mathrm{ALG}$ be any deterministic online fulfillment policy. We can expect that if $\mathrm{ALG}$ uses more FDC $k$ ($k\geq 2$) or RDC at time $t=1$, then it behaves badly on instance $I_1$; conversely, if $\mathrm{ALG}$ uses more FDC 1, then it behaves badly on instance $I_2$. To formally illustrate this idea, we consider the following events based on the decision of $\mathrm{ALG}$ at time period $t=1$:
\[E_1=\left\{ \text{At least $\theta ns$ item in the first row are fulfilled by RDC at time $t=1$} \right\},\]
\[E_k=\left\{ \text{At least $\theta nd$ items in the $k$-th row are fulfilled by RDC or FDC $k$ at time $t=1$} \right\},\ 2\leq k\leq K.\]
Here $\theta$ is a parameter to be specified later. Let $E=\bigcup_{k=1}^K E_k$.

\medskip
\noindent \underline{Case 1.} Suppose event $E$ occurs. Then, for instance $I_1$, the incurred cost by $\mathrm{ALG}$ is at least 
\[\mathrm{ALG}(I_1) \geq \theta n \cdot \min\{sc_0,dc_2\}.\]
Therefore, the competitive ratio of $\mathrm{ALG}$ is at least
\[\frac{\mathrm{ALG}(I_1)}{\mathrm{OPT}(I_1)} \geq \frac{\theta n \cdot \min\{sc_0,dc_2\}}{f_1+[sn+(K-1)dn]c_1}.\]
If we further let $n\to\infty$, then the competitive ratio of $\mathrm{ALG}$ is at least 
\[\lim_{n\to\infty}\frac{\mathrm{ALG}(I_1)}{\mathrm{OPT}(I_1)} \geq\frac{\theta \cdot \min\{sc_0,dc_2\}}{[s+(K-1)d]c_1}.\]

\medskip
\noindent \underline{Case 2.} Suppose event $E$ does not occur. Then at time $t=1$, at least $(1-\theta)ns$ items in the first row are fulfilled by FDC 1, and at least $(1-\theta)nd$ items in row $k$ are fulfilled by FDC 1 for each $2\leq k\leq K$. Thus, for the first row, there are at most $\theta ns/(\frac{s}{2})=2\theta n$ blocks with at least $\frac{s}{2}$ items left unused in FDC 1; similarly, for each row $k=2,\dots,K$, there are at most $\theta nd/(\frac{d}{2})=2\theta n$ blocks with at least $\frac{d}{2}$ items left unused in FDC 1. 

Let $P$ be the following set of columns:
\[P=\left\{ j\in[n]: \text{FDC 1 remains at most half of items in $k$-th block of column $j$, $\forall\,1\leq k \leq K$} \right\}.\]
Then the number of elements in $P$ can be bounded by 
\[|P|\geq n-K\cdot 2\theta n=(1-2K\theta)n.\]
For each $j\in P$, consider the fulfillment cost at time $t=j+1$. Since FDC 1 retains at most half of the items in the first block of column $j$, the remaining items must be fulfilled by the RDC. For $2\leq k\leq K$, because FDC 1 retains at most half of the items in the $k$-th block of column $j$, either the RDC fulfills at least half of these items or FDC $k$ is used. For $k=1$, because the order requests one unit of each item in the last block of column $j$, either the RDC fulfills all of these items or FDC 1 is used. Therefore, the fulfillment cost incurred by $\mathrm{ALG}$ at time $t=j+1$ is at least 
\[f_0+\sum_{k=2}^K \min\left\{f_k,\frac{1}{2}dc_0\right\}+\min\{f_1,dc_0\}\geq f_0+\sum_{k=1}^K \min\left\{f_k,\frac{1}{2}dc_0\right\}.\]
Therefore, for instance $I_2$, the total fulfillment cost by $\mathrm{ALG}$ is at least 
\[\mathrm{ALG}(I_2)\geq (1-2K\theta)n\cdot \left(  f_0+\sum_{k=1}^K \min\left\{f_k,\frac{1}{2}dc_0\right\} \right),\]
and the competitive ratio of $\mathrm{ALG}$ is at least
\[\frac{\mathrm{ALG}(I_2)}{\mathrm{OPT}(I_2)} \geq \frac{(1-2K\theta)n\cdot \left(  f_0+\sum_{k=1}^K \min\left\{f_k,\frac{1}{2}dc_0\right\} \right)}{f_0+\sum_{k=2}^K f_k+snc_0+d(K-1)nc_2+[f_1+(s+Kd)c_1]n}.\]
If we further let $n\to\infty$, then the competitive ratio of $\mathrm{ALG}$ is at least
\[\lim_{n\to\infty}\frac{\mathrm{ALG}(I_2)}{\mathrm{OPT}(I_2)} \geq (1-2K\theta) \cdot \frac{ f_0+\sum_{k=1}^K \min\left\{f_k,\frac{1}{2}dc_0\right\} }{sc_0+d(K-1)c_2+f_1+(s+Kd)c_1}.\]

Combining the above two cases, since $s,d,\theta,c_0,c_1,c_2$ can be chosen arbitrarily, then the lower bound of the competitive ratio of $\mathrm{ALG}$ can be written as
\begin{align*}
    \CR(\mathrm{ALG}) &\geq \sup_{\substack{s,d,\theta,p\\c_0,c_1,c_2}}\min\left\{ p\cdot\frac{\theta \cdot \min\{sc_0,dc_2\}}{[s+(K-1)d]c_1},\  (1-p)\cdot\frac{ (1-2K\theta) \cdot\left( f_0+\sum_{k=1}^K \min\left\{f_k,\frac{1}{2}dc_0\right\} \right) }{sc_0+d(K-1)c_2+f_1+(s+Kd)c_1} \right\}\\
    &\geq \sup_{\substack{s,d,\theta,p\\c_0,c_2}}\lim_{c_1\to0+}\min\left\{ p\cdot\frac{\theta \cdot \min\{sc_0,dc_2\}}{[s+(K-1)d]c_1},\  (1-p)\cdot\frac{ (1-2K\theta) \cdot\left( f_0+\sum_{k=1}^K \min\left\{f_k,\frac{1}{2}dc_0\right\} \right) }{sc_0+d(K-1)c_2+f_1+(s+Kd)c_1} \right\}\\
    &= \sup_{\substack{s,d,\theta,p\\c_0,c_2}} (1-p)\cdot\frac{ (1-2K\theta) \cdot\left( f_0+\sum_{k=1}^K \min\left\{f_k,\frac{1}{2}dc_0\right\} \right) }{sc_0+d(K-1)c_2+f_1}\\
    &\geq \sup_{s,d,c_0}\lim_{\substack{c_2\to0+\\\theta\to0+}} \frac{ (1-2K\theta) \cdot\left( f_0+\sum_{k=1}^K \min\left\{f_k,\frac{1}{2}dc_0\right\} \right) }{sc_0+d(K-1)c_2+f_1}\\
    &= \sup_{s,d,c_0} \frac{f_0+\sum_{k=1}^K \min\left\{f_k,\frac{1}{2}dc_0\right\}}{sc_0+f_1}\\
    &=\frac{f_0+\sum_{k=1}^K f_k}{f_1}.
\end{align*}
This completes the proof. \hfill\Halmos

\subsection{Proof of Claim~\ref{claim:proof-cc-vp-barV-time-varying-single-fdc}}
\label{subsec:proof-claim-proof-cc-vp-barV-time-varying-single-fdc}
For each time period $t$, we need to prove
\begin{equation}
    \bar{V_t}\leq \max\left\{ 1+\frac{f_0}{f_1},\ 1+\sqrt{\frac{b}{a}} \right\} \cdot V_t^*.
    \label{eq:intermediate-cost-upper-bound-time-varying-single-fdc}
\end{equation}
Furthermore, we need to prove
\begin{equation}
    \sum_{t=1}^T V_t\leq \sum_{t=1}^T \bar{V_t}.
    \label{eq:total-cost-upper-bound-by-intermediate-cost-time-varying-single-fdc}
\end{equation}

\medskip
\noindent \underline{We first prove Eq.~\eqref{eq:intermediate-cost-upper-bound-time-varying-single-fdc} for each period $t$.}
For time periods $t$ where OPT uses only the RDC to fulfill the order, the cost of OPT is $V_t^*=f_0 + \sum_{i=1}^n c_{0,t}^i\cdot S_t^i$. If \nsc{\singvaryingalg} uses only the RDC as well, then we have the intermediate cost $\bar{V_t}=V_t^*$. If \nsc{\singvaryingalg} uses at least one FDC, then we know that the gating condition is not triggered, i.e., $\sum_{k=0,1}\left[ f_k\cdot \mathbb{I}\left( \sum_{i=1}^n m_{k,t}^i>0 \right) + \sum_{i=1}^n c_{k,t}^i\cdot m_{k,t}^i \right] \leq f_0 + \sum_{i=1}^n c_{0,t}^i\cdot S_t^i$; in this case, we have
\begin{equation*}
    \bar{V_t}=f_0 + f_1\cdot \mathbb{I}\left( \sum_{i=1}^n m_{1,t}^i>0 \right) + \sqrt{\frac{b}{a}}\cdot\sum_{i=1}^n c_{0,t}^i\cdot S_t^i\leq 2f_0+\left(1+\sqrt{\frac{b}{a}}\right)\sum_{i=1}^n c_{0,t}^i\cdot S_t^i\leq \left( 1+\sqrt{\frac{b}{a}} \right)V_t^*.
\end{equation*}
For time periods $t$ where OPT uses at least one FDC to fulfill the order, we bound the intermediate cost as follows: 
\begin{align*}
    \bar{V_t}=f_0 + f_1 + \sqrt{\frac{b}{a}}\cdot\left( \sum_{i=1}^n c_{0,t}^i\cdot m_{k,t}^{i,*}+\sum_{i=1}^n c_{1,t}^i\cdot m_{k,t}^{i,*} \right)\leq\max\left\{ 1+\frac{f_0}{f_1},\sqrt{\frac{b}{a}} \right\}\cdot V_t^*.
\end{align*}
Combining both cases above, we prove the upper bound in Eq.~\eqref{eq:intermediate-cost-upper-bound-time-varying-single-fdc}.

\medskip\noindent \underline{Next, we prove Eq.~\eqref{eq:total-cost-upper-bound-by-intermediate-cost-time-varying-single-fdc}.}  Let $A$ be the set of time periods that \nsc{\singvaryingalg} uses at least one FDC to fulfill the order.
\begin{itemize}
    \item \underline{Case 1.} For periods $t\notin A$, we have $V_t=f_0 + \sum_{i=1}^n c_{0,t}^i\cdot S_t^i$. 
    \begin{itemize}
        \item \underline{Case 1a.} If OPT uses RDC only, then we further have $V_t\leq\bar{V_t}$.
        \item \underline{Case 1b.}  If OPT uses at least one FDC, since the gating condition is triggered, we have 
\begin{align}
    V_t&=f_0+\sum_{i=1}^nc_{0,t}^i\cdot S_t^i \leq \sum_{k=0,1}\left[ f_k\cdot \mathbb{I}\left( \sum_{i=1}^n \hat{m}_{k,t}^i>0 \right) + \sum_{i=1}^n c_{k,t}^i\cdot \hat{m}_{k,t}^i \right]\notag\\
    &\qquad \leq f_0+f_1+ \sum_{k=0,1} \sum_{i=1}^n c_{k,t}^i\cdot \hat{m}_{k,t}^i  = \bar{V}_t + \left( \sum_{k=0,1} \sum_{i=1}^n c_{k,t}^i\cdot \hat{m}_{k,t}^i - \sqrt{\frac{b}{a}}\cdot\sum_{k=0,1} \sum_{i=1}^n c_{k,t}^i\cdot m_{k,t}^{i,*} \right).
    \label{eq:cost-upper-bound-by-intermediate-cost-case1-time-varying-single-fdc}
\end{align}
    \end{itemize}
    \item \underline{Case 2.} For periods $t\in A$, we have $m_{k,t}^i=\hat{m}_{k,t}^i$ for all $k,i$. If OPT uses RDC only, the definition of $\bar{V}_t$ implies that $\bar{V}_t=f_0 + f_1\cdot \mathbb{I}\left( \sum_{i=1}^n m_{1,t}^i>0 \right) + \sqrt{\frac{b}{a}}\cdot\sum_{i=1}^n \sum_{k=0,1} c_{k,t}^i\cdot m_{k,t}^{i,*}$, and therefore
\begin{equation}
    V_t=\sum_{k=0,1} \left[ f_k \cdot \mathbb{I}\left(\sum_{i=1}^n m_{k,t}^i > 0\right) + \sum_{i=1}^n c_{k,t}^i \cdot \hat{m}_{k,t}^i \right] \leq \bar{V_t}+ \left( \sum_{k=0,1} \sum_{i=1}^n c_{k,t}^i\cdot \hat{m}_{k,t}^i - \sqrt{\frac{b}{a}}\cdot\sum_{k=0,1} \sum_{i=1}^n c_{k,t}^i\cdot m_{k,t}^{i,*} \right).
    \label{eq:cost-upper-bound-by-intermediate-cost-case2-time-varying-single-fdc}
\end{equation}
If OPT uses at least one FDC, then we may also verify that have Eq.~\eqref{eq:cost-upper-bound-by-intermediate-cost-case2-time-varying-single-fdc} holds.
\end{itemize}
Summarizing the cases above, let $B\supseteq A$ be the set of time periods that either \nsc{\singvaryingalg} or $\mathrm{OPT}$ uses at least one FDC to fulfill the order, and we have that
\begin{align}
    V_t \leq \bar{V_t}+ \sum_{t \in B} \left( \sum_{k=0,1} \sum_{i=1}^n c_{k,t}^i\cdot \hat{m}_{k,t}^i - \sqrt{\frac{b}{a}}\cdot\sum_{k=0,1} \sum_{i=1}^n c_{k,t}^i\cdot m_{k,t}^{i,*} \right). \label{eq:cost-upper-bound-by-intermediate-cost-time-varying-single-fdc}
\end{align}

On the other hand, the following claim holds.
\begin{claim}
    \label{claim:proof-cc-vp-tool-time-varying-single-fdc}
    For each item $i\in[n]$, we have
    \[
    \sum_{t\in B}\left( \sum_{k=0,1}c_{k,t}^i\cdot \hat{m}_{k,t}^i \right)\leq \sqrt{\frac{b}{a}}\cdot \sum_{t\in B}\left( \sum_{k=0,1} c_{k,t}^i\cdot m_{k,t}^{i,*} \right).\]
\end{claim}
The proof of Claim~\ref{claim:proof-cc-vp-tool-time-varying-single-fdc} follows arguments similar to those used for Claim~\ref{claim:proof-os-fp-tool} and is deferred to Section~\ref{subsec:proof-claim-proof-cc-vp-tool-time-varying-single-fdc} of the e-companion. A major distinction between the proofs of Claim~\ref{claim:proof-cc-vp-tool-time-varying-single-fdc} and Claim~\ref{claim:proof-os-fp-tool} is that the former requires a modified version of our key technical Lemma~\ref{lemma: key lemma of GPG policy}. Specifically, the original \nsc{\gatedprioritygreedy} framework and Lemma~\ref{lemma: key lemma of GPG policy} are formulated only for time-independent priority rules. To handle the present setting, we establish below a variant of Lemma~\ref{lemma: key lemma of GPG policy} for the extended \nsc{\gatedprioritygreedy} framework with time-dependent priority rules in the case $K=1$. This result is stated as Lemma~\ref{lemma: key lemma of GPG policy K=1}. Its proof closely parallels that of Lemma~\ref{lemma: key lemma of GPG policy} and is deferred to Section~\ref{subsec:proof of key lemma of GPG policy K=1} of the e-companion.
\begin{lemma}
    \label{lemma: key lemma of GPG policy K=1}
    Consider the case $K=1$. Suppose an algorithm satisfies the following conditions: for any $i,k,t$, either $m_{1,t}^i=\max\left\{ S_t^i, I_{1,t-1}^i \right\}=:\hat{m}_{1,t}^i$, or $m_{1,t}^i=0$. 

    Fix an item $i$. Define $C_i=\{t: {m}_{1,t}^i>0\}$ as the set of periods in which the actual fulfillment plan uses FDC to fulfill demand for item $i$. Then for any set $B\supseteq C_i$, we have
    \[\sum_{t\in B}\hat{m}_{1,t}^i\geq \sum_{t\in B}m_{1,t}^{i,*}.\]
\end{lemma}

Combining Eq.~\eqref{eq:cost-upper-bound-by-intermediate-cost-time-varying-single-fdc} and Claim~\ref{claim:proof-cc-vp-tool-time-varying-single-fdc}, we establish Eq.~\eqref{eq:total-cost-upper-bound-by-intermediate-cost-time-varying-single-fdc}. \hfill\Halmos

\subsection{Proof of Claim~\ref{claim:proof-cc-vp-tool-time-varying-single-fdc}}
\label{subsec:proof-claim-proof-cc-vp-tool-time-varying-single-fdc}
Fix an item $i\in[n]$. Let $A_i=\left\{t:c_{1,t}^i<\sqrt{\frac{a}{b}}\cdot c_{0,t}^i\right\}$. We consider the terms corresponding to $t\in B\setminus A_i$ and $t\in B\cap A_i$, respectively. For $t\in B\setminus A_i$, we have $c_{0,t}^i\leq \sqrt{\frac{b}{a}}\cdot c_{1,t}^i$, and $\hat{m}_{0,t}^i=S_t^i,~\hat{m}_{1,t}^i=0$. Therefore, 
\begin{equation}
    \sum_{k=0,1} c_{k,t}^i\cdot \hat{m}_{k,t}^i = c_{0,t}^i\cdot S_t^i = c_{0,t}^i\cdot m_{k,t}^{i,*} + c_{0,t}^i\cdot m_{1,t}^{i,*} \leq \sqrt{\frac{b}{a}}\cdot \left( c_{0,t}^i\cdot m_{0,t}^{i,*} + c_{1,t}^i\cdot m_{1,t}^{i,*} \right).
    \label{eq:claim-proof-cc-vp-tool-time-varying-single-fdc-1}
\end{equation}

For $t\in B\cap A_i$, we have $c_{1,t}^i< \sqrt{\frac{a}{b}}\cdot c_{0,t}^i$, and thus $c_{1,t}^i<\sqrt{\frac{a}{b}}\cdot c_{0,t}^i\leq\sqrt{ab}$ and $c_{0,t}^i>\sqrt{\frac{b}{a}}\cdot c_{1,t}^i\geq \sqrt{ab}$. Then we have
\begin{align*}
    & \sum_{t\in B\cap A_i}\left( c_{0,t}^i\cdot \hat{m}_{0,t}^i+c_{1,t}^i\cdot \hat{m}_{1,t}^i \right)\leq \sum_{t\in B\cap A_i}\left( b\cdot \hat{m}_{0,t}^i+\sqrt{ab}\cdot \hat{m}_{1,t}^i \right),\\
    & \sum_{t\in B\cap A_i}\left( c_{0,t}^i\cdot {m}_{0,t}^{i,*}+c_{1,t}^i\cdot {m}_{1,t}^{i,*} \right)\geq \sum_{t\in B\cap A_i}\left( \sqrt{ab}\cdot {m}_{0,t}^{i,*}+a\cdot {m}_{1,t}^{i,*} \right).
\end{align*}
Therefore, we only need to prove the following inequality:
\begin{equation}
    \sum_{t\in B\cap A_i}\left( b\cdot \hat{m}_{0,t}^i+\sqrt{ab}\cdot \hat{m}_{1,t}^i \right)\leq \sum_{t\in B\cap A_i}\left( b\cdot {m}_{0,t}^{i,*}+\sqrt{ab}\cdot {m}_{1,t}^{i,*} \right).
    \label{eq:claim-proof-cc-vp-tool-time-varying-single-fdc-2}
\end{equation}
Note that
\begin{align*}
    & \quad\sum_{t\in B\cap A_i}\left( b\cdot \hat{m}_{0,t}^i+\sqrt{ab}\cdot \hat{m}_{1,t}^i \right) - \sum_{t\in B\cap A_i}\left( b\cdot {m}_{0,t}^{i,*}+\sqrt{ab}\cdot {m}_{1,t}^{i,*} \right)\\
    &=b\cdot \sum_{t\in B\cap A_i}\left( \hat{m}_{0,t}^i-m_{0,t}^{i,*} \right)+\sqrt{ab}\cdot \sum_{t\in B\cap A_i}\left( \hat{m}_{1,t}^i-m_{1,t}^{i,*} \right)\leq \left( \sqrt{ab}-b \right)\cdot \sum_{t\in B\cap A_i}\left( \hat{m}_{1,t}^i-m_{1,t}^{i,*} \right).
\end{align*}
Since $B\cap A_i$ contains all the periods in which FDC is used to fulfill demand for item $i$, and $B\cap A_i\subseteq A_i$ is contained in the set of periods in which FDC is prioritized over RDC for item $i$, we can apply Lemma~\ref{lemma: key lemma of GPG policy K=1} to $B\cap A_i$ and obtain that $\sum_{t\in B\cap A_i}\hat{m}_{1,t}^i\geq\sum_{t\in B\cap A_i}m_{1,t}^{i,*}$. This proves Eq.~\eqref{eq:claim-proof-cc-vp-tool-time-varying-single-fdc-2}.

Combining Eq.~\eqref{eq:claim-proof-cc-vp-tool-time-varying-single-fdc-1} and Eq.~\eqref{eq:claim-proof-cc-vp-tool-time-varying-single-fdc-2}, we prove the claim. 
\hfill\Halmos

\subsection{Proof of Lemma~\ref{lemma: key lemma of GPG policy K=1}}
\label{subsec:proof of key lemma of GPG policy K=1}
It is clear that $m_{1,t}^{i,*}\leq S_t^i$ and $\sum_{t\in B}m_{1,t}^{1,*}\leq I_{1,0}^i$, which implies that $\sum_{t\in B}m_{1,t}^{i,*}\leq \min\left\{ \sum_{t\in B}S_t^i,~I_{1,0}^i \right\}$. Now we only need to prove
\begin{equation}
    \sum_{t\in B}\hat{m}_{1,t}^i=\sum_{t\in B} \min\left\{ S_t^i,~I_{1,t-1}^i \right\} \geq \min\left\{ \sum_{t\in B}S_t^i,~I_{1,0}^i \right\}.
    \label{eq: important inequality of lemma key lemma K=1}
\end{equation}
We prove Equation~\eqref{eq: important inequality of lemma key lemma K=1} via two steps.

\noindent\underline{{Step I.}} We prove that $\sum_{t\in C_i}\hat{m}_{1,t}^i = \min\left\{ \sum_{t\in C_i}S_t^i,\ I_{1,0}^i \right\}$. We consider the following two cases:
\begin{itemize}
    \item Suppose $\sum_{t\in C_i} S_t^i \leq I_{1,0}^i$. For periods $t\in C_i$, the extended \nsc{\gatedprioritygreedy} policy fulfill the order by greedy quantities, so we have $\hat{m}_{1,t}^i=m_{1,t}^i$. Then we have
    \[I_{1,t-1}^i = I_{1,0}^i-\sum_{\tau=1}^{t-1} m_{1,\tau}^i = I_{1,0}^i-\sum_{\substack{\tau\leq t-1\\\tau\in C_i}} \hat{m}_{1,\tau}^i,\]
    and then
    \[I_{1,t-1}^i = I_{1,0}^i-\sum_{\substack{\tau\leq t-1\\\tau\in C_i}} \hat{m}_{1,\tau}^i \geq I_{1,0}^i-\sum_{\substack{\tau\leq t-1\\\tau\in C_i}}S_\tau^i\geq \sum_{\substack{\tau\geq t\\t\in C_i}}S_t^i.\]
    Therefore, we have $\min\left\{ S_t^i,~I_{1,t-1}^i \right\}= S_t^i$ for $t\in C_i$, and thus
    \[\sum_{t\in C_i}\hat{m}_{1,t}^i = \sum_{t\in C_i} \min\left\{ S_t^i,~I_{1,t-1}^i \right\} = \sum_{t\in C_i} S_t^i = \min\left\{ \sum_{t\in C_i}S_t^i,~I_{1,0}^i \right\}.\]

    \item Suppose $\sum_{t\in C_i} S_t^i > I_{1,0}^i$. We first prove $I_{1,T}^i=0$. Assume not, i.e. $I_{1,T}^i>0$, then $I_{1,t}^i>0$ for all $t$. Then for all $t\in C_i$, the extended \nsc{\gatedprioritygreedy} fulfillment policy fulfills orders from the FDCs, 
    \[\min\left\{ S_t^i,~I_{1,t-1}^i \right\}=\hat{m}_{1,t}^i=m_{1,t}^i=I_{1,t-1}^i-I_{1,t}^i<I_{1,t-1}^i.\]
    This implies that for all $t\in C_i$, $S_t^i<I_{1,t-1}^i$ and $I_{1,t}^i=I_{1,t-1}^i - S_t^i$. For $t\notin C$, we have $I_{1,t}^i=I_{1,t-1}^i$. Therefore, we have
    \[I_{1,T}^i = I_{1,0}^i - \sum_{t\in C_i} S_t^i \leq 0,\]
    which contradicts the assumption that $I_{1,T}^i>0$. Thus we have $I_{1,T}^i=0$, and therefore, 
    \[\sum_{t\in C_i}\hat{m}_{1,t}^i = \sum_{t\in C_i} \min\left\{ S_t^i,~I_{1,t-1}^i \right\} = I_{1,0}^i-I_{1,T}^i = I_{1,0}^i .\]
\end{itemize}

\noindent\underline{{Step II.}} We prove Equation~\eqref{eq: important inequality of lemma key lemma K=1}. We consider the following two cases:
\begin{itemize}
    \item If there exists $t\in B\backslash C_i$ such that $S_t^i>I_{1,t-1}^i$, since $I_{k,t-1}^i\geq\sum_{k=1}^jI_{k,T}^i$, then we have $\sum_{t\in B\backslash C_i}\hat{m}_{1,t}^i = \sum_{t\in B\backslash C_i} \min\left\{ S_t^i,\ I_{1,t-1}^i \right\}\geq I_{1,T}^i$. Therefore, we have
    \[\sum_{t\in B}\hat{m}_{1,t}^i = \sum_{t\in B\backslash C_i}\hat{m}_{1,t}^i + \sum_{t\in C_i}\hat{m}_{1,t}^i \geq I_{1,T}^i + \left( I_{1,0}^i - I_{1,T}^i \right) = I_{1,0}^i.\]

    \item If for all $t\in B\backslash C_i$, we have $S_t^i\leq I_{1,t-1}^i$, then we can obtain $\sum_{t\in B\backslash C_i}\hat{m}_{1,t}^i = \sum_{t\in B\backslash C_i}\min\left\{ S_t^i,~I_{1,t-1}^i \right\}=\sum_{t\in B\backslash C_i} S_t^i$. Therefore, we have
    \[\sum_{t\in B}\hat{m}_{1,t}^i = \sum_{t\in B\backslash C_i}\hat{m}_{1,t}^i + \sum_{t\in C_i}\hat{m}_{1,t}^i \geq \sum_{t\in B\backslash C_i} S_t^i + \min\left\{ \sum_{t\in C_i}S_t^i,\ I_{1,0}^i \right\} \geq \min\left\{ \sum_{t\in B}S_t^i,\ I_{1,0}^i \right\}.\]
\end{itemize}

This completes the proof of Lemma~\ref{lemma: key lemma of GPG policy K=1}. \hfill\Halmos

\subsection{Proof of Theorem~\ref{thm: upper bound with bounded cost k=1 improved}}
\label{subsec:proof-upper-bound-time-varying-single-fdc-improved-alg}
Define $A=\left\{ t:\sum_{t=1}^n S_t\leq \theta \right\}$ and $B_t=\left\{ i: c_{1,t}^i\leq c_{0,t}^i/\eta \right\}$, and $\hat{m}_{1,t}^i=\min\{S_t^i,I_{t-1}^i\}$. 

\medskip\noindent \underline{Step I.} We construct intermediate values $\bar{V}_t$. Consider the following four cases:

\underline{Case 1.} $t\notin A$. 
Note that
\[V_t\leq f_0+f_1+\sum_{i\notin B_t}c_{0,t}^i\cdot S_t^i + \sum_{i\in B_t} \left( c_{0,t}^i\cdot m_{0,t}^i + c_{1,t}^i\cdot m_{1,t}^i \right).\]
Define 
\[\bar{V}_t=\sum_{i\notin B_t}\max\left\{ a,\frac{c_{0,t}^i}{\eta} \right\}\cdot S_t^i + \sum_{i\in B_t} \left( \eta a\cdot m_{0,t}^i + a\cdot m_{1,t}^i \right).\]
Since $t\notin A$, then we have $\sum_{i=1}^nS_t^i > \theta$. For $i\in B_t$, we have $c_{1,t}^i\leq c_{0,t}^i/\eta\leq b/\eta$. Then we have
\begin{align}
    \frac{V_t}{\bar{V}_t} & \leq \frac{f_0+f_1}{\bar{V}_t} + \max\left\{ \frac{\sum_{i\notin B_t}c_{0,t}^i\cdot S_t^i}{\sum_{i\notin B_t}\max\left\{ a,~c_{0,t}^i/\eta \right\}\cdot S_t^i},\ \frac{\sum_{i\in B_t}\left( c_{0,t}^i\cdot m_{0,t}^i + c_{1,t}^i\cdot m_{1,t}^i \right)}{\sum_{i\in B_t} \left( \eta a\cdot m_{0,t}^i + a\cdot m_{1,t}^i \right)} \right\} \notag\\
    & \leq \frac{f_0+f_1}{a\theta} + \max\left\{ \eta,\ \frac{b}{\eta a} \right\}=:R_1.
    \label{eq: ratio I}
\end{align}
For $i\notin B_t$, we have $c_{1,t}^i\geq\max\left\{ a,c_{0,t}^i/\eta \right\}$. For $i\in B_t$, we have $c_{0,t}^i\geq \eta \cdot c_{1,t}^i\geq \eta a$ and $m_{1,t}^i=\hat{m}_{1,t}^i$. Comparing with OPT, we have
\begin{align}
    V_t^* - \bar{V}_t & \geq \sum_{i=1}^n \left( c_{0,t}^i \cdot m_{0,t}^{i,*} + c_{1,t}^i \cdot m_{1,t}^{i,*} \right) - \sum_{i\notin B_t}\max\left\{ a,\frac{c_{0,t}^i}{\eta} \right\}\cdot S_t^i - \sum_{i\in B_t} \left( \eta a\cdot m_{0,t}^i + a\cdot m_{1,t}^i \right) \notag\\
    & \geq \sum_{i\notin B_t}\left( c_{0,t}^i \cdot m_{0,t}^{i,*} + c_{1,t}^i \cdot m_{1,t}^{i,*} - \max\left\{ a,\frac{c_{0,t}^i}{\eta} \right\}\cdot S_t^i \right) + \sum_{i\in B_t} \left( \eta a \cdot m_{0,t}^{i,*} + a \cdot m_{1,t}^{i,*} - \eta a\cdot m_{0,t}^i - a\cdot m_{1,t}^i \right) \notag\\
    & \geq \sum_{i\in B_t} \left( \eta a-a \right)\cdot \left( m_{1,t}^i-m_{1,t}^{i,*} \right) = \sum_{i\in B_t} \left( \eta a-a \right)\cdot \left( \hat{m}_{1,t}^i-m_{1,t}^{i,*} \right).
    \label{eq: difference I}
\end{align}

\underline{Case 2.} $t\in A$, OPT only uses FDC and \nsc{\singvaryingalgimprove} uses RDC to fulfill the order at time $t$. We denote these time periods as $A_1$. Note that
\[V_t\leq f_0+f_1+\sum_{i\notin B_t}c_{0,t}^i\cdot S_t^i + \sum_{i\in B_t} \left( c_{0,t}^i\cdot m_{0,t}^i + c_{1,t}^i\cdot m_{1,t}^i \right).\]
Define
\[\bar{V}_t=\sum_{i\notin B_t}\left( \eta a\cdot \hat{m}_{0,t}^i + c_{1,t}^i\cdot \hat{m}_{1,t}^i \right) + \sum_{i\in B_t} \left( \eta a\cdot m_{0,t}^i + a\cdot m_{1,t}^i \right).\]
Since \nsc{\singvaryingalgimprove} uses RDC to fulfill the order at time $t$, we have $\hat{m}_{0,t}^i>0$ for some $i$. For $i\notin B_t$, we have $c_{0,t}^i\leq \eta\cdot c_{1,t}^i$. For $i\in B_t$, we have $c_{1,t}^i\leq c_{0,t}^i/\eta\leq b/\eta$. Then we have
\begin{align}
    \frac{V_t}{\bar{V}_t} & \leq \frac{f_0+f_1}{\bar{V}_t} + \max\left\{ \frac{\sum_{i\notin B_t}c_{0,t}^i\cdot S_t^i}{\sum_{i\notin B_t}\left( \eta a\cdot \hat{m}_{0,t}^i + c_{1,t}^i\cdot \hat{m}_{1,t}^i \right)},\ \frac{\sum_{i\in B_t}\left( c_{0,t}^i\cdot m_{0,t}^i + c_{1,t}^i\cdot m_{1,t}^i \right)}{\sum_{i\in B_t} \left( \eta a\cdot m_{0,t}^i + a\cdot m_{1,t}^i \right)} \right\} \notag\\
    & \leq \frac{f_0+f_1}{\eta a} + \max\left\{ \eta,\ \frac{b}{\eta a} \right\}=:R_2.
    \label{eq: ratio II}
\end{align}
Since OPT only uses FDC, then $V_t^*=f_1+\sum_{i=1}^n c_{1,t}^i\cdot S_t^i$, $m_{1,t}^{i,*}=S_t^i$. For $i\in B_t$, we have $m_{1,t}^i=\hat{m}_{1,t}^i$. Comparing with OPT, we have
\begin{align}
    V_t^* - \bar{V}_t & \geq \sum_{i=1}^n c_{1,t}^i \cdot S_t^i - \sum_{i\notin B_t}\left( \eta a\cdot \hat{m}_{0,t}^i + c_{1,t}^i\cdot \hat{m}_{1,t}^i \right) - \sum_{i\in B_t} \left( \eta a\cdot m_{0,t}^i + a\cdot m_{1,t}^i \right) \notag\\
    & \geq \sum_{i\notin B_t}\left( c_{1,t}^i \cdot S_t^i - \eta a\cdot \hat{m}_{0,t}^i - c_{1,t}^i\cdot \hat{m}_{1,t}^i \right) + \sum_{i\in B_t} \left( a \cdot S_t^i - \eta a\cdot m_{0,t}^i - a\cdot m_{1,t}^i \right) \notag\\
    & \geq \sum_{i\notin B_t} \left( \eta a-a \right)\cdot \left( \hat{m}_{1,t}^i-S_t^i \right) + \sum_{i\in B_t} \left( \eta a-a \right)\cdot \left( m_{1,t}^i-m_{1,t}^{i,*} \right) \notag\\
    & = \sum_{i=1}^n \left( \eta a-a \right)\cdot \left( \hat{m}_{1,t}^i-m_{1,t}^{i,*} \right).
    \label{eq: difference II}
\end{align}

\underline{Case 3.} $t\in A$, both OPT and \nsc{\singvaryingalgimprove} only use FDC to fulfill the order at time $t$. We denote these time periods as $A_2$. Then we have $V_t=V_t^*$ and $m_{1,t}^i=\hat{m}_{1,t}^i=m_{1,t}^{i,*}$ for all $i$. Define $\bar{V}_t=V_t$.

\underline{Case 4.} $t\in A$ and OPT uses RDC to fulfill the order at time $t$. We denote these time periods as $A_3$. Define 
\[\bar{V}_t=\frac{1}{2}f_0 + a\cdot \sum_{i=1}^n S_t^i.\]
Since $t\in A$, we have $\sum_{i=1}^n S_t^i \leq \theta$. Then we have
\begin{align}
    \frac{V_t}{\bar{V}_t} \leq \frac{f_0+f_1+\sum_{i=1}^n c_{0,t}^i\cdot S_t^i}{f_0/2 + a\cdot \sum_{i=1}^n S_t^i} \leq \frac{f_0+f_1+b\theta}{f_0/2+a\theta}=:R_4.
    \label{eq: ratio IV}
\end{align}
Since OPT uses RDC and $t\in A$, then we have $V_t^*\geq f_0/2+\sum_{i=1}^n \left[ \left( c_{0,t}^i+f_0/(2\theta) \right)\cdot m_{0,t}^{i,*} + c_{1,t}^i\cdot m_{1,t}^{i,*} \right]$. Comparing with OPT, we have
\begin{align}
    V_t^* - \bar{V}_t & \geq \sum_{i=1}^n \left[ \left( c_{0,t}^i+\frac{f_0}{2\theta} \right)\cdot m_{0,t}^{i,*} + c_{1,t}^i\cdot m_{1,t}^{i,*} - a\cdot S_t^i \right] \notag\\
    & \geq \sum_{i=1}^n\left( c_{0,t}^i+\frac{f_0}{2\theta}-a \right)\cdot \left( S_t-m_{1,t}^{i,*} \right).
\end{align}
If we choose $\theta,\eta$ such that $f_0/(2\theta)\geq \eta a$, then we have 
\begin{equation}
    V_t^* - \bar{V}_t \geq \sum_{i=1}^n\left( \eta a - a \right)\cdot \left( S_t-m_{1,t}^{i,*} \right) \geq \sum_{i=1}^n\left( \eta a - a \right)\cdot \left( \hat{m}_{1,t}^i-m_{1,t}^{i,*} \right).
    \label{eq: difference IV}
\end{equation}

\medskip\noindent \underline{Step II.} We prove $\sum_{t=1}^T V_t^* \geq\sum_{t=1}^T \bar{V}_t$. For item $i$, define 
\[T_i=\left\{ t: \left( t\notin A \wedge i\in B_t \right) \vee \left( t\in A\right) \right\}\]
Combining Eqs.~\eqref{eq: difference I}, \eqref{eq: difference II}, \eqref{eq: difference IV}, we have
\begin{align*}
    \sum_{t=1}^T (V_t^* - \bar{V}_t) & \geq \sum_{t\notin A}\sum_{i\in B_t}(\eta a - a)\cdot (\hat{m}_{1,t}^i-m_{1,t}^{i,*}) + \sum_{t\in A_1\cup A_3}\sum_{i=1}^n (\eta a - a)\cdot (\hat{m}_{1,t}^i-m_{1,t}^{i,*}) \notag\\
    & = (\eta a - a)\cdot \sum_{i=1}^n\sum_{t\in T_i} (\hat{m}_{1,t}^i-m_{1,t}^{i,*}).
\end{align*}
Since $T_i$ contains the set of periods that \nsc{\singvaryingalgimprove} uses FDC to fulfill the order, by Lemma~\ref{lemma: key lemma of GPG policy K=1},we have
\[\sum_{t\in T_i} \hat{m}_{1,t}^i\geq \sum_{t\in T_i} m_{1,t}^{i,*}.\]
Therefore, we conclude that $\sum_{t=1}^T V_t^* \geq\sum_{t=1}^T \bar{V}_t$. 

\medskip\noindent \underline{Step III.} Now we analysis the competitive ratio. By definition, we have
\begin{align*}
    \CR(\nsc{\singvaryingalgimprove}) & = \sup_{\mathcal{I}} \frac{\sum_{t=1}^T V_t}{\sum_{t=1}^T V_t^*} \leq \sup_{\mathcal{I}} \frac{\sum_{t=1}^T V_t}{\sum_{t=1}^T \bar{V}_t} \leq \sup_{\mathcal{I}} \max_{t\in[T]} \frac{V_t}{\bar{V}_t}.
\end{align*}
Combining with the bounds in Eqs.~\eqref{eq: ratio I}, \eqref{eq: ratio II}, \eqref{eq: ratio IV}, we have
\begin{align}
    & \CR(\nsc{\singvaryingalgimprove}) \leq \max\{R_1, R_2, R_4\} \notag\\
    & \qquad = \max\left\{ \frac{f_0+f_1}{a\theta} + \max\left\{ \eta,\ \frac{b}{\eta a} \right\},\ \frac{f_0+f_1}{\eta a} + \max\left\{ \eta,\ \frac{b}{\eta a} \right\},\ \frac{f_0+f_1+b\theta}{f_0/2+a\theta} \right\},
    \label{eq: competitive ratio time varying single FDC improved}
\end{align}
with the constraint that $f_0/(2\theta)\geq \eta a$. Now we choose 
\[\eta=\sqrt{\frac{\max\left\{ f_0/2,\ b \right\}}{a}},\quad \theta=\frac{f_0}{2a\eta}=\sqrt{\frac{f_0}{a}\min\left\{ 1,\ \frac{f_0}{2b} \right\}}.\]
Note that we assume $f_0\geq f_1$. Then we have
\begin{align}
    R_1 & = \frac{f_0+f_1}{a\theta} + \max\left\{ \eta,\ \frac{b}{\eta a} \right\} \leq \frac{2f_0}{a\theta} + \max\left\{ \eta,\ \frac{b}{\eta a} \right\} = 5\eta,  \notag\\
    R_2 & = \frac{f_0+f_1}{\eta a} + \max\left\{ \eta,\ \frac{b}{\eta a} \right\} \leq \frac{2f_0}{\eta a} + \max\left\{ \eta,\ \frac{b}{\eta a} \right\} \leq 5\eta, \notag\\
    R_4 & = \frac{f_0+f_1+b\theta}{f_0/2+a\theta} \leq \frac{2f_0}{a\theta} + \frac{2b\theta}{f_0} \leq 4\eta + \sqrt{\frac{2b}{a}}. \notag
\end{align}
Therefore, we have
\begin{align*}
    & \CR(\nsc{\singvaryingalgimprove}) \leq \max\{R_1, R_2, R_4\} \\
    & \qquad \leq \max\left\{ 5\sqrt{\frac{\max\left\{ f_0/2,\ b \right\}}{a}},\ 4\sqrt{\frac{\max\left\{ f_0/2,\ b \right\}}{a}} + \sqrt{\frac{2b}{a}} \right\} \leq (4+\sqrt{2})\sqrt{\frac{\max\left\{ f_0/2,\ b \right\}}{a}}.
\end{align*}
This completes the proof.
\hfill\Halmos

\subsection{Proof of Corollary~\ref{cor:comp-ratio-upper-bound-better-of-two}}
\label{subsec:proof-cor-comp-ratio-upper-bound-better-of-two}
The first inequality in the corollary is straightforward. We only prove the second inequality. When $f_0 < f_1$, we have 
\begin{align}
\CR(\nsc{Better-of-Two})\leq 1 + \max\left\{\frac{f_0}{f_1}, \sqrt{\frac{b}{a}}\right\} \leq  1 + \sqrt{\frac{b}{a}} . \label{eq:proof-cor-comp-ratio-upper-bound-better-of-two-1}
\end{align}
When $f_0 \geq f_1$, we have 
\begin{align}
\CR(\nsc{Better-of-Two})&\leq \min\left\{1 + \max\left\{\frac{f_0}{f_1},~ \sqrt{\frac{b}{a}}\right\}, (4+\sqrt{2}) \max\left\{\sqrt{\frac{f_0}{2a}},~ \sqrt{\frac{b}{a}}\right\}\right\} \notag \\
&\leq   \min\left\{\max\left\{2 \cdot \frac{f_0}{f_1},~  (4+\sqrt{2})\sqrt{\frac{b}{a}}\right\}, (4+\sqrt{2}) \max\left\{\sqrt{\frac{f_0}{2a}},~ \sqrt{\frac{b}{a}}\right\}\right\} \notag\\
&= \max\left\{ \min\left\{2 \cdot \frac{f_0}{f_1},~(4+\sqrt{2})\sqrt{\frac{f_0}{2a}}\right\}, ~(4+\sqrt{2})\sqrt{\frac{b}{a}}\right\} . \label{eq:proof-cor-comp-ratio-upper-bound-better-of-two-2}
\end{align}
Note that, by  Eq.~\eqref{eq:proof-cor-comp-ratio-upper-bound-better-of-two-1}, the competitive ratio in the case $f_0 < f_1$ also satisfies the upper bound in Eq.~\eqref{eq:proof-cor-comp-ratio-upper-bound-better-of-two-2}. Therefore, Eq.~\eqref{eq:proof-cor-comp-ratio-upper-bound-better-of-two-2} serves as a global upper bound of the competitive ratio of the algorithm, and we prove the second inequality in the corollary. \hfill\Halmos

\subsection{Proof of Theorem~\ref{thm: lower bound with bounded cost k=1}}
\label{subsec:proof-lower-bound-bounded-cost-single-fdc}
In the proof of Theorem~\ref{thm: lower bound with bounded cost}, the following lower bound does not require the number of FDCs $K$ to be at least 2:
\[\CR(\mathrm{ALG}) \geq \max\left\{1,\ \frac{1}{4}\max_{n\geq 2} \min\left\{ n,~ \frac{f_0}{f_1+na}\right\} \right\}.\]
So we only need to prove $\CR(\mathrm{ALG}) \geq \frac{1}{3}\sqrt{\frac{b}{a}}$. We also construct two instances to prove the lower bound. In both instances, there is only one type of items $n=1$, and the initial inventories of FDC 1 is $N$, where $N$ is a large integer that can be chosen arbitrarily. For the first instance $I_1$, the time period is set to be $T=1$. At time period $t=1$, the customer order requests $N$ units of the item, and the variable costs are $c_{0,1}^1=\sqrt{ab},\ c_{1,1}^1=a$. For the second instance $I_2$, the time period is set to be $T=2$. At time period $t=1$, the customer order and the variable costs are the same as instance $I_1$. At time period $t=2$, the customer order requests $N$ units of the item, and the variable costs are $c_{0,2}^1=b,\ c_{1,2}^1=a$.

Let $m_0,m_1$ be the number of items fulfilled from RDC and FDC 1 at time period $t=1$ by $\mathrm{ALG}$, respectively. Here $m_0,m_1$ may be random variables. Then the cost incurred at time $t=1$ is at least $\sqrt{ab}\cdot m_0+a\cdot m_1$. For instance $I_2$, the remaining inventory of FDC 1 at time $t=2$ is $N-m_1$, and the cost incurred at time $t=2$ is at least $b\cdot m_1$. Therefore, we have
\[\mathrm{ALG}(I_1) \geq \sqrt{ab}\cdot \mathbb{E}[m_0],\quad \mathrm{ALG}(I_2) \geq b\cdot \mathbb{E}[m_1]=b\cdot(N-\mathbb{E}[m_0]).\] 

For instance $I_1$, the cost of the optimal offline policy is $\mathrm{OPT}(I_1)=f_1+aN$. For instance $I_2$, the optimal offline policy is to use RDC to fulfill the order at time $t=1$ and use FDC 1 to fulfill the order at time $t=2$, and the cost incurred by the optimal offline policy is $\mathrm{OPT}(I_2)\leq f_0+f_1+\sqrt{ab}N+aN$. Therefore, we have
\begin{align*}
    \CR(\mathrm{ALG}) & \geq \sup_N\max\left\{ \frac{\mathrm{ALG}(I_1)}{\mathrm{OPT}(I_1)},\ \frac{\mathrm{ALG}(I_2)}{\mathrm{OPT}(I_2)} \right\} \\
    & = \sup_N\max\left\{ \frac{\sqrt{ab}\cdot \mathbb{E}[m_0]}{f_1+aN},\ \frac{b\cdot (N-\mathbb{E}[m_0])}{f_0+f_1+\sqrt{ab}N+aN} \right\} \\
    & \geq \sup_N \min_{0\leq x\leq N} \max\left\{ \frac{\sqrt{ab}\cdot x}{f_1+aN},\ \frac{b\cdot (N-x)}{f_0+f_1+2\sqrt{ab}N} \right\} \\
    & = \sup_N \frac{\sqrt{ab}\cdot bN}{\sqrt{ab}(f_0+f_1)+bf_1+3abN}   \geq \frac{1}{3}\sqrt{\frac{b}{a}}.
\end{align*}
We complete the proof of the lower bound. \hfill\Halmos

\subsection{Proof of Proposition~\ref{prop:ub-lb-ratio-single-varying}}
\label{subsec: proof-of-ub-lb-ratio-single-varying}
We first prove that for any $t\in(0,1)$, we have 
\begin{equation}
    \max_{n}\min\left\{ n,\ \frac{f_0}{f_1+na}\right\}\geq \min\left\{ t\cdot \sqrt{\frac{f_0}{a}},\ (1-t^2)\cdot \frac{f_0}{f_1} \right\}.
    \label{eq:proof-of-ub-lb-ratio-single-varying-1}
\end{equation}
Note that $\max_{n}\min\left\{ n,\ \frac{f_0}{f_1+na}\right\}=\sqrt{\frac{f_0}{a}+\frac{f_1^2}{4a^2}}-\frac{f_1}{2a}$, so we only need to prove 
\begin{equation}
    \sqrt{\frac{f_0}{a}+\frac{f_1^2}{4a^2}}-\frac{f_1}{2a}\geq \min\left\{ t\cdot \sqrt{\frac{f_0}{a}},\ (1-t^2)\cdot \frac{f_0}{f_1} \right\}.
    \label{eq:proof-of-ub-lb-ratio-single-varying-2}
\end{equation}
Let $r=\frac{f_0}{f_1}\cdot \sqrt{\frac{a}{f_0}}$. Then Eq. \eqref{eq:proof-of-ub-lb-ratio-single-varying-2} is equivalent to
\begin{equation*}
    \sqrt{1+\frac{r^2}{4}} - \frac{r}{2} \geq \min\left\{t,\ \frac{1-t^2}{r}\right\}.
\end{equation*}
If $t \leq (1-t^2)/r$, i.e. $r\leq (1-t^2)/t$, then we have
\[\sqrt{1+\frac{r^2}{4}} - \frac{r}{2} \geq\sqrt{1+\frac{(1-t^2)^2}{4t^2}} - \frac{1-t^2}{2t}=t.\]
If $t > (1-t^2)/r$, i.e. $r > (1-t^2)/t$, then we have
\[\sqrt{1+\frac{r^2}{4}} - \frac{r}{2} \geq\frac{1-t^2}{r} \quad \Longleftrightarrow \quad r^2\geq \frac{(1-t^2)^2}{t^2}.\]
Combining the above two cases, we prove Eq.~\eqref{eq:proof-of-ub-lb-ratio-single-varying-1}.

Now we compare the upper and lower bounds. Combining Eq.~\eqref{eq:proof-of-ub-lb-ratio-single-varying-1} with the lower bound in Theorem~\ref{thm: lower bound with bounded cost k=1}, we have
\[\CR(\mathrm{ALG}) \geq \max\left\{1,\ \frac{1}{3}\sqrt{\frac{b}{a}},\ \frac{1}{4}\min\left\{ t\cdot \sqrt{\frac{f_0}{a}},\ (1-t^2)\cdot \frac{f_0}{f_1} \right\} \right\}.\]
Therefore, the ratio between the upper and lower bounds is at most
\begin{equation*}
    \frac{\max\left\{\min\left\{2 \cdot \frac{f_0}{f_1},~(2\sqrt{2}+1)\sqrt{\frac{f_0}{a}}\right\}, ~(4+\sqrt{2})\sqrt{\frac{b}{a}}\right\}}{\max\left\{1,\ \frac{1}{3}\sqrt{\frac{b}{a}},\ \frac{1}{4}\min\left\{ t\cdot \sqrt{\frac{f_0}{a}},\ (1-t^2)\cdot \frac{f_0}{f_1} \right\} \right\}}\leq \max\left\{ 3(4+\sqrt{2}),~\frac{8\sqrt{2}+4}{t},~\frac{8}{1-t^2} \right\}.
\end{equation*}
We choose $t=(\sqrt{10+4\sqrt{2}}-1)/(2\sqrt{2}+1)\approx 0.772$. Then we have
\[\max\left\{ 3(4+\sqrt{2}),~\frac{8\sqrt{2}+4}{t},~\frac{8}{1-t^2} \right\}\leq \frac{4\cdot (9+4\sqrt{2})}{\sqrt{10+4\sqrt{2}}-1}\leq 19.828.\]
We complete the proof of the proposition. \hfill\Halmos

\subsection{Proof of Claim~\ref{claim:proof-rcc-vp-barV-time-invariant-single-fdc}}
\label{subsec:proof-claim-proof-rcc-vp-barV-time-invariant-single-fdc}

\noindent \underline{The first inequality in Eq.~\eqref{eq:total-cost-upper-bound-by-intermediate-cost-time-invariant-single-fdc}.} By the definition of \nsc{\singinvaralg}, if $\theta_t=1$, then the algorithm uses only the RDC. For $i\in I$, we have $c_1^i<c_0^i$. Thus we have
\begin{align}
    \bar{V_t}-V_t&=f_0 + \sum_{i\in I}(c_1^i-c_0^i)\cdot \left( m_{1,t}^{i,*}-\hat{m}_{1,t}^i\right)^++\sum_{i=1}^n c_0^i\cdot S_t^i-\left( f_0+\sum_{i=1}^n c_0^i \cdot S_t^i \right) \notag\\
    &=\sum_{i\in I}(c_1^i-c_0^i)\cdot \left( m_{1,t}^{i,*}-\hat{m}_{1,t}^i\right)^+ \leq \sum_{i\in I}(c_1^i-c_0^i)\cdot \left( m_{1,t}^{i,*}-\hat{m}_{1,t}^i\right).
    \label{eq:total-cost-upper-bound-by-intermediate-cost-time-invariant-single-fdc-1}
\end{align}
If $\theta_t=0$, then the algorithm uses the greedy allocation, and $m_{1,t}^i=0$ for $i\notin I$. we have
\begin{equation}
    \bar{V_t}-V_t=\sum_{i\in I}(c_1^i-c_0^i)\cdot m_{1,t}^{i,*}+\sum_{i=1}^n c_0^i\cdot S_t^i-\left( \sum_{i=1}^n c_0^i \cdot m_{0,t}^i + \sum_{i=1}^n c_1^i \cdot m_{1,t}^i \right)=\sum_{i\in I}(c_1^i-c_0^i)\cdot \left( m_{1,t}^{i,*}-m_{1,t}^i \right).
    \label{eq:total-cost-upper-bound-by-intermediate-cost-time-invariant-single-fdc-2}
\end{equation}
Combining Eqs.~\eqref{eq:total-cost-upper-bound-by-intermediate-cost-time-invariant-single-fdc-1} and \eqref{eq:total-cost-upper-bound-by-intermediate-cost-time-invariant-single-fdc-2}, we have
\begin{equation}
    \sum_{t=1}^T(\bar{V_t}-V_t)=\sum_{i\in I}(c_1^i-c_0^i)\cdot \sum_{t=1}^T(m_{1,t}^{i,*}-\hat{m}_{1,t}^i).
    \label{eq:total-cost-upper-bound-by-intermediate-cost-time-invariant-single-fdc-3}
\end{equation}
By Lemma~\ref{lemma: key lemma of GPG policy}, we have $\sum_{t=1}^T(m_{1,t}^{i,*}-\hat{m}_{1,t}^i)\leq 0$ for all $i\in I$. Combining with Eq.~\eqref{eq:total-cost-upper-bound-by-intermediate-cost-time-invariant-single-fdc-3}, we establish the first inequality in Eq.~\eqref{eq:total-cost-upper-bound-by-intermediate-cost-time-invariant-single-fdc}.

\medskip\noindent \underline{The second inequality in Eq.~\eqref{eq:total-cost-upper-bound-by-intermediate-cost-time-invariant-single-fdc}.} 
We first give an upper bound of $\mathbb{E}[\bar{V_t}|\mathcal{F}_{t-1}]$ for any $t\in[T]$. Let $q_t=\sum_{i=1}^n(c_0^i-c_1^i)\cdot \hat{m}_{1,t}^i=\sum_{i\in I}(c_0^i-c_1^i)\cdot \hat{m}_{1,t}^i$. We consider the following two cases. 

\begin{itemize}
    \item \underline{Case 1.} If $\sum_{i=1}^n m_{1,t}^{i,*}=0$, then cost of optimal fulfillment policy at time period $t$ is $V_t^*=f_0+\sum_{i=1}^n c_0^i \cdot S_t^i$. By the definition of $\bar{V_t}$, we have
    \begin{equation}
        \frac{\mathbb{E}[\bar{V_t}|\mathcal{F}_{t-1}]}{V_t^*}\leq\frac{p(q_t)\cdot f_0 + (1-p(q_t))\cdot (f_0+f_1) +\sum_{i=1}^n c_0^i\cdot S_t^i}{f_0+\sum_{i=1}^n c_0^i \cdot S_t^i}\leq \frac{f_0+(1-p(q_t))\cdot f_1+q_t}{f_0+q_t}.
        \label{eq:intermediate-cost-upper-bound-time-invariant-single-fdc-1}
    \end{equation}

    \item \underline{Case 2.} If $\sum_{i=1}^n m_{1,t}^{i,*}>0$, then the value of the optimal policy is at least
    \[V_t^*\geq f_1+\sum_{i=1}^n \left(c_0^i\cdot m_{0,t}^{i,*}+c_1^i\cdot m_{1,t}^{i,*}\right).\]
    We consider the value of $\theta_t$. 
    \begin{itemize}
        \item \underline{Case 2a.} If $\theta_t=1$, then by the definition of $\bar{V_t}$, we have 
        \begin{align}
            & \bar{V_t} \leq f_0 + \sum_{i\in I}(c_1^i-c_0^i)\cdot \left( m_{1,t}^{i,*}-\hat{m}_{1,t}^i\right)+\sum_{i=1}^n c_0^i\cdot S_t^i \notag\\
            &\qquad = f_0 + \sum_{i\in I}(c_1^i-c_0^i)\cdot \left( m_{1,t}^{i,*}-\hat{m}_{1,t}^i\right)+\sum_{i\in I}c_0^i\cdot(m_{0,t}^{i,*}+m_{1,t}^{i,*})+\sum_{i\notin I} c_0^i\cdot (m_{0,t}^{i,*}+m_{1,t}^{i,*}) \notag \\
            &\qquad \leq f_0 + \sum_{i=1}^n \left(c_0^i\cdot m_{0,t}^{i,*}+c_1^i\cdot m_{1,t}^{i,*}\right) +q_t. \notag
        \end{align}
        Here the first inequality holds since $c_1^i<c_0^i$ for all $i\in I$; the second inequality holds since $c_1^i\geq c_0^i$ for all $i\notin I$. Thus we have
        \begin{equation}
            \frac{\bar{V_t}}{V_t^*}\leq\frac{f_0 + \sum_{i=1}^n \left(c_0^i\cdot m_{0,t}^{i,*}+c_1^i\cdot m_{1,t}^{i,*}\right) +q_t}{f_1+\sum_{i=1}^n \left(c_0^i\cdot m_{0,t}^{i,*}+c_1^i\cdot m_{1,t}^{i,*}\right)}\leq \max\left\{\frac{f_0+q_t}{f_1},1\right\}.
            \label{eq:intermediate-cost-upper-bound-time-invariant-single-fdc-2}
        \end{equation}

        \item \underline{Case 2b.} If $\theta_t=0$, then by the definition of $\bar{V_t}$, we have
        \begin{align*}
            \bar{V_t}&= f_0\cdot \mathbb{I}\left( \sum_{i=1}^n m_{0,t}^i>0 \right) + f_1\cdot \mathbb{I}\left( \sum_{i=1}^n m_{1,t}^i>0 \right) + \sum_{i\in I}(c_1^i-c_0^i)\cdot m_{1,t}^{i,*}+\sum_{i=1}^n c_0^i\cdot S_t^i\\
            &\leq f_0 + f_1 + \sum_{i\in I}(c_1^i-c_0^i)\cdot m_{1,t}^{i,*}+\sum_{i=1}^n c_0^i\cdot S_t^i \leq f_0 + f_1 + \sum_{i=1}^n \left(c_0^i\cdot m_{0,t}^{i,*}+c_1^i\cdot m_{1,t}^{i,*}\right).
        \end{align*}
        Thus we have
        \begin{equation}
            \frac{\bar{V_t}}{V_t^*}\leq\frac{f_0 + f_1 + \sum_{i=1}^n \left(c_0^i\cdot m_{0,t}^{i,*}+c_1^i\cdot m_{1,t}^{i,*}\right)}{f_1+\sum_{i=1}^n \left(c_0^i\cdot m_{0,t}^{i,*}+c_1^i\cdot m_{1,t}^{i,*}\right)}\leq \frac{f_0+f_1}{f_1}.
            \label{eq:intermediate-cost-upper-bound-time-invariant-single-fdc-3}
        \end{equation}
    \end{itemize}
\end{itemize}
Combining Eq.~\eqref{eq:intermediate-cost-upper-bound-time-invariant-single-fdc-1}, Eq.~\eqref{eq:intermediate-cost-upper-bound-time-invariant-single-fdc-2} and Eq.~\eqref{eq:intermediate-cost-upper-bound-time-invariant-single-fdc-3}, we have
\[\frac{\mathbb{E}[\bar{V_t}|\mathcal{F}_{t-1}]}{V_t^*}\leq \max\left\{ \frac{f_0+f_1}{f_1},~~ p(q_t)\cdot \frac{f_0+q_t}{f_1}+(1-p(q_t))\cdot \frac{f_0+f_1}{f_1},~~ \frac{f_0+(1-p(q_t))\cdot f_1+q_t}{f_0+q_t}  \right\}.\]
Now we define 
\[w=\frac{f_0}{f_1},~~ r=\frac{q_t}{f_1}.\]
Then we have
\[\frac{\mathbb{E}[\bar{V_t}|\mathcal{F}_{t-1}]}{V_t^*}\leq\max\left\{ 1+w,~~1+w+p(q_t)\cdot (r-1),~~1+\frac{1}{w+r}-p(q_t)\cdot \frac{1}{w+r} \right\}.\]
We consider the following three cases based on the definition of the probability function $p(\cdot)$:
\begin{itemize}
    \item If $r\leq 1$, i.e. $q_t\leq f_1$, we have $p(q_t)=1$. Then we have 
    \begin{equation}
        \frac{\mathbb{E}[\bar{V_t}|\mathcal{F}_{t-1}]}{V_t^*}\leq 1+w.
        \label{eq:intermediate-cost-upper-bound-time-invariant-single-fdc-4}
    \end{equation}
    \item If $r>\max\left\{ 1,\frac{1}{w}-w \right\}$, i.e. $q_t>\max\left\{ f_1, \frac{f_1^2}{f_0}-f_0 \right\}$, we have $p(q_t)=0$. Then we have
    \begin{equation}
        \frac{\mathbb{E}[\bar{V_t}|\mathcal{F}_{t-1}]}{V_t^*}\leq \max\left\{ 1+w,~~1+\frac{1}{w+r} \right\}\leq 1+w.
        \label{eq:intermediate-cost-upper-bound-time-invariant-single-fdc-5}
    \end{equation}
    \item If $1<r\leq \frac{1}{w}-w$, i.e. $f_1<q_t\leq \frac{f_1^2}{f_0}-f_0$, we have \[p(q_t)=\frac{\frac{1}{w+r}-w}{\frac{1}{w+r}+r-1}=\frac{f_1^2-(f_0+q_t)f_0}{f_1^2+(f_0+q_t)(q_t-f_1)}.\] 
    Then we have
    \begin{equation}
        \frac{\mathbb{E}[\bar{V_t}|\mathcal{F}_{t-1}]}{V_t^*}\leq 1+\frac{w+r-1}{1+wr+r^2-w-r}\leq 1+\frac{1}{1-w+2\sqrt{1-w}}.
        \label{eq:intermediate-cost-upper-bound-time-invariant-single-fdc-6}
    \end{equation}
    
\end{itemize}
Combining Eq.~\eqref{eq:intermediate-cost-upper-bound-time-invariant-single-fdc-4}, Eq.~\eqref{eq:intermediate-cost-upper-bound-time-invariant-single-fdc-5} and Eq.~\eqref{eq:intermediate-cost-upper-bound-time-invariant-single-fdc-6}, we have
\[\frac{\mathbb{E}[\bar{V_t}|\mathcal{F}_{t-1}]}{V_t^*}\leq \begin{cases}
1+\frac{1}{1-w+2\sqrt{1-w}}, &\text{if } w<\frac{\sqrt{5}-1}{2}\\
1+w, &\text{if } w\geq \frac{\sqrt{5}-1}{2}
\end{cases},\quad\forall\,t\in[T].\]
This proves the claim. \hfill\Halmos

\subsection{Proof of Theorem~\ref{thm: lower bound with time invariant cost k=1}}
\label{subsec:proof-lowerbound-time-invariant-single-FDC}
Note that the proof of Theorem~\ref{thm: lower bound with time invariant cost} also holds when there is only one FDC. So we only need to prove that $\CR_\mathrm{inv}(\mathrm{ALG}) \geq \frac{5}{4}$. Without loss of generality, we assume $\frac{5}{4}\geq \frac{f_0+f_1}{f_1}$, i.e. $f_1\geq 4f_0$.

Given the fixed costs of FDC and RDC $(f_0,f_1)$, we consider the following two instances. In both instances, there is only one type of items $n=1$. The initial inventory of FDC is $MN$, and the variable costs are set to be $c_0^1=\epsilon, c_1^1=0$, where $M$, $N$ and $\epsilon$ are determined later. For the first instance $I_1$, the time period is set to be $T=1$. At time period $t=1$, the customer order requests $N$ units of the item. For the second instance $I_2$, the time period is set to be $T=2$. At time period $t=1$, the customer order is the same as instance $I_1$, and at time period $t=2$, the customer order requests $MN$ units of the item. 

By Yao's principle, the lower bound on the competitive ratio of any randomized online fulfillment policy is lower bounded by that of any deterministic online fulfillment policy under a worst-case distribution over the above two instances, i.e. 
\[\CR_\mathrm{inv}(\mathrm{ALG}) \geq \max_{p\in[0,1]}\min_{\text{deterministic } \mathrm{ALG}}\left\{ p\cdot \frac{\mathrm{ALG}(I_1)}{\mathrm{OPT}(I_1)}+(1-p)\cdot\frac{\mathrm{ALG}(I_2)}{\mathrm{OPT}(I_2)} \right\}.\]

Let $\mathrm{ALG}$ be any deterministic online fulfillment policy. We consider the following two cases based on ALG's decision at time period $t=1$.

\medskip
\noindent \underline{Case 1.} Suppose that at time period $t=1$, $\mathrm{ALG}$ uses only the RDC to fulfill the order. For instance $I_1$, the optimal policy is to use the FDC to fulfill the order, and the cost incurred by the optimal offline policy is $\mathrm{OPT}(I_1)=f_1$. Since $\mathrm{ALG}$ uses only the RDC to fulfill the order, its incurred cost is at least $\mathrm{ALG}(I_1)\geq f_0+N\epsilon$. Then we have 
\[\frac{\mathrm{ALG}(I_1)}{\mathrm{OPT}(I_1)}\geq \frac{f_0+N\epsilon}{f_1},~~\text{and}~~~ p\cdot \frac{\mathrm{ALG}(I_1)}{\mathrm{OPT}(I_1)}+(1-p)\cdot\frac{\mathrm{ALG}(I_2)}{\mathrm{OPT}(I_2)}\geq p\cdot \frac{f_0+N\epsilon}{f_1}+(1-p).\]

\medskip
\noindent \underline{Case 2.} Suppose at time period $t=1$, FDC is used by $\mathrm{ALG}$ to fulfill the order. For instance $I_2$, the optimal fulfillment policy is to use RDC solely to fulfill the order at time $t=1$ and use FDC solely to fulfill the order at time $t=2$, and the cost incurred by the optimal offline policy is $\mathrm{OPT}(I_2)=f_0+f_1+N\epsilon$. Since FDC is used by $\mathrm{ALG}$ to fulfill the order at time $t=1$, the remaining inventory of FDC at time $t=2$ is strictly less than $MN$, so RDC must be used to fulfill part of the order at time $t=2$. 

If $\mathrm{ALG}$ uses only the RDC to fulfill the order at time $t=2$, then the cost incurred by $\mathrm{ALG}$ is at least $\mathrm{ALG}(I_2)\geq f_0+f_1+MN\epsilon$. Otherwise, $\mathrm{ALG}$ uses both the RDC and the FDC to fulfill the order at time $t=2$, in which case the cost is at least $\mathrm{ALG}(I_2)\geq f_1+(f_0+f_1+N\epsilon)=f_0+2f_1+N\epsilon$. Thus, we have $\mathrm{ALG}(I_2)\geq \min\{f_0+f_1+MN\epsilon,\ f_0+2f_1+N\epsilon\}$. Therefore, we have
\[\frac{\mathrm{ALG}(I_2)}{\mathrm{OPT}(I_2)} \geq \frac{\min\{f_0+f_1+MN\epsilon,\ f_0+2f_1+N\epsilon\}}{f_0+f_1+N\epsilon},\]
and 
\[p\cdot \frac{\mathrm{ALG}(I_1)}{\mathrm{OPT}(I_1)}+(1-p)\cdot\frac{\mathrm{ALG}(I_2)}{\mathrm{OPT}(I_2)}\geq p+(1-p)\cdot \frac{\min\{f_0+f_1+MN\epsilon,\ f_0+2f_1+N\epsilon\}}{f_0+f_1+N\epsilon}.\]
If we choose $M$ such that $f_0+f_1+MN\epsilon\geq f_0+2f_1+N\epsilon$, i.e. $M\geq 1+\frac{f_1}{N\epsilon}$, then we have
\[\frac{\mathrm{ALG}(I_2)}{\mathrm{OPT}(I_2)} \geq \frac{f_0+2f_1+N\epsilon}{f_0+f_1+N\epsilon},~~\text{and}~~~ p\cdot \frac{\mathrm{ALG}(I_1)}{\mathrm{OPT}(I_1)}+(1-p)\cdot\frac{\mathrm{ALG}(I_2)}{\mathrm{OPT}(I_2)}\geq p+(1-p)\cdot \frac{f_0+2f_1+N\epsilon}{f_0+f_1+N\epsilon}.\]

Combining the above two cases, let $\lambda=\frac{f_0+N\epsilon}{f_1}$, we have
\begin{align*}
    &\max_{p\in[0,1]}\min_{\text{deterministic } \mathrm{ALG}}\left\{ p\cdot \frac{\mathrm{ALG}(I_1)}{\mathrm{OPT}(I_1)}+(1-p)\cdot\frac{\mathrm{ALG}(I_2)}{\mathrm{OPT}(I_2)} \right\}\\
    \geq\ &\max_{p\in[0,1]}\min\left\{ p\cdot \frac{f_0+N\epsilon}{f_1}+(1-p),~~ p+(1-p)\cdot \frac{f_0+2f_1+N\epsilon}{f_0+f_1+N\epsilon} \right\}\\
    =\ &\max_{p}\min\left\{ p\lambda+(1-p),~~ p+(1-p)\cdot \frac{\lambda+2}{\lambda+1} \right\}\\
    =\ &\max_{p\in[0,1]}\min\left\{ 1+(\lambda-1)p,~~1+\frac{1}{\lambda+1}-\frac{1}{\lambda+1}p \right\}\\
    =\ &
    \begin{cases}
        1+\frac{1}{\lambda}-\frac{1}{\lambda^2},& \lambda\geq1\\
        1,& \lambda<1
    \end{cases}.
\end{align*}
Since $f_1\geq 4f_0$, we can choose $N$ and $\epsilon$ such that $\lambda=2$, i.e., $N\epsilon=2f_1-f_0$. Then we have
\[\CR_\mathrm{inv}(\mathrm{ALG}) \geq 1+\frac{1}{2}-\frac{1}{4}=\frac{5}{4}.\]
This completes the proof.

\section{Details of Numerical Experiments}
\label{sec: additional numerical experiments}

This section complements the theoretical results with numerical experiments. Section~\ref{subsec: experiment settings} describes the experimental design. Section~\ref{subsec: time-varying experiments} presents results for the time-varying variable cost setting. Section~\ref{subsec: time-invariant experiments} presents results for the time-invariant variable cost setting and includes comparisons with forecast-based LP policies. Finally, Section~\ref{subsec: additional extreme cases} reports a set of stress-test instances showing that, under certain time-varying demand patterns, the proposed policy can substantially outperform the myopic baseline. Furthermore, section~\ref{subsec: baseline algorithms} describes the baseline algorithms, including the myopic policy and the LP-based policies. 

\subsection{Experiment Settings}
\label{subsec: experiment settings}
Unless otherwise specified, the numerical experiments use the following baseline parameter setting.

\medskip
\noindent \underline{Basic Network Setting.} We set the number of items to $n=50$, the number of FDCs to $K=10$, and the baseline time horizon to $T=2000$.

\medskip
\noindent \underline{Cost Structure.} The fixed cost of each FDC is $f_k=5$ for all $k=1,\dots,K$, and the fixed cost of the RDC is $f_0=50$. The lower and upper bounds of the variable costs are $a=8$ and $b=30$. In the time-varying setting, for each period $t$, item $i$, and DC $k$, the variable cost $c_{k,t}^i$ is independently drawn from the uniform distribution on $[a,b]$. In the time-invariant setting, we first independently draw $c_k^i$ from the uniform distribution on $[a,b]$ for each item $i$ and DC $k$, and then set $c_{k,t}^i=c_k^i$ for all periods $t$.

\medskip
\noindent \underline{Order Distribution.} We specify the order distribution by the order size, the number of order types for each size, and the arrival probability of each order size. The possible order sizes are $1,2,3,10,15,20$. The corresponding numbers of order types are $50,50,30,20,20,10$, and the corresponding arrival probabilities are $0.4,0.2,0.1,0.1,0.1,0.1$. Conditional on an order size, all order types of that size are equally likely. For comparability with the baseline algorithms, we follow the standard binary-demand convention in these experiments: each item quantity in an order is either 0 or 1. At each period $t$, the customer order $\boldsymbol{S}_t$ is independently drawn from this distribution.

\medskip
\noindent\underline{Inventory Level.} The initial inventory of item $i$ at FDC $k$ is $I_{k,0}^i=\tau\cdot p_iT/K$, where $p_iT$ denotes the expected demand for item $i$ over the selling horizon and $\tau$ controls the inventory scale. We set $\tau=0.2$, so FDC inventory is scarce relative to expected demand, while the RDC remains an unlimited backup.

\medskip
This setting is stylized, but it is chosen to reflect several practical features of online fulfillment. The fixed cost of RDC is higher than that of FDCs, reflecting the high activation costs of the central distribution center. The order distribution is dominated by small orders, while still allowing occasional larger multi-item orders, consistent with the fact that most e-commerce orders are small but large orders do occur. Finally, the FDC inventory level is deliberately limited, which makes the intertemporal value of local inventory nontrivial. Thus, the experiments evaluate the policies in a regime where real-time decisions must balance current fulfillment cost against future inventory availability.

In all experiments, we report average fulfillment costs and running time over $100$ independent replications. When LP-based policies are included, they are given the true demand distribution used to generate orders, and the time of solving LP is also contained in the reported running time.

\subsection{Time-varying Variable Costs}
\label{subsec: time-varying experiments}

We first study the setting with time-varying variable costs. We compare \nsc{\multivaryingalg} with the generalized myopic policy in Algorithm~\ref{alg: generalized myopic policy} for multi-FDC networks, and we also report a single-FDC experiment for the corresponding refined policy.

For the multi-FDC case, we first vary the time horizon. Specifically, we set $T=200\cdot i$ for $i=1,2,\dots,10$, while keeping the remaining parameters at their baseline values. Figure~\ref{fig:multi-fdc-time-varying-horizon} reports the average total fulfillment costs and total running times. The results show that \nsc{\multivaryingalg} achieves fulfillment costs that are comparable to, and in several cases lower than, those of the myopic policy. At the same time, its running time is much smaller. This confirms that the proposed policy can preserve the cost performance of a more expensive myopic benchmark while providing a substantial computational advantage.

\begin{figure}[htbp]
    \centering
    \begin{subfigure}{0.49\textwidth}
        \centering
        \includegraphics[width=\linewidth]{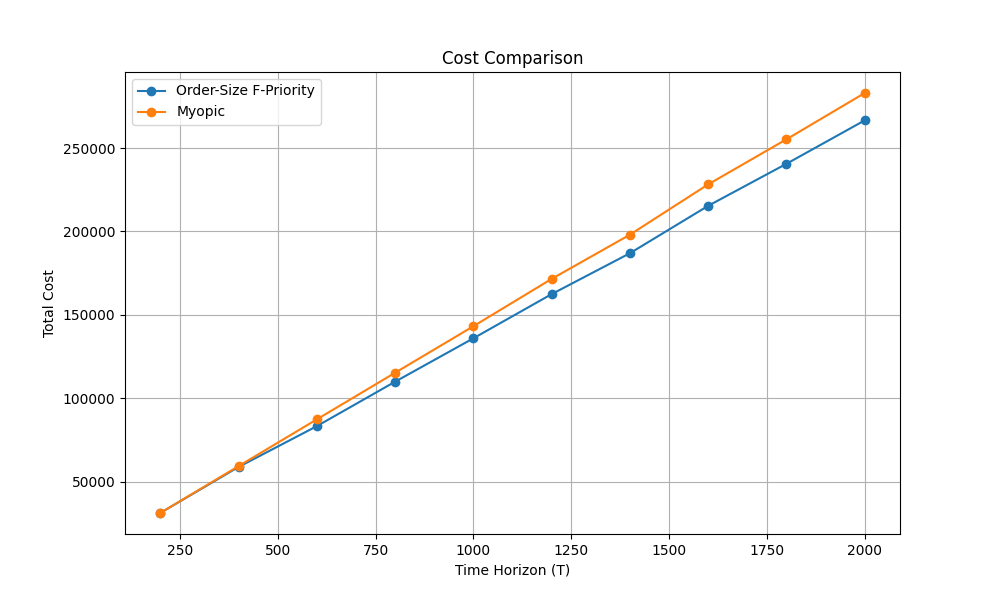}
    \end{subfigure}
    \begin{subfigure}{0.49\textwidth}
        \centering
        \includegraphics[width=\linewidth]{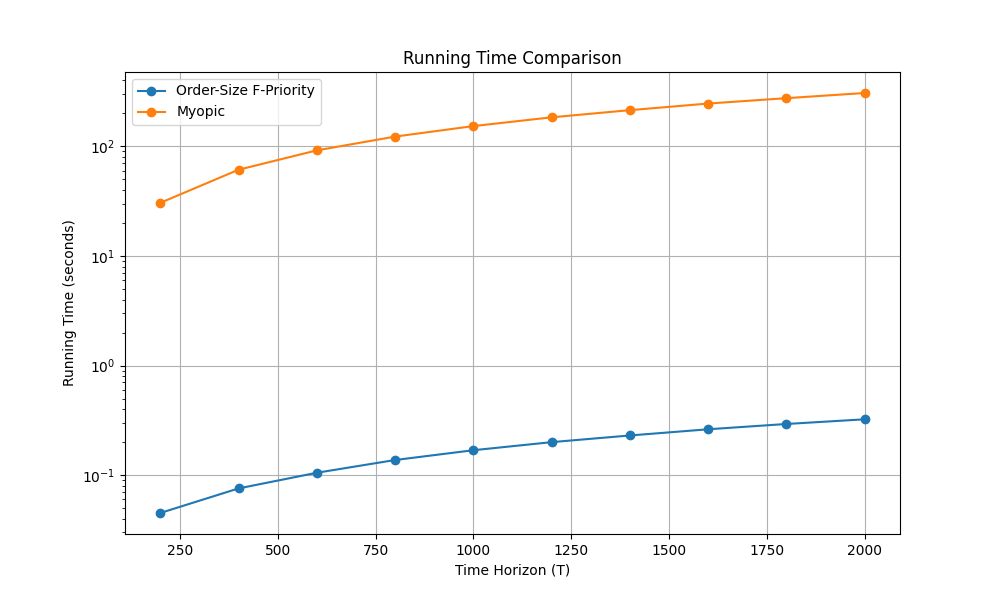}
    \end{subfigure}
    
    \caption{The cost and the total running time of \nsc{\multivaryingalg} and the myopic policy as the time horizon $T$ increases.}
    \label{fig:multi-fdc-time-varying-horizon}
\end{figure}

We next vary the number of FDCs by setting $K=2i+1$ for $i=1,2,\dots,7$. Figure~\ref{fig:multi-fdc-time-varying-k} shows that the running time of the myopic policy grows exponentially as $K$ increases, consistent with the NP-hardness result in Proposition~\ref{prop:np-hardness-of-myopic}. In contrast, the running time of \nsc{\multivaryingalg} remains nearly flat as the network expands. Importantly, this computational gain does not come at a large cost penalty: \nsc{\multivaryingalg} continues to deliver cost performance close to, and sometimes better than, that of the myopic policy.

\begin{figure}[htbp]
    \centering
    \begin{subfigure}{0.49\textwidth}
        \centering
        \includegraphics[width=\linewidth]{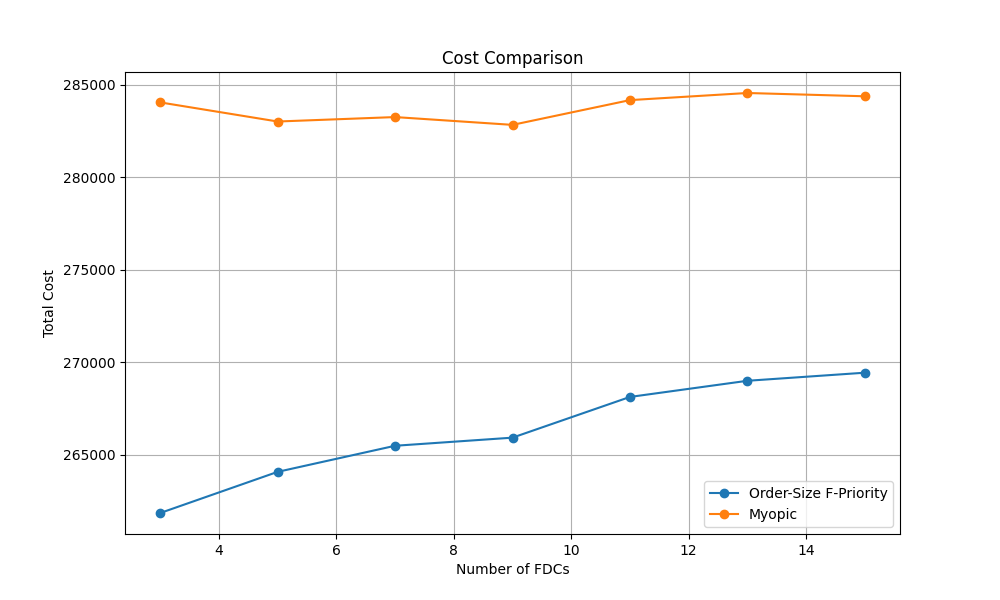}
    \end{subfigure}
    \begin{subfigure}{0.49\textwidth}
        \centering
        \includegraphics[width=\linewidth]{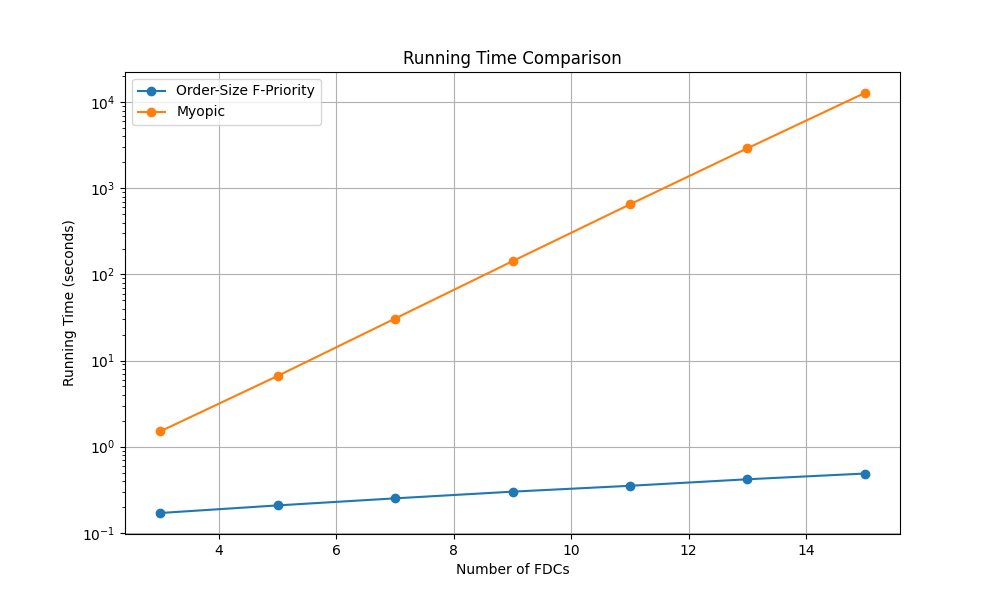}
    \end{subfigure}
    
    \caption{The cost and the total running time of \nsc{\multivaryingalg} and the myopic policy as the number of FDCs $K$ increases.}
    \label{fig:multi-fdc-time-varying-k}
\end{figure}

We also consider the single-FDC setting. We compare \nsc{\singvaryingalgimprove} with the myopic policy as $T$ increases, again using $T=200\cdot i$ for $i=1,2,\dots,10$. Figure~\ref{fig:single-fdc-time-varying-horizon} shows that \nsc{\singvaryingalgimprove} achieves costs comparable to, and often lower than, those of the myopic policy. This suggests that the benefit of the proposed \nsc{\gatedprioritygreedy} framework is not limited to large networks; it also helps in the single-FDC setting by moderating the use of scarce local inventory.

\begin{figure}[htbp]
    \centering
    \includegraphics[width=0.5\linewidth]{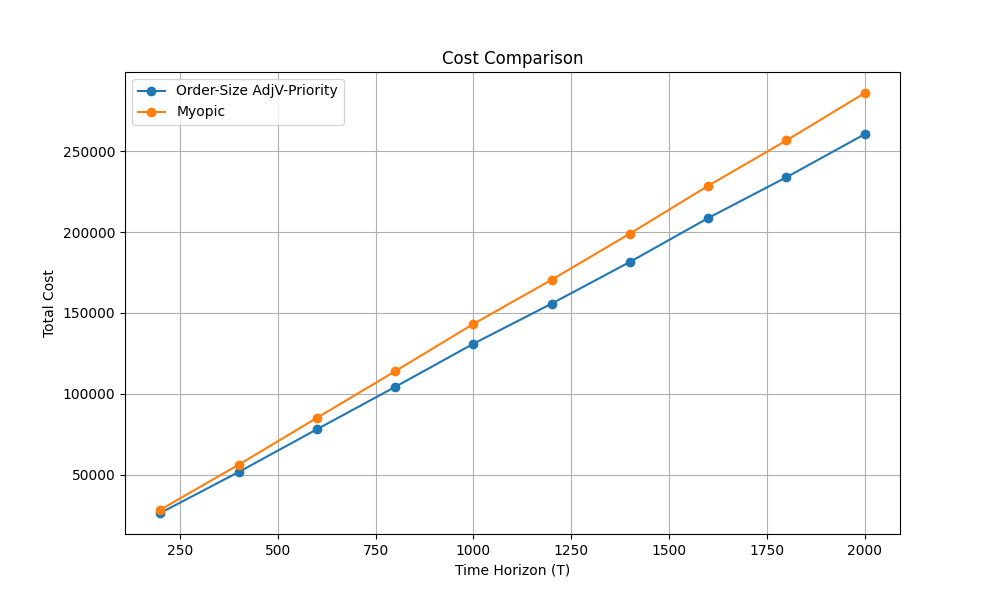}
    \caption{The cost of \nsc{\singvaryingalgimprove} and the myopic policy as the time horizon $T$ increases.}
    \label{fig:single-fdc-time-varying-horizon}
\end{figure}

Overall, the time-varying experiments show the main superiority of our approach: the proposed algorithms are substantially more computationally efficient, while their fulfillment costs remain close to or even lower than those of the myopic benchmark.

\subsection{Time-invariant Variable Costs}
\label{subsec: time-invariant experiments}

We next study the time-invariant variable cost setting. This setting is the standard benchmark in much of the online fulfillment literature and allows comparison with both myopic and LP-based forecast-driven policies. We compare \nsc{\multiinvaralg}, the generalized myopic policy in Algorithm~\ref{alg: generalized myopic policy}, and four LP-based policies with perfect demand forecasts: \textsc{IPFC} and \textsc{DPFC} from \citet{jasinLPBasedCorrelatedRounding2015}, and \textsc{Dilate} and \textsc{ForceOpen} from \citet{maOrderOptimalCorrelatedRounding2023} as specified in Section~\ref{subsec: baseline algorithms}. 

We first vary the time horizon by setting $T=200\cdot i$ for $i=1,2,\dots,10$. Figure~\ref{fig:multi-fdc-time-invariant-horizon} reports the results. \nsc{\multiinvaralg} achieves cost performance comparable to the myopic policy and the LP-based policies, even though the LP-based policies are supplied with the true demand distribution. At the same time, \nsc{\multiinvaralg} is much faster, highlighting the value of a simple online rule when real-time implementation is important.

\begin{figure}[htbp]
    \centering
    \begin{subfigure}{0.49\textwidth}
        \centering
        \includegraphics[width=\linewidth]{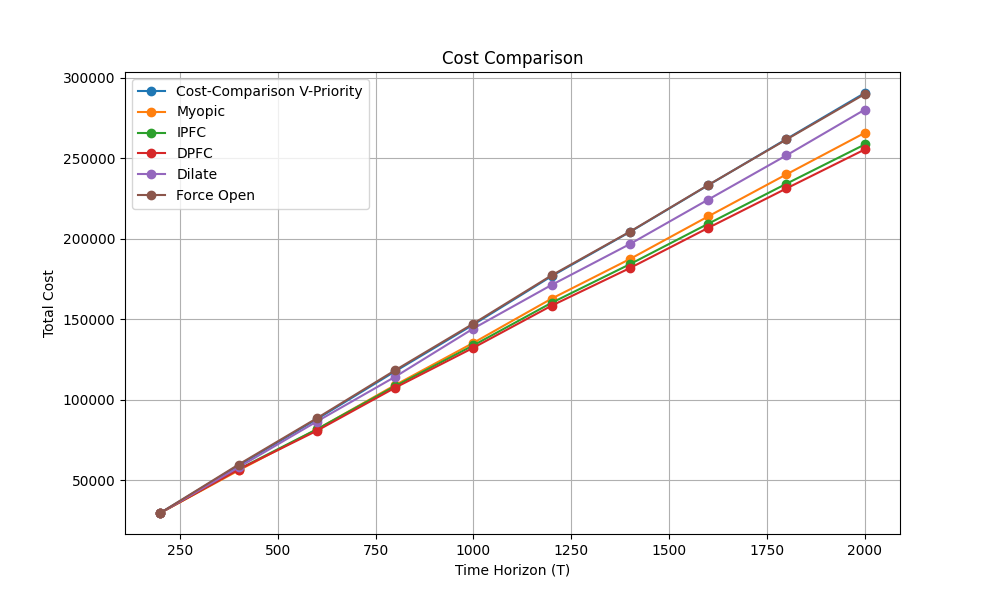}
    \end{subfigure}
    \begin{subfigure}{0.49\textwidth}
        \centering
        \includegraphics[width=\linewidth]{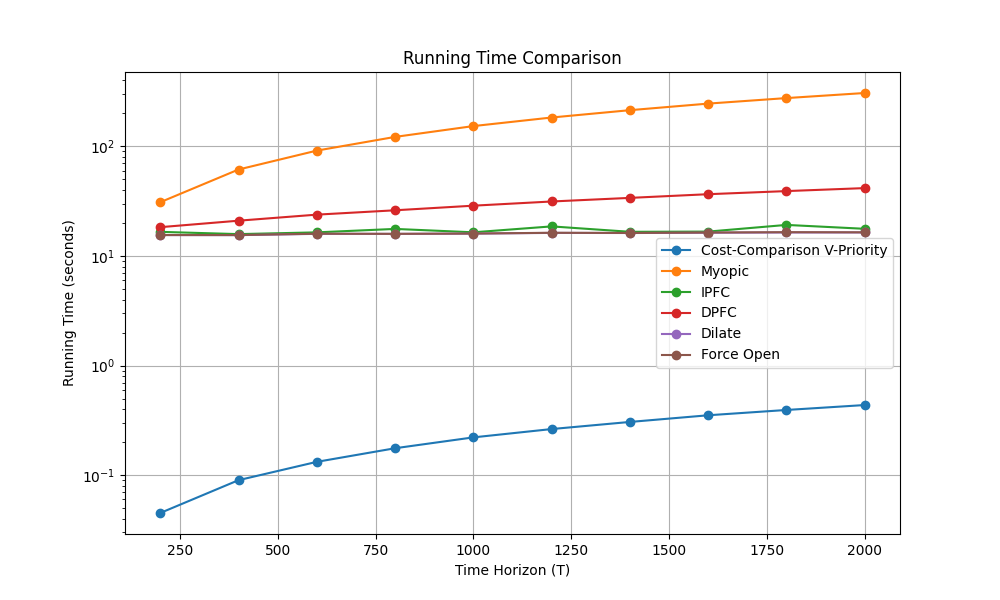}
    \end{subfigure}
    
    \caption{The cost and the total running time of \nsc{\multiinvaralg}, the myopic policy, and LP-based policies as the time horizon $T$ increases.}
    \label{fig:multi-fdc-time-invariant-horizon}
\end{figure}

We then vary the number of FDCs by setting $K=2i+1$ for $i=1,2,\dots,7$. Figure~\ref{fig:multi-fdc-time-invariant-k} shows a clear scalability pattern. The myopic policy becomes increasingly slow as $K$ grows, again reflecting the combinatorial nature of per-order cost minimization. The LP-based policies also become slower as the number of FDCs increases because the underlying linear programs contain more decision variables. By contrast, the running time of \nsc{\multiinvaralg} remains nearly unchanged. Its fulfillment cost remains close to those of the benchmark policies, demonstrating that the computational advantage is achieved without sacrificing much cost performance.

\begin{figure}[htbp]
    \centering
    \begin{subfigure}{0.49\textwidth}
        \centering
        \includegraphics[width=\linewidth]{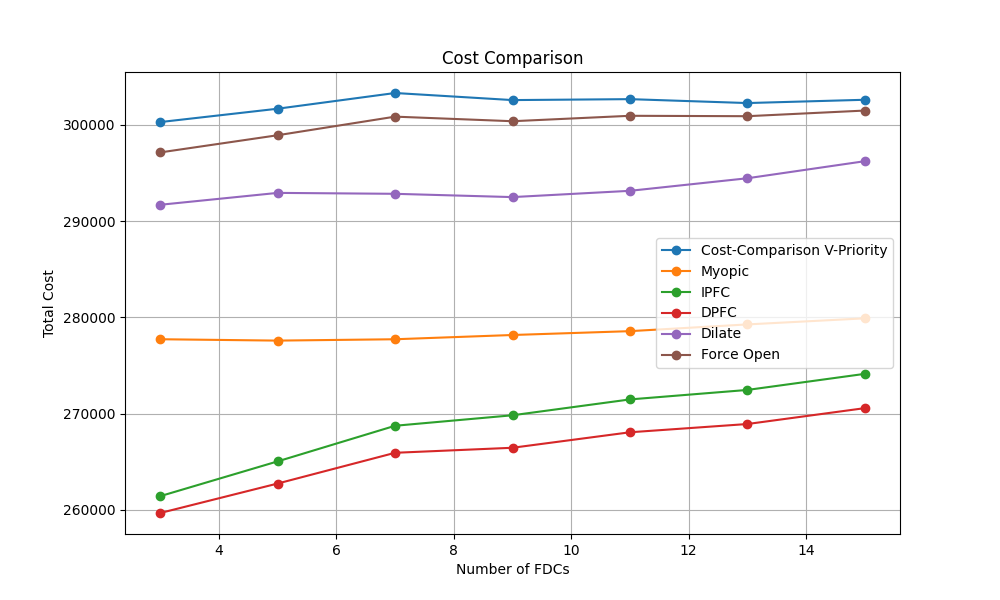}
    \end{subfigure}
    \begin{subfigure}{0.49\textwidth}
        \centering
        \includegraphics[width=\linewidth]{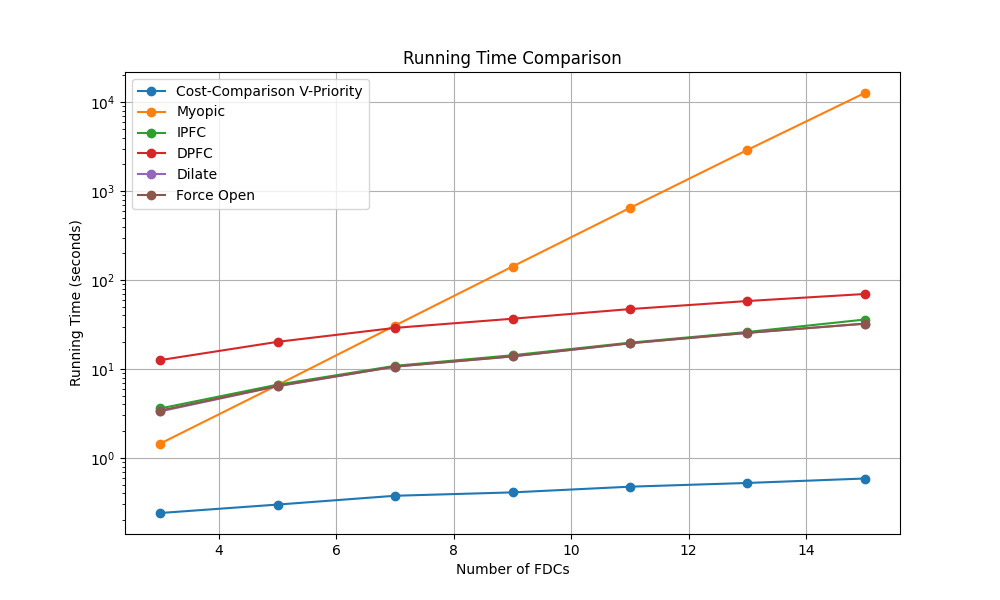}
    \end{subfigure}
    
    \caption{The cost and the total running time of \nsc{\multiinvaralg}, the myopic policy, and LP-based policies as the number of FDCs $K$ increases.}
    \label{fig:multi-fdc-time-invariant-k}
\end{figure}

Finally, we examine the single-FDC setting. We compare \nsc{\singinvaralg}, the myopic policy, and the same LP-based forecast benchmarks as $T$ increases. Figure~\ref{fig:single-fdc-time-invariant-horizon} shows that \nsc{\singinvaralg} also achieves cost performance comparable to the benchmark policies. Together with the multi-FDC results, this indicates that the proposed policies perform well not only in the adversarial worst-case analysis but also in stochastic simulations with time-invariant cost structures.

\begin{figure}[htbp]
    \centering
    \centering
    \includegraphics[width=0.5\linewidth]{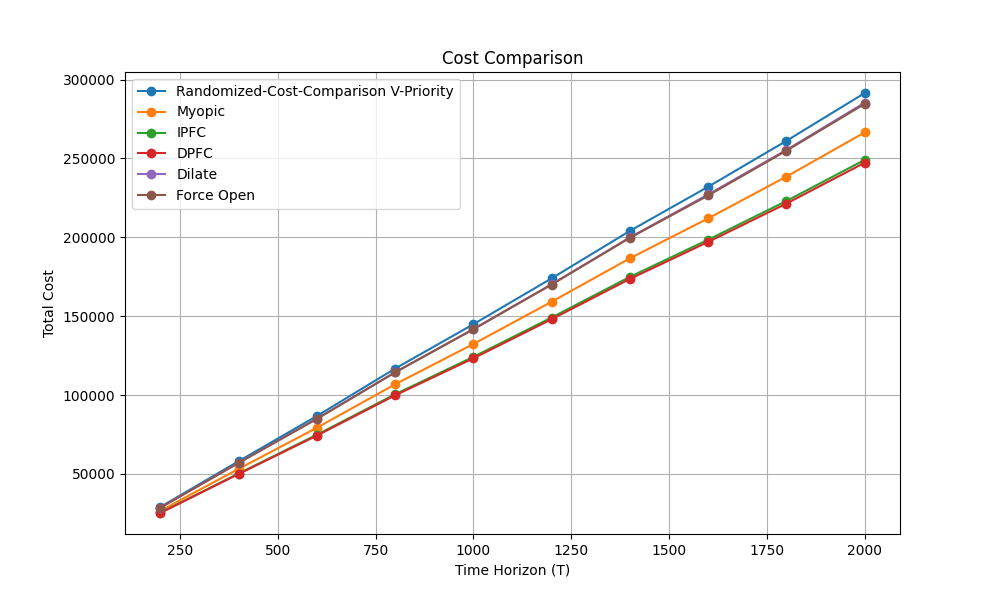}
    \caption{The cost and the total running time of \nsc{\singinvaralg}, the myopic policy, and LP-based policies as the time horizon $T$ increases.}
    \label{fig:single-fdc-time-invariant-horizon}
\end{figure}

\subsection{Stress Test under Extreme Cases}
\label{subsec: additional extreme cases}
The preceding experiments use stochastic instances intended to be practically motivated. We now complement them with an adversarially stressful family of instances within the time-varying framework. These instances are not intended to represent typical daily operations; instead, they isolate a failure mode of myopic fulfillment: using scarce FDC inventory for a large early order can be costly when many smaller future orders would have benefited more from FDCs.

Suppose there is one FDC and one RDC. The fixed cost of the FDC is 0, and the fixed cost of the RDC is $f_0>0$. We set $a=b=1$, so all variable costs are identical and equal to 1; this removes variable cost heterogeneity and isolates the inventory-depletion effect. The number of items is $n=\lceil \sqrt{f_0}\rceil$, and the horizon is $T=n+1$. In the first period, the order is $\boldsymbol{S}_1=(1,1,\dots,1)$, requesting one unit of every item. In each subsequent period $t=2,\dots,n+1$, the order requests one unit of item $t-1$. The initial inventory of each item at the FDC is one unit.

This construction creates a sharp intertemporal trade-off. Serving the first large order from the FDC is attractive from a one-period perspective, but doing so exhausts all local inventory and forces the later single-item orders to use the RDC repeatedly. A policy that protects FDC inventory can avoid this repeated fixed-cost burden. We compare \nsc{\multivaryingalg} with the myopic policy as $f_0$ varies from 50 to 500. Figure~\ref{fig:extreme-case-time-varying} shows that \nsc{\multivaryingalg} achieves a substantial cost advantage over the myopic policy. The widening gap as $f_0$ grows illustrates the stability benefit of the gating mechanism: the proposed policy is less vulnerable to temporal demand patterns that induce premature depletion of scarce FDC inventory.

\begin{figure}[htbp]
    \centering
    \includegraphics[width=0.5\linewidth]{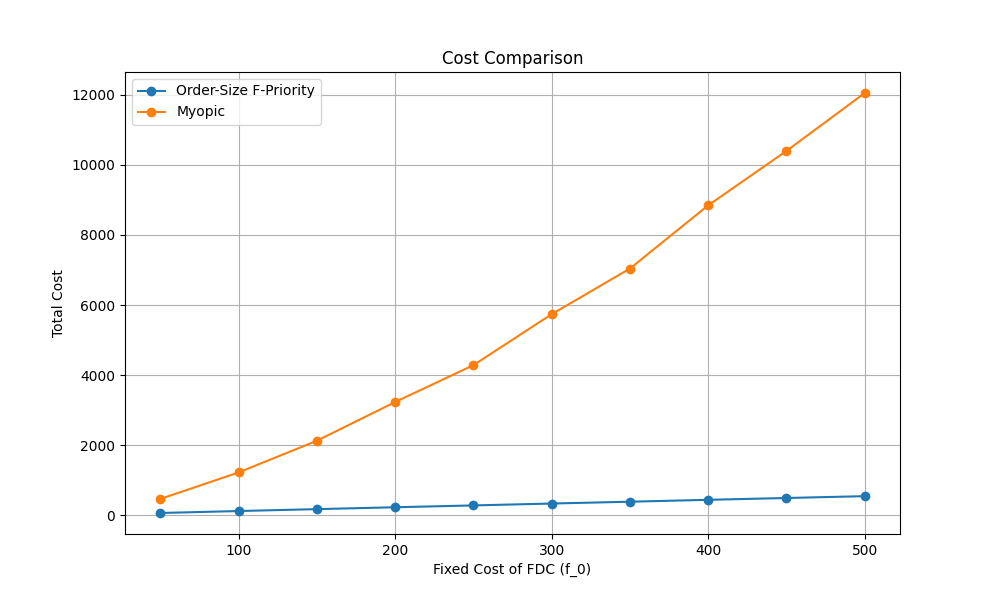}
    \caption{The cost of \nsc{\multivaryingalg} and the myopic policy in the extreme case with time-varying variable costs.}
    \label{fig:extreme-case-time-varying}
\end{figure}

Taken together, the experiments support the main empirical message of the paper. In practically motivated stochastic settings, the proposed policies are much more computationally efficient than the myopic and LP-based benchmarks while maintaining similar fulfillment costs. In more adversarial time-varying settings, the same gating idea can also yield markedly better cost performance than myopic fulfillment, demonstrating greater stability under challenging demand sequences.

\subsection{Baseline Algorithms}
\label{subsec: baseline algorithms}
\noindent \underline{Myopic Policy.} \cite{zhaoMultiitemOnlineOrder2025a} study the two-layer multi-item online order fulfillment problem with a single FDC and variable costs that are independent of both item and time, i.e. $c_{k,t}^i\equiv c_k,\forall i,t$. They propose a myopic policy that minimizes the fulfillment cost of each incoming order independently, without accounting for future costs, which is described in Algorithm~\ref{alg: myopic policy}. 

\begin{algorithm}[ht]
    \caption{Myopic Policy in \cite{zhaoMultiitemOnlineOrder2025a}}
    \label{alg: myopic policy}
    \begin{algorithmic}
        \For{each time period $t=1,2,\dots,T$}
            \State Observe customer order $\boldsymbol{S}_t$;
            \State Solve the following optimization problem to determine the fulfillment quantities:
            \[\{m_{k,t}^i\}_{k,i}=\arg\min_{\{m_{k,t}^i\}_{k,i}} f_0\cdot\mathbb{I}\left( \sum_{i=1}^n m_{0,t}^i > 0 \right)+ c_0 \cdot \sum_{i=1}^n m_{0,t}^i + f_1\cdot\mathbb{I}\left( \sum_{i=1}^n m_{1,t}^i > 0 \right)+ c_1 \cdot \sum_{i=1}^n m_{1,t}^i;\]
            \State Fulfill the order according to the optimal fulfillment quantities $\{m_{k,t}^i\}_{k,i}$;
            \State Update inventory levels: $I_{1,t}^i \gets I_{1,t-1}^i - m_{1,t}^i$ for all $i \in [n]$.
        \EndFor
    \end{algorithmic}
\end{algorithm}

For the setting with multiple FDC and item-specific, time-varying variable costs, the myopic policy in \cite{zhaoMultiitemOnlineOrder2025a} can be directly extended, which is described in Algorithm~\ref{alg: generalized myopic policy}. A natural question is whether the myopic policy can be computed in polynomial time. The answer is no; the difficulty is closely related to the NP-hardness of the Set Cover problem. See Proposition~\ref{prop:np-hardness-of-myopic}.

\begin{algorithm}[ht]
    \caption{Generalized Myopic Policy for Multiple-FDC Case}
    \label{alg: generalized myopic policy}
    \begin{algorithmic}
        \For{each time period $t=1,2,\dots,T$}
            \State Observe customer order $\boldsymbol{S}_t$;
            \State Solve the following optimization problem to determine the fulfillment quantities:
            \[\{m_{k,t}^i\}_{k,i}=\arg\min_{\{m_{k,t}^i\}_{k,i}} \sum_{k=0}^K \left[ f_k\cdot\mathbb{I}\left( \sum_{i=1}^n m_{k,t}^i > 0 \right)+ \sum_{i=1}^n c_{k,t}^i \cdot m_{k,t}^i \right];\]
            \State Fulfill the order according to the optimal fulfillment quantities $\{m_{k,t}^i\}_{k,i}$;
            \State Update inventory levels: $I_{k,t}^i \gets I_{k,t-1}^i - m_{k,t}^i$ for all $k\in [K]$ and $i \in [n]$.
        \EndFor
    \end{algorithmic}
\end{algorithm}

\begin{proposition}
    \label{prop:np-hardness-of-myopic}
    For any customer order $\boldsymbol{S}_t$, it is NP-hard to calculate the myopic fulfillment quantities $\{m_{k,t}^i\}_{k,i}$ that minimize the fulfillment cost of this order.
\end{proposition}

\begin{proof}{Proof of Proposition~\ref{prop:np-hardness-of-myopic}.}
    Given any instance of Set Cover, we show how it can be represented by an instance of the fulfillment problem. Suppose we have a universe $U=\{1,\dots,n\}$ and a collection of subsets $\mathcal{C}=\{C_1,\dots,C_K\}$, where $C_k\subseteq U$ for all $k$. We can construct an instance of the fulfillment problem as follows: we have $n$ items and $K$ FDCs, where the inventory of item $i$ in FDC $k$ is 1 if and only if $i\in C_k$ for all $k=1,\dots,K$. The fixed cost of FDC $k$ is set to be 1 for all $k=1,\dots,K$, and the fixed cost of RDC is set to be a sufficiently large number (e.g., larger than $K$). The variable costs are set to be 0 for all DCs. The customer order $\boldsymbol{S}_t$ is set to be $(1,1,\dots,1)$, i.e., one unit of each item is ordered.

    In this instance, the myopic policy selects a subset of FDCs to fulfill the order, and the fulfillment cost equals the number of selected FDCs. Therefore, computing the myopic policy is equivalent to finding a minimum set cover of $U$ using the subsets in $\mathcal{C}$, which is NP-hard. This completes the proof. \hfill\Halmos
\end{proof}

As we can see, the myopic policy is not computationally efficient for the case with multiple FDCs, which motivates us to design more efficient policies such as \nsc{\multivaryingalg} and \nsc{\multiinvaralg} on both time-varying and time-invariant cases.

\medskip \noindent \underline{LP-based Policies.} We revisit the LP-based policies proposed in \cite{jasinLPBasedCorrelatedRounding2015} and \cite{maOrderOptimalCorrelatedRounding2023}. These policies are designed for time-invariant variable costs and require a demand forecast, modeled as a probability distribution over customer orders. Orders are indexed by order type $q\in \mathcal{S}_Q$, where $\mathcal{S}_Q$ is the set of all possible order types. We write $i\in q$ (or $q\ni i$) if order type $q$ contains item $i$. Following their setting, we allow at most one unit of each item per order in this section. Let $\lambda^q$ denote the arrival probability of order type $q$ in any period $t$, satisfying $\sum_q \lambda^q = 1$. The stochastic version of the online order fulfillment problem can then be formulated as the following stochastic optimization problem:
\begin{align*}
    C^*(T) =~& \min~\sum_{t=1}^T\sum_{q\in\mathcal{S}_Q}\mathbb{E}\left[ D_t^q\cdot \sum_{k=0}^K\left( f_k\cdot\max_{i\in q}\left\{ X_{k,t}^{i,q} \right\} + \sum_{i\in q} c_k^i\cdot X_{k,t}^{i,q} \right) \right],\\
    &~\text{s.t.}~ \sum_{t=1}^T D_t^q X_{k,t}^{i,q} \leq I_{k,0}^i,\quad \forall k\in[K], i\in[n],\\
    &~~~~~~\sum_{k=0}^K X_{k,t}^{i,q} = 1,\quad \forall t\in[T], q\in\mathcal{S}_Q, i\in q,\\
    &~~~~~~X_{k,t}^{i,q}\in\{0,1\},\quad \forall k\in[K], t\in[T], q\in\mathcal{S}_Q, i\in [n].
\end{align*}
where the decision variables $X_{k,t}^{i,q}$ indicate whether item $i$ in order type $q$ is fulfilled by DC $k$ at time $t$, and $D_t^q$ is the indicator variable of whether order type $q$ arrives at time $t$, satisfying $\mathbb{E}[D_t^q] = \lambda^q$ for all $t\in[T]$ and $q\in\mathcal{S}_Q$. The objective is to minimize the expected total fulfillment cost over the time horizon.

\cite{jasinLPBasedCorrelatedRounding2015} develop an approximate deterministic linear-program formulation of the original stochastic optimization problem by replacing the random variables $X_{k,t}^{i,q}$ and $D_t^q$ with their expectations $x_{k,t}^{i,q}$ and $\lambda^q$, respectively, and relaxing the integrality constraints. The linear program is given as follows:
\begin{align*}
    C_{LP}^*(T) =~& \min~\sum_{t=1}^T\sum_{q\in\mathcal{S}_Q}\lambda^q\left[ \sum_{k=0}^K\left( f_k\cdot \max_{i\in q}\left\{ x_{k,t}^{i,q} \right\} + \sum_{i\in q} c_k^i\cdot x_{k,t}^{i,q} \right) \right],\\
    &~\text{s.t.}~ \sum_{t=1}^T \lambda^q x_{k,t}^{i,q} \leq I_{k,0}^i,\quad \forall k\in[K], i\in[n],\\
    &~~~~~~\sum_{k=0}^K x_{k,t}^{i,q} = 1,\quad \forall t\in[T], q\in\mathcal{S}_Q, i\in q,\\
    &~~~~~~x_{k,t}^{i,q}\geq 0,\quad \forall k\in[K], t\in[T], q\in\mathcal{S}_Q, i\in [n] .
\end{align*}

Let $x_k^{i,q}$ denote the average number of times item $i$ in order type $q$ is fulfilled from the DC $k$ during the selling horizon and let $y_k^q$ denote the average number of times order type $q$ is (partially) fulfilled from the DC $k$ during the selling horizon. The time-aggregate formulation of $C_{LP}^*(T)$ is given as follows:
\begin{align*}
    \tilde{C}_{LP}^*(T) =~& \min~T\cdot \sum_{q\in\mathcal{S}_Q}\lambda^q\left[ \sum_{k=0}^K\left( f_k\cdot y_k^q + \sum_{i\in q} c_k^i\cdot x_k^{i,q} \right) \right],\\
    &~\text{s.t.}~ T\cdot \lambda^q x_k^{i,q} \leq I_{k,0}^i,\quad \forall k\in[K], i\in[n],\\
    &~~~~~~\sum_{k=0}^K x_k^{i,q} = 1,\quad \forall t\in[T], q\in\mathcal{S}_Q, i\in q,\\
    &~~~~~~y_k^q \geq x_k^{i,q},\quad \forall k\in[K], q\in\mathcal{S}_Q, i\in q,\\
    &~~~~~~x_k^{i,q}\geq 0,\quad \forall k\in[K], q\in\mathcal{S}_Q, i\in [n].
\end{align*}

It is not difficult to see that $\tilde{C}_{LP}^*(T)= C_{LP}^*(T)\leq C^*(T)$ using Jensen's inequality. After solving the linear program and obtaining the optimal solution $\{x_k^{i,q}, y_k^q\}_{k,i,q}$, \cite{jasinLPBasedCorrelatedRounding2015} and \cite{maOrderOptimalCorrelatedRounding2023} propose different correlated rounding schemes to round the fractional solution to an integral solution, which can be implemented as an online fulfillment policy. 

In \cite{jasinLPBasedCorrelatedRounding2015}, they propose two rounding schemes. The first one is the independent rounding scheme, namely \textsc{IPFC} (Independent Probabilistic Fulfillment Control) policy, which independently rounds $\{ x_k^{i,q} \}_{k=0}^K$ to 0 or 1 with probability $x_k^{i,q}$, i.e. $\Pr[X_{k,t}^{i,q} = 1] = x_k^{i,q}$. The second one is the dependent rounding scheme, namely \textsc{DPFC} (Dependent Probabilistic Fulfillment Control) policy, which rounds $x_k^{i,q}$ in a dependent manner to ensure that the fulfillment decisions of different items in the same order type are correlated. They show that both IPFC and DPFC policies achieve an expected competitive ratio growing linearly with the expected order size $\mathbb{E}[|q|]$. We refer readers to \cite{jasinLPBasedCorrelatedRounding2015} for more details of the DPFC rounding scheme and their theoretical analysis.

In \cite{maOrderOptimalCorrelatedRounding2023}, the authors propose two alternative correlated-rounding schemes, namely the \textsc{Dilate} policy and the \textsc{ForceOpen} policy. These policies construct new rounding rules and achieve improved expected competitive ratios, with growth rates that are logarithmic in the expected order size and linear in the number of FDCs, respectively. We refer readers to \cite{maOrderOptimalCorrelatedRounding2023} for details of these rounding schemes and their theoretical analysis.

\end{document}